\documentclass[11pt]{article}
\usepackage{enumerate}
\usepackage[OT1]{fontenc}
\usepackage[usenames]{color}
\usepackage{smile}
\usepackage{amsmath}
\usepackage{amssymb}
\usepackage{amsthm}
\usepackage{bm}
\usepackage[colorlinks,
            linkcolor=red,
            anchorcolor=blue,
            citecolor=blue
            ]{hyperref}
\usepackage{fullpage}
\usepackage[protrusion=true,expansion=true]{microtype}
\usepackage{pbox}
\usepackage{float}
\usepackage{enumitem}
\usepackage{graphicx}
\usepackage{subfigure}
\usepackage{pgfplots}
\pgfplotsset{compat=1.18}
\usepackage{booktabs} % for professional tables
\usepackage{pifont} % for \ding checkmarks/crosses

\newcommand{\KL}{\text{KL}}

\allowdisplaybreaks
\usepackage{colortbl}
\definecolor{LightCyan}{rgb}{0.8, 0.9, 1}
\definecolor{LightGray}{gray}{0.9}

\newcommand{\pif}{\pi_{\mathrm{F}}}
\newcommand{\pir}{\pi_{\mathrm{R}}}
\newcommand{\piref}{\pi_{\text{ref}}}

\usepackage{todonotes}
\newcommand{\todoqg}[2][]{}

\makeatletter
\newcommand*{\rom}[1]{\expandafter\@slowromancap\romannumeral #1@}
\makeatother
\title{\huge Understanding Off- vs On-Policy Distillation: A Tale of Distinct Training Objectives}
\author{
    Qiwei Di\thanks{Department of Computer Science, University of California, Los Angeles, CA 90095, USA; e-mail: {\tt qiwei2000@cs.ucla.edu}}
    ~~
    Xuheng Li\thanks{Department of Computer Science, University of California, Los Angeles, CA 90095, USA; e-mail: {\tt xuheng.li@cs.ucla.edu}} 
    ~~
    Kaixuan Ji\thanks{Department of Computer Science, University of California, Los Angeles, CA 90095, USA; e-mail: {\tt kaixuanji@cs.ucla.edu}} 
    ~~
    Chenggong Zhang\thanks{Department of Computer Science, University of California, Los Angeles, CA 90095, USA; e-mail: {\tt chenggong61@g.ucla.edu}}
    ~~
    \\
    Heyang Zhao\thanks{Department of Computer Science, University of California, Los Angeles, CA 90095, USA; e-mail: {\tt hyzhao@cs.ucla.edu}} 
    ~~
    Quanquan Gu\thanks{Department of Computer Science, University of California, Los Angeles, CA 90095, USA; e-mail: {\tt qgu@cs.ucla.edu}}
}
\begin{document}
    \date{}
    \maketitle
\begin{abstract}
On-policy distillation (OPD) learns from teacher feedback on student-generated responses and has shown promise in reducing forgetting relative to supervised fine-tuning (SFT). However, its benefits and fragility remain incompletely understood. We study sequential distillation from multiple teachers, where the student minimizes its average divergence from the teachers. Forward Kullback--Leibler (KL) divergence yields a weighted arithmetic mixture, while reverse KL yields a normalized weighted geometric aggregate. We develop algorithms that learn these targets under off-policy and on-policy feedback, respectively, establishing logarithmic regret bounds in the tabular setting and extending the analysis to function approximation. By analyzing these aggregation targets, we identify mechanisms that help explain both the benefits and fragility of OPD. Relative to forward KL, reverse KL can better retain a confident expert's preferences under uninformative feedback, but is more sensitive to teachers that assign very low probabilities to correct responses. Its token-level conditionals also reveal a dependence on continuation distributions that can favor incorrect prefixes over long horizons.
\end{abstract}
\section{Introduction}
Knowledge distillation transfers capabilities from teacher models to a student language model. A conventional approach is supervised fine-tuning (SFT) on teacher-generated responses, where the student learns to predict tokens along trajectories supplied by the teacher. On-policy distillation (OPD) \citep{agarwal2024policy,lu2025onpolicydistillation} instead samples responses from the current student and obtains teacher logits along the resulting trajectories. This allows the student to learn from teacher feedback on its own behavior, an approach that has been adopted in the post-training of large language models \citep{yang2025qwen3,xiao2026mimo,zeng2026glm}. Related studies show that on-policy reinforcement learning can reduce forgetting relative to SFT \citep{shenfeld2026rl,chen2025retaining} and, in some settings, recover capabilities lost during SFT \citep{jin2025rl}. 
% On-policy self-distillation likewise shows promise for continual learning, improving new-task performance while reducing forgetting relative to SFT \citep{shenfeld2026self}.
% Multi-teacher OPD has also been used to combine capabilities from specialized teachers into a single student \citep{ma2026mopd}.

Despite these advantages, OPD can also be fragile. It may fail under teacher--student mismatch and suffer from training instability over long generation horizons \citep{li2026rethinking}. However, a theoretical understanding of when OPD offers advantages over off-policy methods such as SFT and what limits those advantages remains to be established. Recent work has begun to clarify the conditions under which these advantages arise. For example, \citet{sriraman2026behavior} show that in a noisy-expert model, matching the clean expert's return can require exponentially many offline samples in the horizon, whereas an on-policy algorithm achieves polynomial dependence when the corruption is known. \citet{viano2026interaction} show that on-policy value-based imitation can efficiently match the expert's return under realizability of its action-value function, without requiring policy realizability. Together, these works establish conditions under which on-policy interaction offers advantages over offline learning. However, they do not directly explain the fragility observed in practical OPD, leaving a gap in our understanding of when and why it succeeds or fails.

This naturally raises a question: \emph{What are the pros and cons caused by the distinction between off-policy and on-policy feedback in knowledge distillation?}

% These successes raise a theoretical question: how does the distinction between off-policy and on-policy feedback affect the knowledge that a student ultimately acquires? The source of the training trajectories determines where feedback is observed, but this alone does not explain how the student reconciles teachers that differ in expertise or disagree on particular responses. Nor does it imply that on-policy learning always preserves useful knowledge. Empirical studies report that OPD can fail under teacher--student mismatch and identify difficulties with long-horizon distillation \citep{li2026rethinking}. Understanding both its benefits and its failure modes calls for examining the aggregation objectives that these feedback protocols support.
As a first step toward answering this question, we consider the setting where the student distills sequential data from multiple teachers, and the goal of the student is to minimize its average divergence from all teachers. Within this framework, we connect off- and on-policy distillation to forward- and reverse-Kullback--Leibler (KL) objectives, respectively. We then derive closed-form solutions for both objectives, offering a new perspective on the differences between SFT and OPD through the aggregation targets they induce. Our main contributions are listed as follows:
\begin{itemize}[leftmargin=*]
    \item We propose a sequential distillation framework in which the student minimizes its average divergence from all teachers, and characterize the resulting aggregation targets under forward and reverse KL divergence. Specifically, forward KL yields a weighted \textbf{arithmetic} mixture of the teacher distributions, while reverse KL yields a normalized weighted \textbf{geometric} mean. We also derive the token-level conditionals of both targets for autoregressive generation.

    \item We develop algorithms that learn the forward-KL target from off-policy data collection and the reverse-KL target from teacher logits on student-generated responses. In the tabular setting, we establish $\widetilde O(\log T)$ regret bounds for both protocols, with regret defined as the cumulative divergence between the student's policy at each round and the corresponding aggregation target. We further extend the logarithmic regret to the setting with general function approximation. As a result, these guarantees connect off-policy distillation to the forward-KL aggregation objective and on-policy distillation to the reverse-KL aggregation objective.

    \item We compare the two targets to examine the benefits and fragility of reverse-KL aggregation relative to forward-KL aggregation. First, we consider a setting with one confident expert teacher and several uninformative teachers and show that reverse-KL aggregation can retain a higher probability of the correct response than forward-KL aggregation. We also show that a misleading teacher can suppress correct responses under reverse KL, whereas forward KL preserves the expert's weighted contribution. Finally, by examining the token-level conditionals, we show that over long horizons, reverse-KL aggregation can favor an incorrect prefix even while preserving the expert's ranking of complete responses. These results offer possible explanations for the benefits and fragility of OPD observed in practice.
\end{itemize}
\noindent\textbf{Notation.}
For a positive integer $n$, let $[n]:=\{1,\ldots,n\}$. For a finite set $\mathcal S$, let $\Delta(\mathcal S)$ denote the set of probability distributions on $\mathcal S$. We write $\ind(E)$ for the indicator of an event $E$, and use $\EE$, $\PP$, and $\Var$ for expectation, probability, and variance, respectively. For distributions $p,q\in\Delta(\mathcal S)$, the Kullback--Leibler divergence is $\KL(p\|q):=\sum_{z\in\mathcal S}p(z)\log[p(z)/q(z)]$, where terms with $p(z)=0$ are zero, and $\KL(p\|q)=+\infty$ if $p(z)>0=q(z)$ for some $z$. We use $O(\cdot)$ for asymptotic upper bounds and $\widetilde O(\cdot)$ to hide logarithmic factors except $\log T$. The notation $a\lesssim b$ means $a\leq Cb$ for a universal constant $C>0$.

\section{Related Work}
\label{sec:related-work}

\paragraph{Variants of on-policy distillation.}
On-policy distillation trains a language model by having it generate responses and using a teacher's next-token distributions to supervise those responses \citep{agarwal2024policy,gu2024minillm}. Following the empirical success of OPD, subsequent work has explored improvements to its training objectives, the use of teacher feedback, and the construction of teachers. To improve the training objective, existing methods introduce skew KL or contrastive losses \citep{ko2024distillm,ko2025distillm}, combine forward and reverse KL using hybrid objectives or teacher--student agreement \citep{zhu2026hybrid,jin2026entropy,xing2026trust}, or construct intermediate distributions between the teacher and student as training targets \citep{jang2026stable,xie2026trust}. To make better use of teacher feedback, other methods use response outcomes to weight or calibrate supervision \citep{zheng2026scope,hou2026uni}, reduce variance or reshape distillation rewards \citep{oh2026kl,zhao2026poweropd}, and stabilize training through policy constraints and mixed rollout sources \citep{luo2026demystifying}. Data collection is also adapted through selective teacher intervention during student generation \citep{xu2025speculative} and early termination of student rollouts to reduce generation costs and avoid unreliable teacher feedback \citep{ziheng2026less,xin2026escaping,zhang2026fast}. Beyond these changes to objectives and training, the source of supervision has expanded from a separate teacher model to self-teachers conditioned on reasoning traces, demonstrations, additional context, or environmental feedback \citep{zhao2026self,shenfeld2026self,ye2026policy,hubotter2026reinforcement,liu2026self}. Multiple teachers can also supervise a shared student, whether they are independently trained domain experts \citep{ma2026mopd} or constructed through task-specific soft prompts \citep{ma2026one}. These extensions motivate studying how distillation combines knowledge from different teachers. We refer readers to \citet{song2026survey} for a comprehensive survey of OPD methods.

\paragraph{Understanding on-policy distillation.}
Recent studies examine both the advantages and the failure modes of on-policy distillation. For capability retention during RL fine-tuning, \citet{shenfeld2026rl} relate reduced forgetting to an implicit bias toward solutions close to the initial policy in KL divergence, while \citet{chen2025retaining} emphasize on-policy data and mode-seeking behavior. Empirical studies also show that stronger teachers and denser supervision do not necessarily yield better distillation. \citet{li2026rethinking} find that OPD gains depend on compatible reasoning patterns and the teacher offering capabilities beyond those already acquired by the student. They further observe that teacher guidance becomes less effective on longer student-generated prefixes, with training instability emerging at later tokens and spreading to earlier positions. \citet{wang2026demystifying} identify settings where teacher--student mismatch causes incorrect responses to receive higher average token-level rewards than correct ones. \citet{fu2026revisiting} identify three sources of instability in sampled-token OPD: predominantly negative token-level rewards that make learning sensitive to a small set of positively rewarded tokens; unreliable teacher feedback that can reinforce repetitive or meaningless continuations on student-generated prefixes; and tokenizer or special-token mismatches that penalize semantically valid outputs. The theoretical understanding of OPD is relatively limited. \citet{sriraman2026behavior} establish an exponential separation in horizon dependence between offline learning and an on-policy algorithm with known expert corruption, and \citet{viano2026interaction} obtain efficient imitation under expert action-value realizability without policy realizability. Our contribution is to characterize their distinct optimal targets when aggregating multiple teachers and connect those targets to learning guarantees. We then analyze how teacher confidence, misleading feedback, and autoregressive continuations affect the expert's preferences at these targets, providing possible mechanisms for both capability retention and fragility. 

\paragraph{Relevant RL theory.}
To turn our perspective on distinct aggregation objectives into formal learning guarantees, we rely on the established RL theory, including supervised learning, KL-regularized bandits, and reinforcement learning with function approximation. Our analysis of the off-policy distillation protocol builds on the classical theory of behavior cloning, which aims to learn a policy from expert demonstrations. Classical analyses characterize how prediction errors compound under distribution shift \citep{syed2010reduction,ross2010efficient}. DAgger addresses this issue by collecting expert feedback on learner-visited states, reducing imitation learning to no-regret online learning \citep{ross2011reduction}. Subsequent work further characterizes the statistical limits of imitation learning and the benefits of interaction \citep{rajaraman2020toward,rajaraman2021value}. More recently, \citet{foster2024behavior} establish guarantees for behavior cloning with logarithmic loss under realizability, highlighting the importance of the loss function and policy-class complexity. Our forward-KL guarantees draw on log-loss prediction methods, combining smoothed probability estimation in the tabular setting and exponentially weighted policy aggregation under function approximation with a new one-sided bounded martingale concentration inequality.

For on-policy distillation, our analysis draws on prior work establishing logarithmic regret for KL-regularized bandits and reinforcement learning. \citet{zhao2026sharp} establish $O(1/\epsilon)$ sample complexity in a sufficiently small-error regime under reference-policy coverage, and later \citet{zhao2025logarithmic} extend the results to obtain logarithmic regret in online contextual bandits and RL. For multi-armed bandits, \citet{ji2026near} sharpen the dependence on the number of arms and regularization strength through nearly matching regret upper and lower bounds. Similar fast-rate guarantees have also been established in the offline setting for multi-armed bandits \citep{ji2026optimal}, or with sufficient coverage conditions under KL and more general divergences~\citep{zhao2026towards,zhao2026fast}. This line of work has since been extended to multiple reference models \citep{aminian2026kl}, general preference models \citep{wu2026greedy}, game-theoretic settings \citep{nayak2025achieving}, differential privacy \citep{wu2025offline}, and misspecified models \citep{hong2026online}. We adapt these techniques to sequential distillation by treating teacher log-probability ratios as rewards and accounting for dependence among rollouts from the same sampled teacher. This yields logarithmic regret guarantees for learning the reverse-KL aggregation target from on-policy teacher feedback.

While we have established the connection between off- and on-policy distillation and their respective aggregation targets in the tabular setting, we seek to extend this connection to more general settings, drawing inspiration from prior work on general function approximation. In contextual bandits, \citet{russo2013eluder} introduce eluder dimension to quantify how past observations constrain predictions at unobserved actions. Subsequent work extends eluder-based analyses to model estimation through value-targeted regression \citep{ayoub2020model} and general value-function approximation \citep{wang2020reinforcement}. We use a generalized notion of eluder dimension developed in later work \citep{agarwal2023vo,zhao2024nearly,di2024pessimistic} and adapt it to the autoregressive setting.
\section{Preliminaries}
We study sequential distillation, in which a single student learns from multiple teachers over successive training rounds with the goal of aggregating their knowledge.

Let $\cX$ be a finite context space of size $S:=|\cX|$, $\cA$ a finite token space of size $A:=|\cA|$, and $H$ the generation horizon. We write $\cY := \cA^H$ for the response space. Each response $y=(a_1,\ldots,a_H)\in\cY$ is generated autoregressively, and we denote its length-$(h-1)$ prefix by $y_{<h}:=(a_1,\ldots,a_{h-1})$. We consider a sequential distillation setting with a finite collection of teacher models indexed by $\cI$, where $|\cI|=I$. For each $i\in\cI$, let $\rho_i\in\Delta(\cX)$ be its context distribution and let $p_i$ be an autoregressive teacher policy such that for any $h \in [H]$, the token conditionals are written as
\[
p_i(\cdot\mid x,y_{<h})\in\Delta(\cA).
\]
The induced sequence-level policy is
\[
p_i(y\mid x)
=
\prod_{h=1}^{H}p_i(a_h\mid x,y_{<h}).
\]
At round $t$, the student policy $\pi_t$ is determined by the observations from previous rounds. A latent teacher index $i_t$ is then drawn uniformly from $\cI$, independently of previous rounds. Conditional on $i_t$, we draw $m$ contexts independently from $\rho_{i_t}$. The entire batch shares this teacher: teacher responses and feedback are generated according to $p_{i_t}$. We consider two distillation protocols:
\begin{itemize}[leftmargin=*]
    \item (Off-policy Distillation) We sample $m$ contexts $\{x_{t,j}\} \sim \rho_{i_t}$ and full teacher rollouts $y_{t,j} \sim p_{i_t}(\cdot|x_{t,j})$. If we further have access to the teacher logits, at each teacher-generated prefix $y_{t,j,<h}$, we additionally observe the full logit vector $Z_{t,j,h}\in\RR^{\cA}$. The corresponding token log-probabilities are obtained by applying log-softmax:
    \begin{align*}
        \log p_{i_t}(a\mid x_{t,j},y_{t,j,<h})
        &=Z_{t,j,h}(a)-\log\sum_{b\in\cA}e^{Z_{t,j,h}(b)}, \qquad a\in\cA.
    \end{align*}
    % \item 
    % \todox{Explain why we study this setting, e.g., for fair comparison}
    % (Off-policy Distillation with Full Logits) We sample $m$ contexts $\{x_{t,j}\} \sim \rho_{i_t}$ and full teacher rollouts $y_{t,j} \sim p_{i_t}(\cdot|x_{t,j})$. At each teacher-generated prefix $y_{t,j,<h}$, we additionally observe the full logit vector $Z_{t,j,h}\in\RR^{\cA}$. The corresponding token log-probabilities are obtained by applying log-softmax:
    % \begin{align*}
    %     \log p_{i_t}(a\mid x_{t,j},y_{t,j,<h})
    %     &=Z_{t,j,h}(a)-\log\sum_{b\in\cA}e^{Z_{t,j,h}(b)}, \qquad a\in\cA.
    % \end{align*}
    \item (On-policy Distillation) We sample $m$ contexts $\{x_{t,j}\} \sim \rho_{i_t}$, use the current student policy $\pi_t$ to roll out with $y_{t,j} \sim \pi_t(\cdot|x_{t,j})$ token by token, and receive the per-token teacher log-probabilities
    \begin{align*}
        l_{t,j,h} = \log p_{i_t}(a_{t,j,h}\,|\,x_{t,j},\,y_{t,j,<h}), \qquad h \in [H].
    \end{align*}
\end{itemize} 
Rather than matching any single teacher selected in the current round, our goal is to learn a single autoregressive student policy that aggregates the behavior of all teacher models. This objective is designed to avoid catastrophic forgetting: sequentially adapting the student to the current teacher may reduce its loss on the corresponding task while degrading its performance on teachers encountered previously. To formalize the desired aggregate policy, let
\[
    D:\Delta(\cY)\times\Delta(\cY)
    \rightarrow \RR_{\geq 0}\cup\{+\infty\}
\]
be a divergence between distributions. Given the collection of teacher models
$\{(\rho_i,p_i)\}_{i\in\cI}$, we define the target policy by
\begin{align}
\label{eq:average-objective}
    \pi^*
    \in
    \argmin_{\pi:\cX\rightarrow\Delta(\cY)}
    \sum_{i\in\cI}
    \EE_{x\sim\rho_i}
    \big[
        D\big(
            \pi(\cdot\mid x)
            \,\big\|\,
            p_i(\cdot\mid x)
        \big)
    \big].
\end{align}
Here, $\pi(\cdot\mid x)$ and $p_i(\cdot\mid x)$ denote the induced distributions over complete responses in $\cY=\cA^H$. 
\begin{remark}
The initial student policy can also serve as a reference during training. In reinforcement learning from human feedback (RLHF), for example, the updated policy is often regularized toward its initial policy to limit how far it departs from that reference. To capture this role in our framework, we regard the reference policy as an additional teacher and include it in the teacher set as a modeling simplification. Under this interpretation, even distillation from a single external teacher can be viewed as learning from both that teacher and the initial student.
\end{remark}
Specifically, when $D$ is the reverse KL or forward KL divergence, the objective can be decomposed into a sum of token-level divergences using the chain rule.
% \begin{proposition}[Token-level decomposition]
% \label{prop:kl-chain}
% Fix $x$ and let $\pi(\cdot|x), p(\cdot|x)$ be two distributions on $\cY$. Then
% \begin{itemize}[leftmargin=*]
%     \item If $D(p\|q) = \KL(p\|q)$, then 
%     \begin{align*}
%     D\big(\pi(\cdot|x)\,\big\|\,p(\cdot|x)\big) &= \sum_{h=1}^H \EE_{y \sim \pi(\cdot|x)}\Big[\KL\big(\pi(\cdot|x,y_{<h})\,\big\|\,p(\cdot|x,y_{<h})\big)\Big].
%     \end{align*}
%     \item If $D(p\|q) = \KL(q\|p)$, then
%     \begin{align*}
%     D\big(\pi(\cdot|x)\,\big\|\,p(\cdot|x)\big) &= \sum_{h=1}^H \EE_{y \sim p(\cdot|x)}\Big[\KL\big(p(\cdot|x,y_{<h})\,\big\|\,\pi(\cdot|x,y_{<h})\big)\Big].
%     \end{align*}
% \end{itemize}
% Each summand depends on $y$ only through the prefix $y_{<h}$.
% \end{proposition}
In this work, we will mainly focus on the cases of forward KL $(D(p\|q)=\KL(q\|p))$ and reverse KL $(D(p\|q)=\KL(p\|q))$.
\paragraph{Regret.} Since teacher indices are sampled uniformly, the marginal context distribution is the uniform mixture of the teacher context distributions:
\begin{align*}
    \bar{\rho}(x)
    :=
    \frac{1}{I}\sum_{i\in\cI}\rho_i(x).
\end{align*}
For each choice of $D$, let $\pi^*$ denote the corresponding minimizer of
\eqref{eq:average-objective}. We define regret as the cumulative divergence of the student policies $\{\pi_t\}_{t=1}^T$ from this target, averaged over the marginal context distribution:
\begin{align}
    \operatorname{Regret}(T)
    :=
    \sum_{t=1}^T
    \EE_{x\sim\bar{\rho}}
    \left[
        D\bigl(
            \pi_t(\cdot\mid x)
            \,\big\|\,
            \pi^*(\cdot\mid x)
        \bigr)
    \right].
    \label{eq:regret}
\end{align}
SFT has a well-known forward-KL interpretation, while reverse KL is commonly used in OPD. Motivated by these connections, we take $D(p\|q)=\KL(q\|p)$ for off-policy distillation and $D(p\|q)=\KL(p\|q)$ for on-policy distillation.

%%%%%%%%%%%%%%%%%%%%%%%%%%%%%%%%%%%%%%%%%%%%%%%%%%%%%%%%%%%%%%%%%%%%%%%%
% NOTE: ``the next section'' in this section refers to ``Forward KL + Tabular setting'' (single-step)
\section{Off-policy Distillation \& Forward KL}
\label{sec:ar-fkl-tabular}
In this section, we study how off-policy feedback can be used to learn the forward-KL aggregate. We first characterize this target, then construct a tabular algorithm using teacher-generated responses, with or without full teacher logits, and establish logarithmic regret bounds. As the first step, the following theorem gives the target distribution and its token conditionals.
\begin{theorem}[Forward KL] 
\label{thm:obj-forward-kl}
If $D(p\|q) = \KL(q\|p)$, a minimizer of \eqref{eq:average-objective} is given at the sequence level, for each $x$ with $\bar\rho(x)>0$, by
\begin{align}
\label{eq:seq-fkl}
\pif^*(y|x) = \sum_{i \in \cI} w_i(x) p_i(y|x),
\end{align}
where $w_i(x) = \rho_i(x)/[\sum_{j \in \cI} \rho_j(x)]$. Moreover, for any $h \in [H]$, supposing $\sum_{j\in\cI}\rho_j(x)p_j(y_{<h}\mid x)>0$, the token-level distributions are given by
    \begin{align}
    \label{eq:ar-fkl}
        \pif^*(a|x,y_{<h}) = \sum_{i \in \cI} w_i(x,y_{<h})\, p_i(a|x,y_{<h}), \qquad w_i(x,y_{<h}) := \frac{\rho_i(x)\, p_i(y_{<h}|x)}{\sum_{j \in \cI} \rho_j(x)\, p_j(y_{<h}|x)}.
    \end{align}
\end{theorem}
\begin{remark}
The prefix-dependent weights admit a Bayesian interpretation. Under the
off-policy sampling protocol, the teacher index is first sampled uniformly,
the context is drawn from $\rho_{i_t}$, and the response is generated by
$p_{i_t}$. Therefore, for every prefix satisfying
$\sum_{j\in\cI}\rho_j(x)p_j(y_{<h}\mid x)>0$, Bayes' rule gives
\begin{align}
\notag
    w_i(x,y_{<h})
    =
    \PP\big(
        i_t=i
        \big|
        x_{t,j}=x,
        y_{t,j,<h}=y_{<h}
    \big).
\end{align}
Thus, $w_i(x,y_{<h})$ is the posterior probability that teacher $i$
generated the observed context and prefix. Consequently, teachers that better explain the observed prefix receive greater weight in the target policy's next-token distribution.
\end{remark}
\subsection{Algorithm Design}
The posterior weights above also determine the next-token distribution in teacher-generated data. For a teacher-generated rollout $(x_{t,j},y_{t,j})$, let $u\in\cA^{h-1}$ denote a prefix. Conditioning on $(x_{t,j},y_{t,j,<h})=(x,u)$ gives teacher weights $w_i(x,u)$. Hence, the next-token distribution is
\begin{align*}
    \PP\bigl(
        a_{t,j,h}=a
        \mid
        x_{t,j}=x,\,
        y_{t,j,<h}=u
    \bigr)
    &=
    \sum_{i\in\cI}
    w_i(x,u)p_i(a|x,u)=
    \pif^*(a|x,u).
\end{align*}
Thus, at every visited prefix, the observed next token is a sample from the corresponding token conditional of the forward-KL target, which suggests estimating each target conditional from the token frequencies observed at the corresponding context--prefix pair. When full logits are available, each sampled token can instead be replaced by the teacher's conditional probability vector. To express both updates in a common form, define
\begin{align}
\notag
    q_{t,j,h}(a):=
    \begin{cases}
        \ind(a_{t,j,h}=a)
        & \mathrm{w/o \ logit},\\
        \frac{\exp(Z_{t,j,h}(a))}
             {\sum_{b\in\cA}\exp(Z_{t,j,h}(b))},
        & \mathrm{w/ \ logit}.
    \end{cases}
\end{align}
We then define the per-round normalized counts as
\begin{align}
\label{eq:fkl-count}
    c_t(x,u,a)
    :=
    \frac{1}{m}\sum_{j=1}^m
    \ind(x_{t,j}=x,\,y_{t,j,<h}=u)\,
    q_{t,j,h}(a).
\end{align}
Here, the token position $h=|u|+1$ is determined by the prefix $u$ and is hidden from the notation $c_t(x,u,a)$. The algorithm accumulates these counts after each round according to
\begin{align*}
    W_t(x,u,a)
    &:=
    \sum_{s=1}^t c_s(x,u,a), \
    W_t(x,u)
    :=
    \sum_{a\in\cA}W_t(x,u,a),
\end{align*}
and outputs the conditional policy
\begin{align}
    \hat\pi_{t+1}(a|x,u)
    :=
    \frac{W_t(x,u,a)+1/2}
         {W_t(x,u)+A/2}.
    \label{eq:off-policy-smoothed-estimator}
\end{align}
Here we add a $(1/2)$-smoothing to each accumulated weight, following the spirit of the classical Krichevsky--Trofimov estimator \citep{krichevsky1981performance}, which prevents zero predicted probabilities and assigns the uniform distribution to unvisited states.
\begin{algorithm}[H]
\caption{Per-Prefix Forward KL}
    \begin{algorithmic}[1]\label{algo:ar-fkl}
    \STATE \textbf{Input:} Context set $\cX$, token set $\cA$, horizon $H$.
    \STATE Initialize $W_0(x,u,a)=0$; equivalently, $\hat\pi_1(a|x,u)=1/A$ at every active prefix.
    \FOR{$t=1,\ldots, T$}
    \STATE Observe the batch of teacher rollouts $\{(x_{t,j},y_{t,j})\}_{j=1}^{m}$ and, when available, the teacher logits $\{Z_{t,j,h}\}_{j,h}$.
    \STATE For every level-$h$ state $(x,u)$ and token $a$, set $c_t(x,u,a)$ as in \eqref{eq:fkl-count}.
    \STATE Update the weight $W_t(x,u,a) = W_{t-1}(x,u,a)+c_t(x,u,a)$ and $W_t(x,u) = \sum_{a} W_t(x,u,a)$.
    \STATE Let $\hat\pi_{t+1}(a|x,u)$ be as defined in \eqref{eq:off-policy-smoothed-estimator}.
    \ENDFOR
    \STATE Output $\{\hat\pi_t\}_{t=1}^T$.
    \end{algorithmic}
\end{algorithm}
\subsection{Theoretical Guarantees}
We now bound the cumulative forward-KL divergence from the target $\pif^*$ to the policies produced by Algorithm~\ref{algo:ar-fkl}. The following theorem gives logarithmic regret bounds under both feedback protocols, with expectation and probability taken over the teachers and rollouts sampled during training.
\begin{theorem}
\label{thm:ar-fkl-regret}
Let $T\geq2$, and define the regret in \eqref{eq:regret} with $D(p\|q)=\KL(q\|p)$. Under the off-policy distillation protocol, both with and without access to full teacher logits, Algorithm~\ref{algo:ar-fkl} satisfies
\begin{align}
\notag
    \EE\big[\operatorname{Regret}(T)\big]
    &
    % =\sum_{t=1}^T
    % \EE\Big[
    %     \EE_{x\sim\bar\rho} \
    %     \KL\big[
    %         \pif^*(\cdot|x)
    %         \big|
    %         \hat\pi_t(\cdot|x)
    %     \big]
    % \Big]
    \notag
    \le 4SA^H \log T.
\end{align}
Moreover, for any $\delta\in(0,1)$, with probability at least $1-\delta$, Algorithm~\ref{algo:ar-fkl} satisfies
\begin{align}
    \operatorname{Regret}(T)
    \leq
    8S A^H\log T
    +
    4H\bigl(1+\log(2T+A)\bigr)
    \log\frac1\delta.
\notag
\end{align}
\end{theorem}
For fixed $S$, $A$, and $H$, Theorem~\ref{thm:ar-fkl-regret} implies that the forward-KL divergence from $\pif^*$ to $\hat\pi_t$, averaged over contexts and training rounds, decays as $O(\log T/T)$ in expectation and with high probability. Thus, off-policy feedback allows the student to learn the forward-KL aggregate in this average-divergence sense. The factor $SA^H$ reflects the size of the effective state--action space: at step $h$, each context--prefix pair $(x,u)\in\cX\times\cA^{h-1}$ acts as a distinct state, yielding $SA^h$ state--action pairs. The exponential dependence on $H$ comes from the growth of the effective state space itself and the bound scales linearly with the total number of state--action pairs across all steps. %For the detailed proof, see Appendix~\ref{sec:proof-fkl-thm}.
%%%%%%%%%%%%%%%%%%%%%%%%%%%%%%%%%%%%%%%%%%%%%%%%%%%%%%%%%%%%%%%%%%%%%%%%
\section{On-policy Distillation \& Reverse KL}
\label{sec:ar-rkl-tabular}
In this section, we study how on-policy feedback can be used to learn the reverse-KL aggregate. With the same structure as the last section, we first characterize this target, then construct a tabular algorithm using teacher log-probabilities evaluated along student-generated responses. The following theorem gives the target distribution and its token conditionals.
\begin{theorem}[Reverse KL] 
\label{thm:obj-reverse-kl}
Let $w_i(x) = \rho_i(x)/[\sum_{j \in \cI} \rho_j(x)]$. Define the unnormalized token-level geometric mean
    \begin{align*}
        g_h(a|x,y_{<h}) &:= \prod_{i \in \cI} p_i(a|x,y_{<h})^{w_i(x)}.
    \end{align*}
We define the value function iteratively backward as follows: for any $x \in \cX$, $y \in \cA^H$,
\begin{align*}
    V_{H+1}(x,y) := 1, \qquad V_h(x,y_{<h}) &:= \sum_{a \in \cA} g_h(a|x,y_{<h})\, V_{h+1}\big(x,(y_{<h},a)\big), \qquad h = H,\ldots,1.
\end{align*}
Then the sequence-level reverse-KL minimizer is unique and is given by
\begin{align}
    \label{eq:seq-rkl}
    \pir^*(y|x)
    =
        \prod_{i\in\cI}  p_i(y|x)^{w_i(x)}/
        V_1(x)
    .
\end{align}
Moreover, we assume that at every prefix $y_{<h}$, $V_h(x,y_{<h})>0$. Then, the token-level distribution is unique, i.e., 
\begin{align}
\label{eq:ar-rkl}
    \pir^*(a|x,y_{<h}) = g_h(a|x,y_{<h})\cdot V_{h+1}\big(x,(y_{<h},a)\big)/V_h(x,y_{<h}).
\end{align}
\end{theorem}
\begin{remark}
In contrast with the last section, the teacher weights $w_i(x)$ depend only on the context. However, the token conditional is not obtained by normalizing $g_h$ alone: each token score is also multiplied by $V_{h+1}$, which sums the products of subsequent scores over all possible continuations. To summarize, the target accounts for both the current token and the continuations that follow it.
\end{remark}

\subsection{Algorithm Design}
To control the estimation error of teacher feedback, we assume that teacher token log-probabilities differ from those of a fixed reference policy by a bounded amount.
\begin{assumption}
\label{assump:token-bound}
Let $B>0$. There exists a fixed full-support reference policy $\piref$ such that, for every
$i\in\cI$, $h\in[H]$, and
$(x,u,a)\in\cX\times\cA^{h-1}\times\cA$,
\begin{align*}
    \big|
        \log
        [{p_i(a|x,u)}/
             {\piref(a|x,u)}]
    \big|
    \leq B.
\end{align*}
\end{assumption}
No such assumption is required for our forward-KL analysis, since the feedback consists of token indicators or teacher probabilities, both bounded in $[0,1]$. Here, we instead estimate teacher log-probabilities, which can be unbounded below. Assumption~\ref{assump:token-bound} imposes a two-sided bound on the token-level probability ratios between each teacher and the reference policy.

Let $\cF_t$ denote the observations available before round $t$, so that $\pi_t$ is $\cF_t$-measurable. For a student-generated rollout $(x_{t,j},y_{t,j})$, let $u\in\cA^{h-1}$ denote a prefix. At any visited context--prefix--token triple $(x,u,a)$, Bayes' rule gives
\begin{align*}
    &\PP\bigl(i_t=i\mid\cF_t,\,x_{t,j}=x,\,y_{t,j,<h}=u,\,a_{t,j,h}=a\bigr) = \frac{\rho_i(x)\pi_t(u|x)\pi_t(a|x,u)}
    {\sum_{k\in\cI}\rho_k(x)\pi_t(u|x)\pi_t(a|x,u)}
    = w_i(x).
\end{align*}
Given the context and past observations, the prefix and token are generated by the student policy, which is shared across teacher indices. Their probabilities therefore cancel, leaving the teacher weights unchanged.
Recall that the observed feedback is
$l_{t,j,h}=\log p_{i_t}(a|x,u)$. Its conditional mean satisfies
\begin{align}
    \EE\big[
        l_{t,j,h}
        \,\big|\,
        \cF_t,\,
        x_{t,j}=x,\,
        y_{t,j,<h}=u,\,
        a_{t,j,h}=a
    \big]
    &=
    \sum_{i\in\cI}
    w_i(x)\log p_i(a|x,u)=
    \log g_h(a|x,u).\notag
\end{align}
Visits to $(x,u,a)$ thus provide unbiased observations of the log-geometric score $\log g_h(a|x,u)$. Motivated by this, we estimate this score by averaging the feedback collected at the same triple. Define the cumulative visit count $N_t$, the cumulative feedback $S_t$, and the empirical mean $\bar l_t$ as
\begin{align}
\notag
    N_t(x,u,a)
    &:=
    \sum_{s=1}^t\sum_{j=1}^m
    \ind\bigl(
        x_{s,j}=x,\,
        y_{s,j,<h}=u,\,
        a_{s,j,h}=a
    \bigr),\\\notag
    S_t(x,u,a)
    &:=
    \sum_{s=1}^t\sum_{j=1}^m
    l_{s,j,h}
    \ind\bigl(
        x_{s,j}=x,\,
        y_{s,j,<h}=u,\,
        a_{s,j,h}=a
    \bigr),\\\label{eq:rkl-NSL}
    \bar l_t(x,u,a)
    &:=
    {S_t(x,u,a)}/\big(
         {N_t(x,u,a)\vee1}\big).       
\end{align}
To account for uncertainty in these empirical means, we construct optimistic estimates of the log-geometric scores. The following lemma bounds the estimation error at visited triples.
\begin{lemma}
\label{lem:ar-token-confidence}
Under Assumption~\ref{assump:token-bound}, fix any $\delta\in(0,1)$. With probability at least
$1-2\delta$, the following inequality holds for every $t\in[T]$ and every
$(x,u,a)$ with $N_t(x,u,a)\geq1$ simultaneously,
\begin{align}
    \big|
        \bar l_t(x,u,a)
        -
        \log g_h(a|x,u)
    \big|
    &\leq
    \beta_t(x,u,a), \notag
\end{align}
where $\beta_t(x,u,a):= \tilde O\big(B \sqrt{mH/(N_t(x,u,a) \vee 1)} + BmH/(N_t(x,u,a) \vee 1)\big)$.
\end{lemma}
As a result, we can define the optimistic estimate
% \begin{align*}
%     \hat l_t(x,u,a)
%     :=
%     \begin{cases}
%         \log\piref(a|x,u)
%         +
%         \big[
%             \bar l_t(x,u,a)
%             -\log\piref(a|x,u)
%             +\beta_t(x,u,a)
%         \big]_{[-B,B]},
%         & N_t(x,u,a)\geq1,\\[1mm]
%         \log\piref(a|x,u)+B,
%         & N_t(x,u,a)=0.
%     \end{cases}
% \end{align*}
\begin{align}
    \hat l_t(x,u,a)
    := \begin{cases}
        \bar l_t(x,u,a) + \min \big\{\beta_t(x,u,a), 2B\big\}& N_t(x,u,a) \ge 1,\\
         \log \piref(a|x,u) + B & N_t(x,u,a) = 0.
        \end{cases}
    \label{eq:optimistic-token-score}
\end{align}
At visited triples, Lemma~\ref{lem:ar-token-confidence} ensures that $\hat l_t(x,u,a)\geq\log g_h(a|x,u)$ on the confidence event. At unvisited triples, this inequality follows directly from Assumption \ref{assump:token-bound}. We then substitute these optimistic scores into the backward recursion defining the reverse-KL target in \eqref{eq:ar-rkl}:
\begin{align}
    \hat g_{t,h}(a|x,u)
    &:=
    \exp\bigl(\hat l_t(x,u,a)\bigr), \
    \hat V_{t,H+1}(x,y)
    :=
    1,\notag\\
    \hat V_{t,h}(x,u)
    &:=
    \sum_{a\in\cA}
    \hat g_{t,h}(a|x,u)
    \hat V_{t,h+1}\bigl(x,(u,a)\bigr),
    \qquad h=H,\ldots,1.
\notag
\end{align}
Therefore, for any $h$, any context $x$ and prefix $y_{<h} = u$, the output policy for the next token is defined as
\begin{align}
    \hat\pi_{t+1}(a|x,u)
    :=
    {
        \hat g_{t,h}(a|x,u)
        \hat V_{t,h+1}\bigl(x,(u,a)\bigr)
    }/{
        \hat V_{t,h}(x,u)
    }.
    \label{eq:plugin-policy}
\end{align}
The algorithm is outlined in Algorithm~\ref{algo:ar-rkl}.
\begin{algorithm}[H]
\caption{Per-Prefix Reverse KL}
\label{algo:ar-rkl}
\begin{algorithmic}[1]
    \STATE \textbf{Input:} Context set $\cX$, token set $\cA$, $H$, $B$, $\delta$, and $T$.
    \STATE Initialize 
    $\hat\pi_1(a|x,u)=1/A$ for any $(x,u,a)$.
    \FOR{$t=1,\ldots,T$}
        \STATE Observe the contexts
        $\{x_{t,j}\}_{j=1}^m$ generated from $\rho_{i_t}$.
        \STATE For each $j$, generate $y_{t,j}\sim\hat\pi_t(\cdot|x_{t,j})$
        and observe $\{l_{t,j,h}\}_{h=1}^H$.
        \STATE For each triple $(x,u,a)$, update $N_t(x,u,a)$, $S_t(x,u,a)$, and $\bar l_t(x,u,a)$ as in \eqref{eq:rkl-NSL}.
        \STATE Construct the optimistic estimate $\hat l_t$ according to
        \eqref{eq:optimistic-token-score}.
        \STATE Define $\hat\pi_{t+1}$ according to
        \eqref{eq:plugin-policy}.
    \ENDFOR
    \STATE Output $\{\hat\pi_t\}_{t=1}^T$.
\end{algorithmic}
\end{algorithm}
\subsection{Theoretical Guarantees}
We now bound the regret induced by Algorithm~\ref{algo:ar-rkl}. 
\begin{theorem}
\label{thm:ar-rkl-regret}
Let $T\geq2$, and define the regret in \eqref{eq:regret} with $D(p\|q)=\KL(p\|q)$. Under the on-policy distillation protocol and Assumption~\ref{assump:token-bound}, for any $\delta\in(0,1/3)$, with probability at least $1-3\delta$, Algorithm~\ref{algo:ar-rkl} satisfies
\begin{align*}
    \operatorname{Regret}(T)
    &\le \tilde O\big(B^2H^2SA^H \log T\big).
\end{align*}
\end{theorem}

For fixed $S$, $A$, $H$, $B$, and $\delta$, Theorem~\ref{thm:ar-rkl-regret} implies that the reverse-KL divergence from $\hat\pi_t$ to $\pir^*$, averaged over contexts and training rounds, decays as $\widetilde O(\log T/T)$ with high probability. Thus, on-policy feedback allows the student to learn the reverse-KL aggregate in this average-divergence sense. As in the off-policy setting, the factor $SA^H$ reflects the size of the effective state--action space. For the detailed proof, see Appendix~\ref{sec:proof-rkl-thm}.
\section{Comparison of Distinct Aggregation Objectives}
\label{sec:ar-target-comparison}
The preceding sections establish learning guarantees that connect off-policy distillation to forward-KL aggregation and on-policy distillation to reverse-KL aggregation. We now examine the differences between these two forms of aggregation and their implications for combining the capabilities of multiple teachers in a single policy.

% Throughout this section, fix a context $x$ with $\bar\rho(x)>0$ and write $w_i=w_i(x)$. We assume that $w_i>0$ for every $i\in\cI$ and that all teacher logits are finite, so that $p_i(y|x)>0$ for every $i\in\cI$ and $y\in\cY$.
\subsection{Uninformative Teachers}
Consider a context $x$ that falls within teacher $k$'s area of expertise but is unfamiliar to the remaining teachers. We represent this difference in coverage by assuming that $\rho_k(x)$ is large relative to $\sum_{i\neq k}\rho_i(x)$. The expert therefore receives a large aggregation weight $\alpha:=w_k(x) \in (0,1)$.
% \[
%     \alpha:=w_k(x)
%     =\frac{\rho_k(x)}{\rho_k(x)+\sum_{i\neq k}\rho_i(x)}
%     \in(0,1).
% \]

Write $p(\cdot|x):=p_k(\cdot|x)$ and $N:=|\cY|\geq2$, and let $q:=p(y^\star|x)$ denote the expert's probability of a desirable response $y^\star$. A value of $q$ close to one means that the expert already produces the correct response with high probability. We ask whether the aggregate can maintain this high probability when the other teachers are uninformative. To represent the remaining teachers' lack of information at $x$, we model their response distributions as uniform: $p_i(y|x)=1/N$ for every $i\neq k$ and $y\in\cY$. These teachers have total weight $1-\alpha$ and are equivalent, under either objective, to a single uniform teacher with this weight. The two aggregation targets are therefore
\begin{align}
% \label{eq:uniform-teacher-targets}
\notag
    \pif^*(y|x)
    &=\alpha p(y|x)+\frac{1-\alpha}{N},\quad
    \pir^*(y|x)
    =\frac{p(y|x)^\alpha}{\sum_{z\in\cY}p(z|x)^\alpha}.
\end{align}
For a concrete comparison, suppose $y^\star$ is the only correct response and $q\in(1/N,1)$, with the expert assigning uniform probability to each incorrect response. The following proposition shows that, for fixed teacher weights and response space, the comparison is determined by a confidence threshold.
\begin{proposition}
\label{prop:uniform-teacher-threshold}
Fix $N\geq2$ and $\alpha\in(0,1)$, and consider the teacher distributions above. If $N\geq3$, there exists a unique $q_{\mathrm c}=q_{\mathrm c}(N,\alpha)\in(1/N,1)$ such that, for every $q\in(1/N,1)$,
\[
    \pir^*(y^\star|x)>\pif^*(y^\star|x)
    \quad\Longleftrightarrow\quad q>q_{\mathrm c}.
\]
\end{proposition}
The proof is given in Appendix~\ref{sec:proof-uniform-threshold}. Figure~\ref{fig:teacher-aggregation}(a) plots the two correct-response probabilities for $q>0.5$, with $N=10$ and $\alpha=0.9$ fixed.
\begin{figure}[!htbp]
    \centering
    \includegraphics[width=0.88\linewidth]{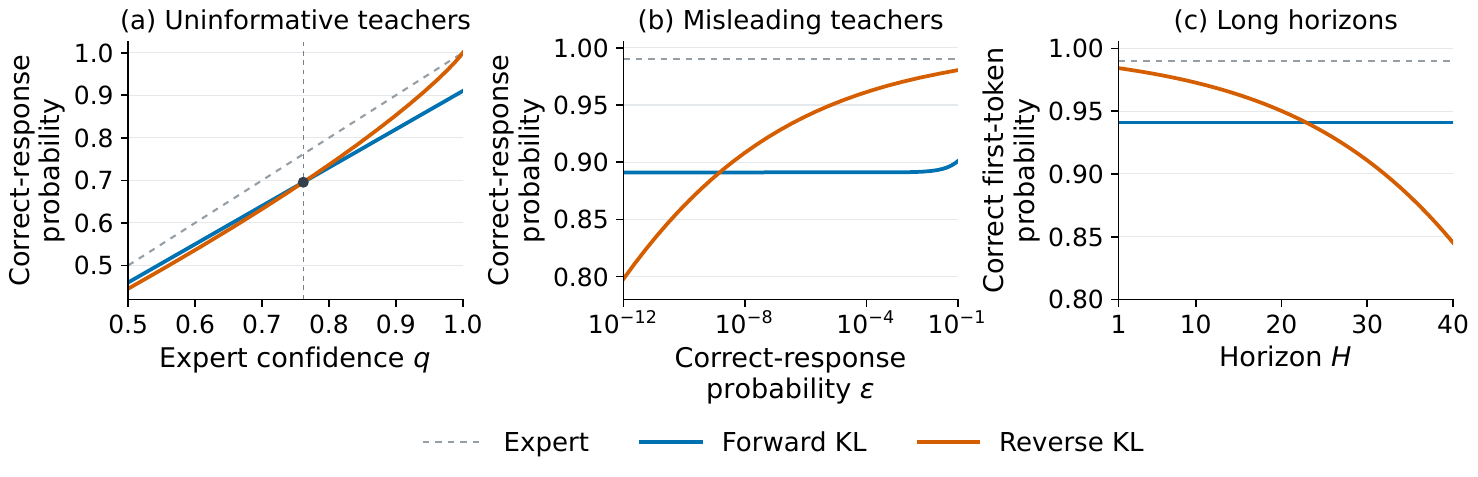}
    \caption{Forward- and reverse-KL aggregation with expert weight $\alpha=0.9$. Panels (a,b) use $N=10$ responses, with each teacher distributing its remaining probability uniformly over incorrect responses. (a) The other teachers are uniform, and expert confidence $q$ varies above $0.5$; the targets cross at $q_{\mathrm c}\approx0.7617$. (b) The expert's probability is fixed at $q=0.99$, while the other teacher's correct-response probability $\varepsilon$ varies on a logarithmic axis. (c) For binary responses, the horizon $H$ varies while all token probabilities remain fixed. The expert chooses the correct first token $a$ with probability $0.99$; subsequent tokens have probabilities $(0.99,0.01)$ after $a$ and $(1/2,1/2)$ after $b$. The other teacher is uniform at every prefix.}
    \label{fig:teacher-aggregation}
    \vspace{-5mm}
\end{figure}
% Both targets reduce the expert's correct-response probability. Below the threshold, forward aggregation retains more of this probability; above it, reverse aggregation does. Thus, a large expert weight alone does not ensure an advantage for reverse aggregation: the expert's confidence also matters.
\begin{remark}
This setting illustrates how forgetting can arise when a student, after acquiring an expert teacher's capability, continues learning from other teachers that are uninformative on the same context. Proposition~\ref{prop:uniform-teacher-threshold} identifies when this reduction is smaller under reverse aggregation than under forward aggregation, offering a possible explanation for empirical findings that OPD is less susceptible to forgetting than SFT. For example, \citet{shenfeld2026self} show in both single-task and sequential continual-learning experiments that OPD preserves general and previously acquired capabilities while learning new tasks, whereas SFT leads to substantial performance regression. Similar observations are reported by \citet{hubotter2026reinforcement} and \citet{ma2026one}, and have also been made in the field of vision language action model \citep{zhong2026vla}.
\end{remark}
\subsection{Misleading Teachers}
The preceding examples model teachers that are unfamiliar with a context as uniform. In this section, we consider another scenario where teachers may be misleading. As the small aggregation weights of some teachers reflect their limited coverage of the context, their predictions on such unfamiliar contexts may be highly inaccurate. Thus, they may assign extremely low probability to the correct response. We show that forward aggregation preserves a high probability of the correct response when the expert is highly confident and receives most of the weight. Under reverse aggregation, however, even a teacher with a small weight can drive this probability arbitrarily close to zero by assigning sufficiently low probability to the correct response.

Figure~\ref{fig:teacher-aggregation}(b) illustrates this effect by varying only the misleading teacher's correct-response probability $\varepsilon$, while fixing the expert's probability at $q=0.99$ and the teacher weights at $0.9$ and $0.1$. As $\varepsilon$ decreases toward zero, the reverse target approaches zero, whereas the forward target remains bounded below by the expert's weighted contribution. The following proposition extends this comparison to a set of correct responses.
\begin{proposition}
\label{prop:reverse-extreme-value}
Let $\varnothing\neq E\subsetneq\cY$ be a set of correct responses. For any teacher $k$, the forward target satisfies
\[
    \pif^*(E|x)\geq w_kp_k(E|x).
\]
In contrast, for any $\alpha,q,\eta\in(0,1)$, there exist two teachers with weights $w_1=\alpha$ and $w_2=1-\alpha$, each assigning positive probability to every response, such that
\[
    p_1(E|x)=q,
    \qquad
    \pir^*(E|x)<\eta.
\]
\end{proposition}
Intuitively, forward aggregation is a weighted sum of nonnegative probabilities, so other teachers cannot cancel the expert's contribution to correct responses. Reverse aggregation instead averages log-probabilities, which can be arbitrarily negative. Consequently, a teacher that assigns sufficiently low probability to correct responses relative to incorrect ones can override the expert's preference, even with a small aggregation weight.

% \begin{proof}
% The forward bound follows from the nonnegative terms in \eqref{eq:seq-fkl}. For the reverse target, assign total probabilities $q$ and $\varepsilon$ to $E$ under teachers $1$ and $2$, respectively, distributing each teacher's probability uniformly within $E$ and within its complement. The resulting $\pir^*(E|x)$ has the same expression as \eqref{eq:misleading-teacher-targets} for $\pir^*(y^\star|x)$ and tends to zero as $\varepsilon\downarrow0$. Choosing a sufficiently small positive $\varepsilon$ proves the claim. Each instance has finite logits; the limit does not impose a common bound on the teachers' log-probability ratios.
% \end{proof}
\begin{remark}
The sensitivity identified by Proposition~\ref{prop:reverse-extreme-value} offers a possible explanation for some empirical failures of OPD under teacher--student mismatch. Such failures can occur even with a stronger teacher: \citet{ma2026mopd} find that a stronger but distributionally different teacher produces predominantly negative feedback on student rollouts and degrades student performance. Related empirical remedies address either extreme feedback or the initial mismatch. \citet{wang2026demystifying} show that clipping or compressing extreme token-level log-ratios can improve training stability, while \citet{li2026rethinking} find that an SFT cold start on teacher-generated responses reduces the initial pattern mismatch and improves subsequent OPD.
\end{remark}
\subsection{Long Horizons}
The preceding comparisons focus on the probabilities of complete responses. During autoregressive generation, however, reverse aggregation can favor a prefix that the expert considers less likely, even while preserving the expert's ranking of complete responses. This can steer generation toward an incorrect trajectory.

Consider an expert teacher $p_1$ with weight $\alpha\in(0,1)$ and a uniform teacher $p_2$ with weight $\beta=1-\alpha$. Since $p_2(y|x)=A^{-H}$ is the same for every complete response, its contribution to the geometric aggregate cancels upon normalization, leaving
\begin{align}
    \pir^*(y|x)
    ={p_1(y|x)^\alpha}/\big[{\textstyle\sum_{z\in\cY}p_1(z|x)^\alpha}\big].
    \notag
    % \label{eq:reverse-sequence-softening}
\end{align}
Raising probabilities to the power $\alpha<1$ preserves the ranking of complete responses while increasing the relative weight of less likely responses. Token-level probabilities, however, also depend on a sum over possible continuations, as shown by \eqref{eq:ar-rkl}. For a prefix $u$ of length $h-1$ and a candidate token $c$,
\begin{align}
    \pir^*(c|x,u)
    \propto p_1(c|x,u)^\alpha
    \sum_{v\in\cA^{H-h}}p_1(v|x,u,c)^\alpha.
    \notag
    % \label{eq:reverse-continuation-factor}
\end{align}
Since $\alpha<1$, spreading probability mass more evenly across continuations increases this sum. The continuation factor can therefore favor tokens with more diffuse continuations over those whose continuations are concentrated on a few responses.

To further illustrate this effect, consider responses of length $H$ over $\cA=\{a,b\}$, where $a$ is the correct first token. The expert $p_1$ has weight $\alpha=0.9$ and selects $a$ at the first step with probability $r=0.99$. The other teacher $p_2$ has weight $\beta=0.1$ and assigns probability $1/2$ to each token at every prefix.

At each of the remaining $H-1$ positions, the expert's token distribution depends only on the first token. After an initial $a$, it assigns probability $0.99$ to $a$ and $0.01$ to $b$, regardless of the intervening tokens. After an initial $b$, it assigns probability $1/2$ to each token. Thus, the expert's continuations are concentrated after $a$ and uniform after $b$, while all token probabilities remain strictly positive. Under forward aggregation, the probability of choosing $a$ at the first step is the weighted average of the teachers' corresponding probabilities and is therefore independent of the continuation distributions and the horizon. Figure~\ref{fig:teacher-aggregation}(c) plots the probability of choosing the correct first token $a$ as $H$ varies from $1$ to $40$, keeping the teacher weights and token-level probabilities fixed.

More generally, the following proposition shows that this reversal can occur even when the uniform teacher receives an arbitrarily small weight.
\begin{proposition}
\label{prop:reverse-long-horizon}
Fix $\beta\in(0,1/2)$ and $r\in(1/2,1)$. For every sufficiently large horizon $H$, there exist two teacher policies $p_1$ and $p_2$ on $\cY=\{a,b\}^H$, each assigning positive probability to every response, such that $p_1(a|x)=r$ and $p_2$ is uniform at every prefix. With aggregation weights $w_1(x)=1-\beta$ and $w_2(x)=\beta$, their forward and reverse targets satisfy
\[
    \pif^*(a|x)>\pif^*(b|x),
    \qquad
    \pir^*(a|x)<\pir^*(b|x).
\]
\end{proposition}
\begin{remark}
The reversal in Proposition~\ref{prop:reverse-long-horizon} illustrates a potential difficulty for OPD on long responses: increasing the horizon can change the reverse aggregation target's preferred prefix while the expert's preference remains fixed. Empirically, \citet{li2026rethinking} find that longer rollouts can exhibit late-stage collapse, with instability emerging at later tokens and spreading toward earlier positions. \citet{ziheng2026less} find that teachers are less able to recover correct answers from longer student-generated prefixes and that stopping rollouts early can improve training stability.
\end{remark}
\section{Conclusion}
We studied off- and on-policy distillation in a sequential multi-teacher framework. By characterizing the forward- and reverse-KL targets and establishing learning guarantees, we connected the two feedback protocols to distinct aggregation objectives. Our comparisons show that reverse-KL aggregation can better preserve a confident expert's preference in the presence of uniform teachers, but can suppress correct responses under misleading feedback or favor incorrect prefixes over long horizons. These results offer possible explanations for the benefits and fragility of OPD relative to SFT, and provide a perspective on the difference between SFT and OPD through the lens of distinct aggregation objectives.

%%%%%%%%%%%%%%%%%%%%%%%%%%%%%%%%%%%%%%%%%%%%%%%%%%%%%%%%%%%%%%%%%%%%%%%%

\subsection*{AI use statement}

This work builds on the authors' research intuition and ideas. We used AI models, including GPT 5.6 Sol, GPT 6 Astra, Claude Fable 5, and Grok 4.6, to help articulate and summarize these ideas and to draft and polish the manuscript. These models also assisted in proving several key technical lemmas and developing proof strategies. The authors verified the proofs and reviewed the AI-assisted writing, and take full responsibility for the correctness and final content of the paper.

% \subsection*{Ethics statement}

% This work presents a theoretical analysis of knowledge distillation and does not involve human subjects or the collection of personal data. We do not identify specific ethical concerns arising directly from the analysis. Applications of knowledge distillation may nevertheless transfer biases or undesirable behaviors from teacher models to students; our theoretical guarantees concern distributional aggregation and learning, and do not establish the safety or fairness of the resulting models.

% \subsection*{Reproducibility statement}

% This work is theoretical. All assumptions and algorithms are stated explicitly, and complete proofs are provided in the appendices. All figures are deterministic evaluations of closed-form expressions, with parameter values given in the text and captions.
% \tableofcontents



\appendix
\section{Proof of the KL Aggregation Theorems}
\label{sec:autoregressive}

In this section, we prove Theorems~\ref{thm:obj-forward-kl}
and~\ref{thm:obj-reverse-kl}. We first reduce the objective to a
separate optimization problem at each context. Recall that
\begin{align*}
    \bar\rho(x):=\frac1I\sum_{i\in\cI}\rho_i(x),
    \qquad
    w_i(x):=\frac{\rho_i(x)}{\sum_{j\in\cI}\rho_j(x)}
\end{align*}
for every $x$ with $\bar\rho(x)>0$. Define
\begin{align*}
    \cI_x^+&:=\{i\in\cI:w_i(x)>0\}, \ 
    L_{D,x}(\pi):=\sum_{i\in\cI_x^+}w_i(x)
    D\bigl(\pi(\cdot\mid x)\,\big\|\,p_i(\cdot\mid x)\bigr).
\end{align*}
Using $\rho_i(x)=I\bar\rho(x)w_i(x)$, we obtain
\begin{align}
    &\sum_{i\in\cI}\EE_{x\sim\rho_i}
    \bigl[D\bigl(\pi(\cdot\mid x)\,\big\|\,p_i(\cdot\mid x)\bigr)\bigr] =I\sum_{x:\bar\rho(x)>0}\bar\rho(x)L_{D,x}(\pi).
    \label{eq:agg-proof-context-reduction}
\end{align}
The optimization domain in \eqref{eq:average-objective} allows a
separate distribution in $\Delta(\cY)$ at every context. Therefore,
it suffices to minimize $L_{D,x}$ at each context with
$\bar\rho(x)>0$. Contexts with $\bar\rho(x)=0$ do not affect the
objective, and their policies may be chosen arbitrarily.
Throughout the proofs, zero-weight teachers are omitted from
geometric products. We use the conventions $0\log(0/q)=0$
for $q\geq0$ and $p\log(p/0)=+\infty$ for $p>0$.

\subsection{Proof of Theorem~\ref{thm:obj-forward-kl}}
\label{app:proof-obj-forward-kl}

\begin{proof}[Proof of Theorem~\ref{thm:obj-forward-kl}]
Fix a context $x$ with $\bar\rho(x)>0$. For forward KL, the
contextwise objective in~\eqref{eq:agg-proof-context-reduction} is
\begin{align*}
    L_{\mathrm F,x}(\pi)
    :=\sum_{i\in\cI_x^+}w_i(x)
    \KL\bigl(p_i(\cdot\mid x)\,\big\|\,\pi(\cdot\mid x)\bigr).
\end{align*}
We first identify its minimizing sequence distribution, and then
derive the corresponding token conditionals.

\textbf{Sequence-level minimizer.}
Define the candidate distribution
\begin{align*}
    q_{\mathrm F}(y\mid x)
    :=\sum_{i\in\cI_x^+}w_i(x)p_i(y\mid x),
    \qquad y\in\cY.
\end{align*}
Since the weights are nonnegative and sum to one,
$q_{\mathrm F}(\cdot\mid x)\in\Delta(\cY)$. Let
\begin{align*}
    \mathcal S_{\mathrm F}(x)
    &:=\{y\in\cY:q_{\mathrm F}(y\mid x)>0\},\\*
    C_{\mathrm F}(x)
    &:=\sum_{i\in\cI_x^+}w_i(x)
    \KL\bigl(p_i(\cdot\mid x)\,\big\|\,q_{\mathrm F}(\cdot\mid x)\bigr).
\end{align*}
For every $i\in\cI_x^+$, we have
$q_{\mathrm F}(y\mid x)\geq w_i(x)p_i(y\mid x)$. Consequently,
$C_{\mathrm F}(x)$ is finite and does not depend on $\pi$.
For a policy that is positive on $\mathcal S_{\mathrm F}(x)$,
expanding the KL divergences gives
\begin{align}
    L_{\mathrm F,x}(\pi)-C_{\mathrm F}(x)
    &=\sum_{i\in\cI_x^+}w_i(x)
      \sum_{y\in\mathcal S_{\mathrm F}(x)}
      p_i(y\mid x)
      \log\frac{q_{\mathrm F}(y\mid x)}{\pi(y\mid x)}\notag\\*
    &=\sum_{y\in\mathcal S_{\mathrm F}(x)}
      q_{\mathrm F}(y\mid x)
      \log\frac{q_{\mathrm F}(y\mid x)}{\pi(y\mid x)}\notag\\*
    &=\KL\bigl(q_{\mathrm F}(\cdot\mid x)
      \,\big\|\,\pi(\cdot\mid x)\bigr),
    \label{eq:agg-proof-forward-decomposition}
\end{align}
where the second equality uses the definition of $q_{\mathrm F}$.
If $\pi$ assigns zero probability to some
$y\in\mathcal S_{\mathrm F}(x)$, at least one positive-weight
teacher assigns positive probability to that response. In that
case, both sides of \eqref{eq:agg-proof-forward-decomposition} are
$+\infty$, so the identity remains valid in the extended-real sense.

By nonnegativity of KL divergence, the right-hand side of
\eqref{eq:agg-proof-forward-decomposition} is minimized at zero, with
equality if and only if
$\pi(\cdot\mid x)=q_{\mathrm F}(\cdot\mid x)$. Thus, the minimizing
sequence distribution is unique at this context and satisfies
\begin{align*}
    \pif^*(y\mid x)
    =q_{\mathrm F}(y\mid x)
    =\sum_{i\in\cI}w_i(x)p_i(y\mid x),
\end{align*}
which proves \eqref{eq:seq-fkl}.

\textbf{Token-level conditionals.}
Fix $h\in[H]$ and $u\in\cA^{h-1}$. Marginalizing the sequence
distribution over all continuations gives
\begin{align}
    \pif^*(u\mid x)
    &=\sum_{i\in\cI}w_i(x)p_i(u\mid x),\notag\\*
    \pif^*((u,a)\mid x)
    &=\sum_{i\in\cI}w_i(x)p_i(u\mid x)p_i(a\mid x,u).
    \label{eq:agg-proof-forward-prefix}
\end{align}
The first equality follows from linearity of marginalization;
the second additionally uses the autoregressive factorization
of each teacher.
Suppose that $\sum_i\rho_i(x)p_i(u\mid x)>0$. Then
$\pif^*(u\mid x)>0$, and taking the ratio of the two prefix
probabilities yields
\begin{align}
    \pif^*(a\mid x,u)
    &=\frac{\pif^*((u,a)\mid x)}{\pif^*(u\mid x)}\notag\\*
    &=\frac{\sum_{i\in\cI}w_i(x)p_i(u\mid x)p_i(a\mid x,u)}
            {\sum_{j\in\cI}w_j(x)p_j(u\mid x)}\notag\\*
    &=\sum_{i\in\cI}
      \frac{\rho_i(x)p_i(u\mid x)}
           {\sum_{j\in\cI}\rho_j(x)p_j(u\mid x)}
      p_i(a\mid x,u),
    \label{eq:agg-proof-forward-conditional}
\end{align}
where the last equality substitutes the definition of $w_i(x)$
and cancels its common normalization factor. This is exactly
\eqref{eq:ar-fkl}. At a prefix with $\pif^*(u\mid x)=0$, the
conditional may be chosen arbitrarily without changing the
sequence distribution. Since $x$ was arbitrary, combining the
contextwise minimizers using \eqref{eq:agg-proof-context-reduction}
completes the proof.
\end{proof}

\subsection{Proof of Theorem~\ref{thm:obj-reverse-kl}}
\label{app:proof-obj-reverse-kl}

\begin{proof}[Proof of Theorem~\ref{thm:obj-reverse-kl}]
Fix a context $x$ with $\bar\rho(x)>0$. The contextwise
reverse-KL objective is
\begin{align*}
    L_{\mathrm R,x}(\pi)
    :=\sum_{i\in\cI_x^+}w_i(x)
    \KL\bigl(\pi(\cdot\mid x)\,\big\|\,p_i(\cdot\mid x)\bigr).
\end{align*}
Define the unnormalized sequence score and its total mass by
\begin{align*}
    G_x(y):=\prod_{i\in\cI_x^+}p_i(y\mid x)^{w_i(x)},
    \qquad Z_x:=\sum_{y\in\cY}G_x(y).
\end{align*}
We first show that $Z_x=V_1(x)$, then identify the minimizing
sequence distribution and its token conditionals.

\textbf{Backward recursion and normalization.}
By the autoregressive factorization of the teachers,
\begin{align}
    G_x(y)
    &=\prod_{i\in\cI_x^+}
      \left(\prod_{h=1}^H p_i(a_h\mid x,y_{<h})\right)^{w_i(x)}
      \notag\\*
    &=\prod_{h=1}^H\prod_{i\in\cI_x^+}
      p_i(a_h\mid x,y_{<h})^{w_i(x)}\notag\\*
    &=\prod_{h=1}^H g_h(a_h\mid x,y_{<h}).
    \label{eq:agg-proof-reverse-factorization}
\end{align}
For every $h\in[H+1]$ and $u\in\cA^{h-1}$, we claim that
\begin{align}
    V_h(x,u)
    =\sum_{a_h,\ldots,a_H\in\cA}
      \prod_{k=h}^H g_k(a_k\mid x,y_{<k}),
    \qquad y=(u,a_h,\ldots,a_H).
    \label{eq:agg-proof-continuation-value}
\end{align}
At $h=H+1$, the sum contains one empty continuation and its
empty product is one, so the identity agrees with
$V_{H+1}(x,y)=1$. Suppose that it holds at level $h+1$.
Using the defining recursion for $V_h$, we obtain
\begin{align*}
    V_h(x,u)
    &=\sum_{a_h\in\cA}g_h(a_h\mid x,u)
      V_{h+1}(x,(u,a_h))\\*
    &=\sum_{a_h\in\cA}g_h(a_h\mid x,u)
      \sum_{a_{h+1},\ldots,a_H\in\cA}
      \prod_{k=h+1}^H g_k(a_k\mid x,y_{<k})\\*
    &=\sum_{a_h,\ldots,a_H\in\cA}
      \prod_{k=h}^H g_k(a_k\mid x,y_{<k}),
\end{align*}
where the second equality uses the induction hypothesis.
This proves \eqref{eq:agg-proof-continuation-value} by backward
induction. At the empty prefix, it gives
\begin{align}
    V_1(x)=\sum_{y\in\cY}G_x(y)=Z_x.
    \label{eq:agg-proof-reverse-normalizer}
\end{align}
In particular, the positivity condition $V_1(x)>0$ implies
$Z_x>0$.

\textbf{Sequence-level minimizer.}
Define
\begin{align*}
    q_{\mathrm R}(y\mid x):=\frac{G_x(y)}{Z_x},
    \qquad
    \mathcal S_{\mathrm R}(x):=\{y\in\cY:G_x(y)>0\}.
\end{align*}
By \eqref{eq:agg-proof-reverse-normalizer},
$q_{\mathrm R}(\cdot\mid x)$ is a probability distribution.
The set $\mathcal S_{\mathrm R}(x)$ is the common support of
all positive-weight teachers. Any policy assigning positive
mass outside this set has infinite reverse KL to at least one
such teacher, and therefore incurs infinite objective value.
For a policy supported on $\mathcal S_{\mathrm R}(x)$, we have
\begin{align}
    L_{\mathrm R,x}(\pi)
    &=\sum_{y\in\mathcal S_{\mathrm R}(x)}\pi(y\mid x)
      \left[\log\pi(y\mid x)
      -\sum_{i\in\cI_x^+}w_i(x)\log p_i(y\mid x)\right]
      \notag\\*
    &=\sum_{y\in\mathcal S_{\mathrm R}(x)}\pi(y\mid x)
      \log\frac{\pi(y\mid x)}{G_x(y)}\notag\\*
    &=\sum_{y\in\mathcal S_{\mathrm R}(x)}\pi(y\mid x)
      \log\frac{\pi(y\mid x)}{q_{\mathrm R}(y\mid x)}-\log Z_x
      \notag\\*
    &=\KL\bigl(\pi(\cdot\mid x)
      \,\big\|\,q_{\mathrm R}(\cdot\mid x)\bigr)-\log V_1(x).
    \label{eq:agg-proof-reverse-decomposition}
\end{align}
Here, the second equality uses the definition of $G_x$, the
third uses $G_x=Z_xq_{\mathrm R}$ and
$\sum_y\pi(y\mid x)=1$, and the last uses
\eqref{eq:agg-proof-reverse-normalizer}. The same identity holds
in the extended-real sense when $\pi$ puts positive mass
outside $\mathcal S_{\mathrm R}(x)$, since both sides are
then infinite.

Since $-\log V_1(x)$ does not depend on $\pi$, nonnegativity
of KL divergence shows that the unique minimizing sequence
distribution is
\begin{align*}
    \pir^*(y\mid x)
    =q_{\mathrm R}(y\mid x)
    =\frac{\prod_{i\in\cI_x^+}p_i(y\mid x)^{w_i(x)}}{V_1(x)},
\end{align*}
which proves \eqref{eq:seq-rkl}.

\textbf{Token-level conditionals.}
Fix $h\in[H]$ and a prefix $u=(u_1,\ldots,u_{h-1})$, and let
\begin{align*}
    G_{<h}(x,u):=\prod_{k=1}^{h-1}g_k(u_k\mid x,u_{<k}),
\end{align*}
where the empty product is one. Combining
\eqref{eq:agg-proof-reverse-factorization} and
\eqref{eq:agg-proof-continuation-value}, and summing over all
continuations, gives
\begin{align}
    \pir^*(u\mid x)
    &=\frac{G_{<h}(x,u)V_h(x,u)}{V_1(x)},\notag\\*
    \pir^*((u,a)\mid x)
    &=\frac{G_{<h}(x,u)g_h(a\mid x,u)
      V_{h+1}(x,(u,a))}{V_1(x)}.
    \label{eq:agg-proof-reverse-prefix}
\end{align}
If $\pir^*(u\mid x)>0$, then both $G_{<h}(x,u)$ and
$V_h(x,u)$ are positive. Taking the ratio of the two
prefix probabilities yields
\begin{align}
    \pir^*(a\mid x,u)
    &=\frac{\pir^*((u,a)\mid x)}{\pir^*(u\mid x)}\notag\\*
    &=\frac{g_h(a\mid x,u)V_{h+1}(x,(u,a))}{V_h(x,u)},
    \label{eq:agg-proof-reverse-conditional}
\end{align}
which is \eqref{eq:ar-rkl}. These conditionals are uniquely
determined at every positive-probability prefix.

At a null prefix with $V_h(x,u)>0$, the right-hand side of
\eqref{eq:agg-proof-reverse-conditional} still defines a valid
conditional, because its numerator sums to $V_h(x,u)$.
When $V_h(x,u)=0$, choose an arbitrary conditional instead.
To verify that these choices induce the minimizing sequence
law, take any $y\in\mathcal S_{\mathrm R}(x)$. All factors
and continuation values along this response are positive,
so the selected conditionals telescope:
\begin{align*}
    \prod_{h=1}^H
    \frac{g_h(a_h\mid x,y_{<h})V_{h+1}(x,y_{\leq h})}
         {V_h(x,y_{<h})}
    &=\frac{\prod_{h=1}^H g_h(a_h\mid x,y_{<h})}{V_1(x)}\\*
    &=\frac{G_x(y)}{Z_x}
    =\pir^*(y\mid x),
\end{align*}
where $y_{\leq h}:=(a_1,\ldots,a_h)$ and the first equality
uses $V_{H+1}(x,y)=1$.
These probabilities already sum to one over
$\mathcal S_{\mathrm R}(x)$. Since every selected token
conditional is normalized, the induced autoregressive law
assigns zero mass to all other responses and equals
$\pir^*(\cdot\mid x)$.
Thus, uniqueness concerns the sequence distribution and its
conditionals at positive-probability prefixes, not arbitrary
conditionals at null prefixes. Applying
\eqref{eq:agg-proof-context-reduction} across contexts completes
the proof.
\end{proof}

\section{Proof of Theorem~\ref{thm:ar-fkl-regret}}
\label{sec:proof-fkl-thm}
\begin{proof}
In this section, we prove Theorem~\ref{thm:ar-fkl-regret}.
Let $\cF_t:=\sigma\big(
        \{
            x_{s,j},\,y_{s,j},\,q_{s,j,h}(a)
            :
            1\leq s<t,\,
            j\in[m],\,
            h\in[H],\,
            a\in\cA
        \}
    \big)$. Then, $\hat\pi_t$ is $\cF_t$-measurable. 
    
Fix a level $h$, a state $(x,u)\in\cX\times\cA^{h-1}$, and a token
$a\in\cA$. Under either feedback model,
\begin{align}
    \EE\left[
        q_{t,j,h}(a)
        \,\middle|\,
        \cF_t,\,
        i_t=i,\,
        x_{t,j}=x,\,
        y_{t,j,<h}=u
    \right]
    =
    p_i(a|x,u).
    \label{eq:proof-feedback-conditional-mean}
\end{align}
Indeed, without logits, $q_{t,j,h}(a)=\ind(a_{t,j,h}=a)$, whose
conditional mean is $p_i(a|x,u)$. With logits,
$q_{t,j,h}(a)=p_i(a|x,u)$ is observed directly.

Conditioned on $i_t=i$, using the independence of samples between each round, we have
\begin{align}
\label{eq:condition-prob}
    \PP\big(
        x_{t,j}=x,\,
        y_{t,j,<h}=u |
        i_t=i,\cF_t
    \big)=\rho_i(x)p_i(u|x).
\end{align}
Combining
\eqref{eq:proof-feedback-conditional-mean} and \eqref{eq:condition-prob}, and using that $i_t$ is uniform
over $\cI$, we have
\begin{align}
&\EE\big[
    \ind(x_{t,j}=x,\,y_{t,j,<h}=u)
    q_{t,j,h}(a)
    \,\big|\,
    \cF_t
\big]
\notag\\
&\quad =
\sum_{i\in\cI}
\PP(i_t=i\mid\cF_t)\,
\EE\left[
    \ind(x_{t,j}=x,\,y_{t,j,<h}=u)
    q_{t,j,h}(a)
    \,\middle|\,
    \cF_t,i_t=i
\right]
\notag\\
&\quad =
\frac{1}{I}
\sum_{i\in\cI}
\PP\big(
    x_{t,j}=x,\,
    y_{t,j,<h}=u
   \big|
    \cF_t,i_t=i
\big)
\notag\\
&\qquad\cdot
\EE\big[
    q_{t,j,h}(a)
    \big|
    \cF_t,i_t=i,\,
    x_{t,j}=x,\,
    y_{t,j,<h}=u
\big]
\notag\\
&\quad =
\frac{1}{I}
\sum_{i\in\cI}
\rho_i(x)p_i(u|x)p_i(a|x,u).
\label{eq:proof-one-rollout-mean}
\end{align}
The right-hand side is precisely the joint prefix--token probability under
the forward target \eqref{eq:ar-fkl}. More explicitly,
\begin{align}
    \frac{1}{I}
    \sum_{i\in\cI}
    \rho_i(x)p_i(u|x)p_i(a|x,u)
    =
    \bar\rho(x)
    \pif^*(u|x)
    \pif^*(a|x,u).
    \label{eq:proof-forward-joint-law}
\end{align}
To see this identity, first note that
\begin{align}
    \bar\rho(x)\pif^*(u|x)
    &=
    \bigg(
        \frac{1}{I}\sum_{k\in\cI}\rho_k(x)
    \bigg)
    \left(
        \sum_{i\in\cI}
        \frac{\rho_i(x)}
             {\sum_{k\in\cI}\rho_k(x)}
        p_i(u|x)
    \right)
    \notag\\
    &=
    \frac{1}{I}
    \sum_{i\in\cI}
    \rho_i(x)p_i(u|x).
    \label{eq:proof-forward-prefix-law}
\end{align}
Moreover, conditioned on the prefix $(x,u)$, the posterior weight of
teacher $i$ is
\begin{align*}
    w_i(x,u)
    =
    \frac{\rho_i(x)p_i(u|x)}
         {\sum_{k\in\cI}\rho_k(x)p_k(u|x)}.
\end{align*}
Therefore,
\begin{align}
    \pif^*(a|x,u)
    &=
    \sum_{i\in\cI}
    w_i(x,u)p_i(a|x,u)
    \notag\\
    &=
    \frac{
        \sum_{i\in\cI}
        \rho_i(x)p_i(u|x)p_i(a|x,u)
    }{
        \sum_{k\in\cI}
        \rho_k(x)p_k(u|x)
    }.
    \label{eq:proof-forward-token-law}
\end{align}
Multiplying \eqref{eq:proof-forward-prefix-law} and
\eqref{eq:proof-forward-token-law}, the prefix normalizing factor cancels,
giving
\begin{align}
\notag
    \bar\rho(x)\pif^*(u|x)\pif^*(a|x,u)
    =
    \frac{1}{I}
    \sum_{i\in\cI}
    \rho_i(x)p_i(u|x)p_i(a|x,u).
\end{align}
Recall that
\begin{align*}
    c_t(x,u,a)
    :=
    \frac{1}{m}\sum_{j=1}^m
    \ind(x_{t,j}=x,\,y_{t,j,<h}=u)\,
    q_{t,j,h}(a).
\end{align*}
Averaging \eqref{eq:proof-one-rollout-mean} over the $m$ rollouts therefore
yields
\begin{align}
    \EE\left[
        c_t(x,u,a)
        \,\middle|\,
        \cF_t
    \right]
    =
    \bar\rho(x)
    \pif^*(u|x)
    \pif^*(a|x,u).
    \label{eq:proof-count-mean}
\end{align}
Only linearity of expectation is used here. In particular, the $m$ rollouts
need not be independent after marginalizing over their shared teacher index
$i_t$.

For any autoregressive policy $\pi$, define its round-$t$ empirical loss by
\begin{align}
    \ell_t(\pi)
    :=
    -\sum_{h=1}^H
    \sum_{x\in\cX}
    \sum_{u\in\cA^{h-1}}
    \sum_{a\in\cA}
    c_t(x,u,a)\log\pi(a|x,u).
    \label{eq:proof-empirical-loss}
\end{align}
Define $L(\pi):=\EE_{x\sim\bar\rho,y\sim\pif^*(\cdot|x)}\big[-\log\pi(y|x)\big]$. Using the autoregressive representation of $\pi$, we have
\begin{align*}
    -\log\pi(y|x)
    =
    -\sum_{h=1}^H
    \log\pi(a_h|x,y_{<h}).
\end{align*}
Consequently, using the linearity of expectation, $L(\pi)$ can be expressed as
\begin{align}
\notag
    L(\pi)
    =
    -\sum_{h=1}^H
    \sum_{x,u,a}
    \bar\rho(x)
    \pif^*(u|x)
    \pif^*(a|x,u)
    \log\pi(a|x,u).
\end{align}
Substituting \eqref{eq:proof-count-mean} into
\eqref{eq:proof-empirical-loss} now gives, for every
$\cF_t$-measurable policy $\pi$,
\begin{align}
    \EE\left[
        \ell_t(\pi)
        \,\middle|\,
        \cF_t
    \right]
    =
    L(\pi).
    \label{eq:proof-unbiased-loss}
\end{align}
Moreover, we have
\begin{align}
    L(\pi)-L(\pif^*)
    &=
    \EE_{x\sim\bar\rho,
                    y\sim\pif^*(\cdot|x)}
    \bigg[
        \log
        \frac{\pif^*(y|x)}
             {\pi(y|x)}
    \bigg]
    \notag\\\label{eq:L-diff}
    &=
    \EE_{x\sim\bar\rho}
    \KL\big[
        \pif^*(\cdot|x)
        \|
        \pi(\cdot|x)
    \big].
\end{align}
Applying \eqref{eq:proof-unbiased-loss} and \eqref{eq:L-diff} to $\hat\pi_t$ and $\pif^*$ and the tower property, we have
\begin{align}
\notag
    \EE\big[\operatorname{Regret}(T)\big]
    &= \EE \bigg[\sum_{t=1}^T \big[L(\hat \pi_t) - L(\pif^*)\big]\bigg]\\
    &=\EE\bigg[
        \sum_{t=1}^T
        \big[
            \ell_t(\hat\pi_t)
            -
            \ell_t(\pif^*)
        \big]
    \bigg].
    \label{eq:proof-regret-empirical}
\end{align}
Fix $h \in [H]$, define 
\begin{align*}
\ell_{t,h}(\pi) := -\sum_{x\in\cX}\sum_{u\in\cA^{h-1}}\sum_{a\in\cA}c_t(x,u,a)\log\pi(a|x,u).
\end{align*}
Then, $\ell_{t}(\pi) = \sum_{h} \ell_{t,h}(\pi)$. We consider $\sum_{t=1}^T
        \big[
            \ell_{t,h}(\hat\pi_t)
            -
            \ell_{t,h}(\pif^*)
        \big]$.
Recall that $W_t(x,u,a) = \sum_{i=1}^t c_i(x,u,a)$, $W_t(x,u) = \sum_a W_t(x,u,a)$. For a fixed state $(x,u)$, we omit $(x,u)$ and use the shorthand notation $W_t(a)$ and $W_t$ when it will not cause any confusion. Similarly, we write $c_t(a) = c_t(x,u,a)$ and define $c_t := \sum_a c_t(a)$. Thus, $W_t=\sum_{i=1}^t c_i$.

Using the definition of $c_t(x,u,a)$ in \eqref{eq:fkl-count}, we have
\begin{align}
\notag
    c_t=
    \frac{1}{m}\sum_{j=1}^m
    \ind(x_{t,j}=x,\,y_{t,j,<h}=u)\,
    \sum_a q_{t,j,h}(a).
\end{align}
Note that in both cases, $\sum_a q_{t,j,h}(a)=1$. Thus, we have
\begin{align}
\notag
    c_t
    &=
    \frac{1}{m}
    \sum_{j=1}^m
    \ind(x_{t,j}=x,\,y_{t,j,<h}=u)
    \leq 1.
\end{align}
In Algorithm~\ref{algo:ar-fkl}, the policy at context $(x,u)$ is defined as
\begin{align}
\notag
    \hat\pi_t(a|x,u)
    =
    \frac{W_{t-1}(a)+1/2}
         {W_{t-1}+A/2}.
\end{align}
For simplicity, we write the action set as $\cA = \{1,2,\ldots, A\}$. Let $\mathbf Q_t$ be a random $A$-dim
probability vector with Dirichlet distribution
\begin{align*}
    \mathbf Q_t
    \sim
    \operatorname{Dir}\bigg(
        W_{t-1}(1)+\frac12,\ldots,
        W_{t-1}(A)+\frac12
    \bigg).
\end{align*}
The goal of this construction is the following fact: the mean of $\Qb_t$ is exactly the output policy:
\begin{align}
\notag
    \EE_{q\sim \Qb_t}[q(a)]
    =
    \frac{W_{t-1}(a)+1/2}
         {W_{t-1}+A/2}
    =
    \hat\pi_t(a|x,u).
\end{align}
For a positive vector $v=(v_1,\ldots,v_A)$, define the multivariate beta
function by
\begin{align}
\label{eq:beta-function}
    \mathrm B(v)
    :=
    \frac{\prod_{a=1}^A\Gamma(v_a)}
         {\Gamma(\sum_{a=1}^A v_a)},
\end{align}
where $\Gamma$ is the gamma function. The beta function is the normalizing constant of the Dirichlet distribution. Indeed, for every positive vector $v$, the density of the Dirichlet distribution $\Qb_v$ on the simplex $\Delta_A := \{q \in \RR_+^A: \sum_a q(a) = 1\}$ is
\begin{align*}
    f_v(q)
    =
    \frac{1}{\mathrm B(v)}
    \prod_{a=1}^A q(a)^{v_a-1},
    \qquad q\in\Delta_A.
\end{align*}
Equivalently,
\begin{align*}
    \int_{\Delta_A}
    \prod_{a=1}^A q(a)^{v_a-1}\,\mathrm dq
    =
    \mathrm B(v).
\end{align*}
More generally, for every nonnegative vector $c\in\RR_+^A$,
\begin{align}
\notag
    \EE_{q \sim \Qb_v}\bigg[
        \prod_{a=1}^A q(a)^{c_a}
    \bigg]
    & = \frac{1}{\mathrm B(v)}\int_{\Delta_A} \prod_{a=1}^A q(a)^{c_a + v_a-1} \mathrm dq\\
    &= \frac{\mathrm B(v+c)}{\mathrm B(v)}.
\notag
\end{align}
Therefore, we have
\begin{align}
\notag
    \EE_{q\sim \Qb_t}\bigg[
        \prod_{a\in\cA}
         q(a)^{c_t(a)}
    \bigg]
    &=
    \frac{
        \mathrm B\big(
            W_{t-1}(1) + c_t(1)+\frac12,\ldots,
            W_{t-1}(A) +c_t(A)+\frac12
        \big)
    }{
        \mathrm B\big(
            W_{t-1}(1)+\frac12,\ldots,
            W_{t-1}(A)+\frac12
        \big)
    }\\\label{eq:proof-dirichlet-moment}
    &= \frac{
        \mathrm B\big(
            W_{t}(1)+\frac12,\ldots,
            W_{t}(A)+\frac12
        \big)
    }{
        \mathrm B\big(
            W_{t-1}(1)+\frac12,\ldots,
            W_{t-1}(A)+\frac12
        \big)
    },
\end{align}
where we use $W_{t}(a)=W_{t-1}(a)+c_{t}(a)$. Let
\begin{align*}
    \alpha_t:=\sum_{a\in\cA}c_t(a)\leq1,
    \qquad
    \cA_t^+:=\{a\in\cA:c_t(a)>0\}.
\end{align*}
For every $a\in\cA_t^+$, define the H\"older exponent $r_a:={1}/{c_t(a)}$.
If $\alpha_t<1$, introduce one additional exponent $r_0 := 1/(1-\alpha_t)$. These
exponents satisfy
\begin{align*}
    \frac{1}{r_0}
    +
    \sum_{a\in\cA_t^+}\frac{1}{r_a}
    =
    (1-\alpha_t)+\sum_{a\in\cA_t^+}c_t(a)
    =
    1.
\end{align*}
Using the generalized H\"older's inequality, we have
\begin{align}
    \EE_{q\sim \Qb_t}\bigg[
        \prod_{a\in\cA}
         q(a)^{c_t(a)}
    \bigg]
    &\leq
    \EE[1^{r_0}]^{\frac{1}{r_0}}
    \prod_{a\in\cA_t^+}
    \EE_{q\sim \Qb_t}\bigg[
        \Big[ q(a)^{c_t(a)}\Big]^{r_a}
    \bigg]^{\frac{1}{r_a}}
    \notag\\
    &=
    \prod_{a\in\cA_t^+}
    \EE_{q\sim \Qb_t}\big[q(a)\big]^{c_t(a)}
    =
    \prod_{a\in\cA}
    \hat\pi_t(a|x,u)^{c_t(a)},
    \label{eq:proof-holder}
\end{align}
where we used 
$\EE_{q\sim \Qb_t}[q(a)]=\hat\pi_t(a|x,u)$. If $\alpha_t=1$, we can skip $r_0$ and use the same generalized H\"older's inequality.

Combining \eqref{eq:proof-dirichlet-moment} and
\eqref{eq:proof-holder}, multiplying over $t$, and using the telescoping argument, we have
\begin{align}
    \prod_{t=1}^T
    \prod_{a\in\cA}
    \hat\pi_t(a|x,u)^{c_{t}(a)}
    &\geq
    \prod_{t=1}^T
    \frac{
        \mathrm B\big(W_{t}(\cdot)+\frac12\mathbf 1\big)
    }{
        \mathrm B\big(W_{t-1}(\cdot)+\frac12\mathbf 1\big)
    }
    \notag\\
    &=
    \frac{
        \mathrm B\big(W_{T}(\cdot)+\frac12\mathbf 1\big)
    }{
        \mathrm B\big(\frac12\mathbf 1\big)
    }.
\notag
\end{align}
Taking the negative logarithm, we have
\begin{align}
    -\sum_{t=1}^T\sum_{a\in\cA}
    c_{t}(a)\log\hat\pi_t(a|x,u)
    \leq
    \log
    \frac{
        \mathrm B(\frac12\mathbf 1)
    }{
        \mathrm B\big(W_{T}(\cdot)+\frac12\mathbf 1\big)
    }.\label{eq:telescope}
\end{align}
Next, we consider an arbitrary
policy $\pi(\cdot|x,u)$. First suppose $W_T>0$. By the
nonnegativity of KL divergence,
\begin{align}
\notag
    \sum_{a\in\cA}W_T(a)\log\pi(a|x,u)
    &=
    \sum_{a\in\cA}
    W_T(a)\log\frac{W_T(a)}{W_T}
    -
    W_T
    \sum_{a\in \cA} \frac{W_T(a)}{W_T} \log \frac{W_T(a) }{W_T \cdot \pi(a|x,u)}\\
    &=\sum_{a\in\cA}
    W_T(a)\log\frac{W_T(a)}{W_T}
    -
    W_T
    \KL\bigg(
        \frac{W_T(\cdot)}{W_T}
        \bigg\|
        \pi(\cdot|x,u)
    \bigg)
    \notag\\
    &\leq
    \sum_{a\in\cA}
    W_T(a)\log\frac{W_T(a)}{W_T},
    \label{eq:proof-comparator}
\end{align}
where we set $0\log0:=0$. Summing \eqref{eq:telescope} and \eqref{eq:proof-comparator}, we have
\begin{align}
    \sum_{t=1}^T\sum_{a\in\cA}
    c_t(a)
    \log\frac{\pi(a|z)}{\hat\pi_t(a|z)}
    &\leq
    \log
    \frac{
        \mathrm B(\frac12\mathbf1)
    }{
        \mathrm B(W_T(\cdot)+\frac12\mathbf1)
    }
    +
    \sum_{a\in\cA}
    W_T(a)\log\frac{W_T(a)}{W_T}=: \cR\big(W_T(\cdot)\big).\notag
\end{align}
When $W_T=0$, let $\cR\big(W_T(\cdot)\big) = 0$. The following lemma bounds $\cR\big(W_T(\cdot)\big)$ using Stirling's formula.
\begin{lemma}
\label{lem:statewise-kt}
Let $W_T(a)\geq0$ for every $a\in\cA$, and let
\begin{align*}
    W_T:=\sum_{a\in\cA}W_T(a).
\end{align*}
Then, $\cR\big(W_T(\cdot)\big)$ satisfies
\begin{align}
\notag
    \cR\bigl(W_T(\cdot)\bigr)
    \leq
    A\log(W_T+1)+3A.
\end{align}
\end{lemma}
We defer the proof of Lemma \ref{lem:statewise-kt} to Appendix \ref{sec:proof-fkl-lemma}. Choose $\pi= \pif^*$. For any
$(x,u)\in\cX\times\cA^{h-1}$, we have
\begin{align}
\sum_{t=1}^T \sum_{a\in\cA}
    c_t(x,u,a)
    \log
    \frac{\pif^*(a|x,u)}
         {\hat\pi_t(a|x,u)} \le A\log\bigl(W_T(x,u)+1\bigr)+3A
    \notag.
\end{align}
Therefore, we have
\begin{align*}
    \sum_{t=1}^T
        \big[
            \ell_{t,h}(\hat\pi_t)
            -
            \ell_{t,h}(\pif^*)
        \big] &= \sum_{x\in\cX}\sum_{u\in\cA^{h-1}}\sum_{t=1}^T\sum_{a\in\cA}c_t(x,u,a)\log\frac{\pif^*(a|x,u)}{\hat\pi_t(a|x,u)}\\
        &\leq
    \sum_{(x,u)\in\cX\times\cA^{h-1}}
    \Big[
        A\log\bigl(W_T(x,u)+1\bigr)+3A
    \Big].
\end{align*}
For every fixed level $h$ and round $t$, it is easy to check that
\begin{align*}
    \sum_{(x,u)\in\cX \times \cA^{h-1}} \sum_{a\in \cA}
    c_t(x,u,a)
    =1.
\end{align*}
Consequently,
\begin{align*}
    \sum_{(x,u)\in\cX\times\cA^{h-1}}W_T(x,u)
    =T,
\end{align*}
and hence $W_T(x,u)\leq T$ for any $(x,u)$. This leads to
\begin{align}
\label{eq:fkl-final}
    \sum_{t=1}^T
        \big[
            \ell_{t,h}(\hat\pi_t)
            -
            \ell_{t,h}(\pif^*)
        \big] 
        &\leq
    2SA^{h} \log T.
\end{align}
Finally, summing \eqref{eq:fkl-final} over
$h\in[H]$, we have
\begin{align}
\notag
    \sum_{t=1}^T
        \big[
            \ell_{t}(\hat\pi_t)
            -
            \ell_{t}(\pif^*)
        \big] 
        &\leq
    2S \log T \sum_{h=1}^H A^h \\\notag
    &= 2S \cdot \log T \frac{A(A^{H}-1)}{A-1}\\\label{eq:fkl-final2}
    &\le 4S A^H\log T.
\end{align}
Using \eqref{eq:proof-regret-empirical}, we obtain 
\begin{align*}
    \EE \big[\operatorname{Regret}(T)\big]
    &\le 4S A^H\log T.
\end{align*}
We next consider the high-probability regret bound. Define
\begin{align*}
    X_t
    &:=
    \ell_t(\pif^*)-\ell_t(\hat\pi_t),\\
    r_t
    &:=
    -\EE[X_t\mid\cF_t]
    =
    \EE_{x\sim\bar\rho}
    \KL\big(
        \pif^*(\cdot|x)
        \|
        \hat\pi_t(\cdot|x)
    \big).
\end{align*}
Thus, $\operatorname{Regret}(T)=\sum_{t=1}^T r_t$.
Define $M_0 = 0$, and
\begin{align}
    M_t
    :=
    \sum_{s=1}^t(X_s+r_s),
    \qquad t\in[T].
\notag
\end{align}
Since $r_t=-\EE[X_t\mid\cF_t]$, we have $\EE[M_t-M_{t-1} | \cF_t]=0$. Thus, $\{M_t\}_{t=0}^T$ is a martingale with respect to $\{\cF_{t+1}\}_{t=0}^T$.

Our goal is to control $M_T$. Using the
definition of $\ell_t$,
\begin{align}
\notag
    X_t
    &=
    \sum_{h=1}^H
    \sum_{(x,u,a)}
    c_t(x,u,a)
    \log
    \frac{\hat\pi_t(a|x,u)}
         {\pif^*(a|x,u)}.
\end{align}
First note that in Algorithm \ref{algo:ar-fkl},
\begin{align*}
    \hat\pi_t(a|x,u)
    &=
    \frac{W_{t-1}(x,u,a)+1/2}
         {W_{t-1}(x,u)+A/2}\\
    &\geq
    \frac{1}{2(t-1)+A}
    \geq
    \frac{1}{2T+A}.
\end{align*}
Since $\pif^*(a|x,u)\leq1$, it follows that
\begin{align}
    \frac{X_t}{H}
    \geq
    -L_T,
    \qquad
    L_T:=\log(2T+A).
\label{eq:martingale-lower-bound}
\end{align}
Here we use the following equation 
\begin{align*}
    \sum_{(x,u,a)}
    c_t(x,u,a)
    =1.
\end{align*}
Moreover, we can see
\begin{align}
    \sum_{h=1}^H
    \sum_{(x,u,a)}
    \frac{c_t(x,u,a)}{H}
    =1.
\notag
\end{align}
Using the AM-GM inequality, we have
\begin{align}
    \exp\bigg(\frac{X_t}{H}\bigg)
    &=
    \prod_{h=1}^H
    \prod_{(x,u,a)}
    \left(
        \frac{\hat\pi_t(a|x,u)}
             {\pif^*(a|x,u)}
    \right)^{c_t(x,u,a)/H}
    \notag\\
    &\leq
    \frac1H
    \sum_{h=1}^H
    \sum_{(x,u,a)}
    c_t(x,u,a)
    \frac{\hat\pi_t(a|x,u)}
         {\pif^*(a|x,u)}.
\notag
\end{align}
Taking conditional expectation over $\cF_t$ and using
\begin{align*}
    \EE[c_t(x,u,a)\mid\cF_t]
    =
    \bar\rho(x)
    \pif^*(u|x)
    \pif^*(a|x,u),
\end{align*}
we obtain
\begin{align}
    \EE\bigg[
        \exp\bigg(\frac{X_t}{H}\bigg)
        \bigg|
        \cF_t
    \bigg]
    &\leq
    \frac1H
    \sum_{h=1}^H
    \sum_{(x,u,a)}
    \bar\rho(x)\pif^*(u|x)
    \pif^*(a|x,u)
    \frac{\hat\pi_t(a|x,u)}
         {\pif^*(a|x,u)}
    \notag\\
    &=
    \frac1H
    \sum_{h=1}^H
    \sum_{(x,u)}
    \bar\rho(x)\pif^*(u|x)
    \sum_{a}
    \hat\pi_t(a|x,u)
    \notag\\
    &\leq
    \frac1H
    \sum_{h=1}^H
    \sum_{x\in\cX}
    \bar\rho(x)
    \sum_{u\in\cA^{h-1}}
    \pif^*(u|x)
    \notag\\
    &=1.
    \label{eq:proof-Xt-central-condition}
\end{align}
The following lemma shows that a one-sided lower bound and an exponential central condition provide exponential control of deviations from the conditional mean.
\begin{lemma}
\label{lem:central-to-mgf}
Let $\cF$ be a $\sigma$-algebra, and let $Y$ be a random variable satisfying,
almost surely,
\begin{align}
\notag
    Y\geq-L,
    \qquad
    \EE[e^Y\mid\cF]\leq1,
\end{align}
for some $L\geq0$. Define $ \mu:=-\EE[Y\mid\cF]$. Then $\mu\in[0,L]$, and, for every $\theta\in[0,1]$,
\begin{align}
    \log
    \EE\Big[
        e^{\theta(Y+\mu)}
        \Big|
        \cF
    \Big]
    \leq
    (1+L)\theta^2\mu.
\notag
\end{align}
\end{lemma}
Applying Lemma \ref{lem:central-to-mgf} to $Y = X_t/H$, $L=L_T$, we have for every $\lambda\in[0,1/H]$
\begin{align}
    \log
    \EE\left[
        e^{\lambda(X_t+r_t)}
        \,\middle|\,
        \cF_t
    \right]
    \leq
    H(1+L_T)\lambda^2r_t.
    \label{eq:fkl-centered-mgf}
\end{align}
Define
\begin{align*}
    \mathcal E_t(\lambda)
    :=
    \exp\bigg(
        \lambda\sum_{s=1}^t(X_s+r_s)
        -
        H(1+L_T)\lambda^2\sum_{s=1}^t r_s
    \bigg).
\end{align*}
Then, \eqref{eq:fkl-centered-mgf} implies that $\{\mathcal E_s(\lambda)\}_{s=0}^T$ is a nonnegative supermartingale with
$\mathcal E_0(\lambda)=1$. As a result, $\EE[\cE_T(\lambda)] \le \EE[\cE_0(\lambda)]=1$. Moreover, we have
\begin{align*}
    \cE_T(\lambda)= \exp\Big(\lambda M_T - H(1+L_T)\lambda^2 \operatorname{Regret}(T)\Big).
\end{align*}
Therefore, for any $C > 0$, by Markov's inequality, we have
\begin{align*}
    \PP \Big[\lambda M_T - H(1+L_T)\lambda^2 \operatorname{Regret}(T) \ge C\Big] \le \frac{\EE[\cE_T(\lambda)]}{e^C} \le \frac{1}{e^C}.
\end{align*}
Set $C = \log(1/\delta)$. Now we obtain with probability at least $1-\delta$,
\begin{align}
    M_T = \sum_{t=1}^T(X_t+r_t)
    \leq
    H(1+L_T)\lambda\operatorname{Regret}(T)
    +
    \frac{\log(1/\delta)}{\lambda}.
    \label{eq:fkl-martingale-bound}
\end{align}

On the other hand, \eqref{eq:fkl-final2} gives
\begin{align*}
    -\sum_{t=1}^T X_t
    =
    \sum_{t=1}^T
    \left[
        \ell_t(\hat\pi_t)-\ell_t(\pif^*)
    \right]
    \leq 4S A^H\log T.
\end{align*}
Moreover, we have
\begin{align*}
    \operatorname{Regret}(T)
    =
    -\sum_{t=1}^T X_t
    +
    \sum_{t=1}^T(X_t+r_t).
\end{align*}
Therefore, with probability at least $1-\delta$, the following inequality holds:
\begin{align}
    \operatorname{Regret}(T)
    \leq
    4S A^H\log T
    +
    H(1+L_T)\lambda\operatorname{Regret}(T)
    +
    \frac{\log(1/\delta)}{\lambda}.
\label{eq:fkl-self-bounding-regret}
\end{align}
Choose
\begin{align*}
    \lambda
    =
    \frac{1}{2H(1+\log(2T+A))} \le \frac{1}{H}.
\end{align*}
We have $H(1+L_T)\lambda = 1/2$. Rearranging
\eqref{eq:fkl-self-bounding-regret} yields
\begin{align}
    \operatorname{Regret}(T)
    \leq
    8S A^H\log T
    +
    4H\bigl(1+\log(2T+A)\bigr)
    \log\frac1\delta.
\notag
\end{align}
\end{proof}
\section{Proof of Theorem \ref{thm:ar-rkl-regret}}
\label{sec:proof-rkl-thm}
\begin{proof}[Proof of Theorem \ref{thm:ar-rkl-regret}]
We first control the error of the optimistic token-score estimates. Let $\cF_t$ denote the history before round $t$. In this way, both $\hat l_{t-1}$ and $\hat\pi_t$ are $\cF_t$-measurable.

Let $\cE$ be the event in Lemma~\ref{lem:ar-token-confidence}, which
satisfies $\PP(\cE)\geq1-2\delta$. On this event, simultaneously for
every round $s$ and every visited triple $(x,u,a)$,
\begin{align}
    \left|
        \bar l_s(x,u,a)-\log g_h(a|x,u)
    \right|
    \leq \beta_s(x,u,a).
    \label{eq:rkl-proof-confidence}
\end{align}
Moreover, Assumption \ref{assump:token-bound} and the definition of $g_h$ imply that for every
teacher $i$,
\begin{align*}
    \big|
        \log p_i(a|x,u)-\log g_h(a|x,u)
    \big|
    &=
    \bigg|
        \log\frac{p_i(a|x,u)}{\piref(a|x,u)}
        -
        \sum_{k\in\cI}w_k(x)
        \log\frac{p_k(a|x,u)}{\piref(a|x,u)}
    \bigg|
    \leq2B.
\end{align*}
Since $\bar l_s(x,u,a)$ is an average of the observed teacher
log-probabilities at this triple, it follows that
\begin{align}
\label{eq:rkl-beta-trivial}
    \big|
        \bar l_s(x,u,a)-\log g_h(a|x,u)
    \big|
    \leq2B.
\end{align}
Combining 
\eqref{eq:rkl-proof-confidence} and \eqref{eq:rkl-beta-trivial}, we obtain
\begin{align*}
    \big|
        \bar l_s(x,u,a)-\log g_h(a|x,u)
    \big|
    \leq\min\big\{\beta_s(x,u,a),2B\big\}.
\end{align*}
Consequently, on the event $\cE$, the optimistic estimate
$\hat l_s=\bar l_s+\min\{\beta_s,2B\}$ satisfies
\begin{align*}
    0
    \leq
    \hat l_s(x,u,a)-\log g_h(a|x,u)
    \leq
    \min\big\{2\beta_s(x,u,a),4B\big\}.
\end{align*}
For an unvisited triple, the algorithm sets
\begin{align*}
    \hat l_s(x,u,a)
    =
    \log\piref(a|x,u)+B.
\end{align*}
Assumption~\ref{assump:token-bound} implies
\begin{align*}
    -B
    \leq
    \log g_h(a|x,u)-\log\piref(a|x,u)
    \leq B.
\end{align*}
Therefore,
\begin{align*}
    0
    \leq
    \hat l_s(x,u,a)-\log g_h(a|x,u)
    \leq2B.
\end{align*}
We next consider the output policy. Multiplying the token conditionals defined by the backward recursion gives
\begin{align}
    \hat\pi_t(y|x)
    &=
    \prod_{h=1}^H
    \frac{
        \exp\bigl(\hat l_{t-1}(x,y_{<h},a_h)\bigr)
        \hat V_{t-1,h+1}(x,y_{\leq h})
    }{
        \hat V_{t-1,h}(x,y_{<h})
    }\notag\\
    &=
    \frac{
        \exp\left(
            \sum_{h=1}^H
            \hat l_{t-1}(x,y_{<h},a_h)
        \right)
    }{
        \hat V_{t-1,1}(x)
    },\label{eq:rkl-proof-sequence-policy}
\end{align}
where the second equation holds by canceling out the telescoping terms and
$\hat V_{t-1,H+1}=1$. Comparing this expression with the reverse target
in \eqref{eq:seq-rkl}, we obtain
\begin{align}
    \log\frac{\hat\pi_t(y|x)}{\pir^*(y|x)}
    &=
    \sum_{h=1}^H
    \Big[
        \hat l_{t-1}(x,y_{<h},a_h)
        -
        \log g_h(a_h|x,y_{<h})
    \Big]+
    \log\frac{V_1(x)}{\hat V_{t-1,1}(x)}.
    \label{eq:rkl-proof-sequence-ratio}
\end{align}
Moreover, the ratio of the normalizing constants can be expressed as an expectation under $\hat\pi_t$. To see this, first expand the expectation:
\begin{align}
\notag
    &\EE_{y\sim\hat\pi_t(\cdot|x)}
    \exp\bigg(
        -\sum_{h=1}^H
        \Big[
            \hat l_{t-1}(x,y_{<h},a_h)
            -
            \log g_h(a_h|x,y_{<h})
        \Big]
    \bigg)\\\label{eq:rkl-expect}
    &\quad=
    \sum_{y\in\cY}
    \hat\pi_t(y|x)
    \exp\bigg(
        -\sum_{h=1}^H\hat l_{t-1}(x,y_{<h},a_h)
        +
        \sum_{h=1}^H\log g_h(a_h|x,y_{<h})
    \bigg).
\end{align}
Substituting
\eqref{eq:rkl-proof-sequence-policy} into \eqref{eq:rkl-expect}, the right-hand side becomes
\begin{align*}
    &\sum_{y\in\cY}
    \frac{
        \exp\big(
            \sum_{h=1}^H\hat l_{t-1}(x,y_{<h},a_h)
        \big)
    }{
        \hat V_{t-1,1}(x)
    }
    \exp\bigg(
        -\sum_{h=1}^H\hat l_{t-1}(x,y_{<h},a_h)
        +
        \sum_{h=1}^H\log g_h(a_h|x,y_{<h})
    \bigg)\\
    &\quad=
    \frac{1}{\hat V_{t-1,1}(x)}
    \sum_{y\in\cY}
    \exp\bigg(
        \sum_{h=1}^H\log g_h(a_h|x,y_{<h})
    \bigg)\\
    &\quad=
    \frac{1}{\hat V_{t-1,1}(x)}
    \sum_{y\in\cY}
    \prod_{h=1}^H g_h(a_h|x,y_{<h}).
\end{align*}
Finally, recall that in Theorem \ref{thm:obj-reverse-kl}, we have seen
\begin{align*}
    V_1(x)
    =
    \sum_{y\in\cY}
    \prod_{h=1}^H g_h(a_h|x,y_{<h}).
\end{align*}
Consequently,
\begin{align}
    \EE_{y\sim\hat\pi_t(\cdot|x)}
    \exp\bigg(
        -\sum_{h=1}^H
        \Big[
            \hat l_{t-1}(x,y_{<h},a_h)
            -
            \log g_h(a_h|x,y_{<h})
        \Big]
    \bigg)
    =
    \frac{V_1(x)}{\hat V_{t-1,1}(x)}.
    \label{eq:rkl-proof-normalizer-ratio}
\end{align}
Taking expectation over $\hat \pi_t(\cdot|x)$ and substituting \eqref{eq:rkl-proof-normalizer-ratio} into
\eqref{eq:rkl-proof-sequence-ratio}  yields
\begin{align}
    &\KL\bigl(
        \hat\pi_t(\cdot|x)
        \,\big\|\,
        \pir^*(\cdot|x)
    \bigr)=
    \EE_{y\sim\hat\pi_t(\cdot|x)}
    \bigg[
        \sum_{h=1}^H
        \Big[
            \hat l_{t-1}(x,y_{<h},a_h)
            -
            \log g_h(a_h|x,y_{<h})
        \Big]
    \bigg]\notag\\
    &\quad+
    \log\EE_{y\sim\hat\pi_t(\cdot|x)}
    \exp\bigg(
        -\sum_{h=1}^H
        \Big[
            \hat l_{t-1}(x,y_{<h},a_h)
            -
            \log g_h(a_h|x,y_{<h})
        \Big]
    \bigg).
    \label{eq:rkl-proof-kl-identity}
\end{align}
On the event $\cE$, $\hat l_{t-1}(x,y_{<h},a_h) \ge \log g_h(a_h|x,y_{<h})$ for any $t,h$. We can therefore apply the basic inequality $e^{-v}\leq1-v+v^2/2$ for $v\geq0$. Let $V = \sum_{h=1}^H
        \big[
            \hat l_{t-1}(x,y_{<h},a_h)
            -
            \log g_h(a_h|x,y_{<h})
        \big]$. The right-hand side can be bounded as
        \begin{align*}
            \EE[V] + \log \EE \exp(-V) &\le \EE[V] + \log \big(1-\EE [V]+\EE [V^2]/2\big)\\
            & \le \EE[V] - \EE[V] + \frac{1}{2} \EE [V^2]\\
            &= \frac{1}{2} \EE [V^2],
        \end{align*}
where we use $\log x \le x - 1$ for any $x > 0$.  
As a result, we have
\begin{align}
    \KL\bigl(
        \hat\pi_t(\cdot|x)
        \,\big\|\,
        \pir^*(\cdot|x)
    \bigr)&\leq
    \frac12
    \EE_{y\sim\hat\pi_t(\cdot|x)}
    \bigg[
        \bigg(
            \sum_{h=1}^H
            \Big[
                \hat l_{t-1}(x,y_{<h},a_h)
                -
                \log g_h(a_h|x,y_{<h})
            \Big]
        \bigg)^2
    \bigg]\notag\\
    &\quad\leq
    \frac H2
    \EE_{y\sim\hat\pi_t(\cdot|x)}
    \bigg[
        \sum_{h=1}^H
        \Big[
            \hat l_{t-1}(x,y_{<h},a_h)
            -
            \log g_h(a_h|x,y_{<h})
        \Big]^2
    \bigg],
    \label{eq:rkl-proof-kl-token-error}
\end{align}
where the last inequality holds due to the Cauchy--Schwarz inequality. 

We next relate the expected squared errors in
\eqref{eq:rkl-proof-kl-token-error} to those evaluated on the observed
rollouts. Define
\begin{align*}
    X_t
    :=
    \frac{H}{2m}
    \sum_{j=1}^m\sum_{h=1}^H
    \Big[
        \hat l_{t-1}(x_{t,j},y_{t,j,<h},a_{t,j,h})
        -
        \log g_h(a_{t,j,h}|x_{t,j},y_{t,j,<h})
    \Big]^2.
\end{align*}
Conditionally on $\cF_t$, each rollout has joint distribution
\begin{align*}
    \PP(x_{t,j}=x,y_{t,j}=y\mid\cF_t)
    &=
    \frac1I\sum_{i\in\cI}
    \rho_i(x)\hat\pi_t(y|x)\\
    &=
    \bar\rho(x)\hat\pi_t(y|x).
\end{align*}
Since $\hat l_{t-1}$ and $\hat\pi_t$ are $\cF_t$-measurable,
linearity of conditional expectation therefore gives
\begin{align}
    \EE[X_t\mid\cF_t]
    &=
    \frac{H}{2m}
    \sum_{j=1}^m
    \EE_{x\sim\bar\rho}
    \EE_{y\sim\hat\pi_t(\cdot|x)}
    \bigg[
        \sum_{h=1}^H
        \Big[
            \hat l_{t-1}(x,y_{<h},a_h)
            -
            \log g_h(a_h|x,y_{<h})
        \Big]^2
    \bigg]\notag\\
    &=
    \frac H2
    \EE_{x\sim\bar\rho}
    \EE_{y\sim\hat\pi_t(\cdot|x)}
    \bigg[
        \sum_{h=1}^H
        \Big[
            \hat l_{t-1}(x,y_{<h},a_h)
            -
            \log g_h(a_h|x,y_{<h})
        \Big]^2
    \bigg].
    \label{eq:rkl-proof-empirical-mean}
\end{align}
Note that this does not require independence among the rollouts within a round. On the event $\cE$, combining \eqref{eq:rkl-proof-empirical-mean} with~\eqref{eq:rkl-proof-kl-token-error}, we have
\begin{align}
    \operatorname{Regret}(T)
    &=
    \sum_{t=1}^T
    \EE_{x\sim\bar\rho}
    \KL\bigl(
        \hat\pi_t(\cdot|x)
        \,\big\|\,
        \pir^*(\cdot|x)
    \bigr)
    \leq
    \sum_{t=1}^T\EE[X_t\mid\cF_t].
    \label{eq:rkl-proof-regret-predictable}
\end{align}
To apply Lemma~\ref{lem:foster}, we need a bound on $X_t$. At every previously visited triple,
\eqref{eq:rkl-beta-trivial} and the definition of $\hat l_{t-1}$ give
\begin{align*}
    \big|
        \hat l_{t-1}(x,u,a)-\log g_h(a|x,u)
    \big|
    &\leq
    \big|
        \bar l_{t-1}(x,u,a)-\log g_h(a|x,u)
    \big|
    +
    \min\{\beta_{t-1}(x,u,a),2B\}\\
    &\leq 4B.
\end{align*}
At an unvisited triple, 
$\hat l_{t-1}(x,u,a)=\log\piref(a|x,u)+B$ indicates $\big|
        \hat l_{t-1}(x,u,a)-\log g_h(a|x,u)
    \big| \in [0,2B]$. Consequently, in both cases, 
\begin{align*}
    0\leq X_t
    \leq
    \frac{H}{2m}\cdot mH\cdot(4B)^2
    =
    8H^2B^2.
\end{align*}
Moreover, $X_t$ is $\cF_{t+1}$-measurable. Applying
Lemma~\ref{lem:foster}, with probability at least $1-\delta$, the following inequality holds
\begin{align}
    \sum_{t=1}^T\EE[X_t\mid\cF_t]
    \leq
    2\sum_{t=1}^T X_t
    +
    64H^2B^2\log\frac2\delta.
    \label{eq:rkl-proof-martingale}
\end{align}
Substituting \eqref{eq:rkl-proof-martingale} into \eqref{eq:rkl-proof-regret-predictable} and taking a union bound, we conclude
that, with probability at least $1-3\delta$,
\begin{align}
\notag
    \operatorname{Regret}(T)
    &\leq
    \frac Hm
    \underbrace{\sum_{t=1}^T\sum_{j=1}^m\sum_{h=1}^H
    \Big[
        \hat l_{t-1}(x_{t,j},y_{t,j,<h},a_{t,j,h})
        -
        \log g_h(a_{t,j,h}|x_{t,j},y_{t,j,<h})
    \Big]^2}_{I}\\
    &\qquad +
    64H^2B^2\log\frac2\delta.
    \label{eq:rkl-proof-regret-empirical}
\end{align}
It remains to bound $I$ for every data trajectory on $\cE$.
Regrouping the observations by their corresponding triples gives
\begin{align*}
    I
    =
    \sum_{h=1}^H
    \sum_{(x,u,a)}
    \sum_{t=1}^T
    \big[
        N_t(x,u,a)-N_{t-1}(x,u,a)
    \big]
    \Big[
        \hat l_{t-1}(x,u,a)-\log g_h(a|x,u)
    \Big]^2.
\end{align*}
Here $N_t(x,u,a)-N_{t-1}(x,u,a)$ counts the visits to this triple
in round $t$, each of which contributes the same squared error.
Since there are $m$ samples per round,
\begin{align*}
    0
    \leq N_t(x,u,a)-N_{t-1}(x,u,a)
    \leq m,
    \qquad
    N_T(x,u,a)\leq mT.
\end{align*}
For each fixed triple $(x,u,a)$, we split the inner sum according to whether
$N_{t-1}(x,u,a)=0$ or $N_{t-1}(x,u,a)\geq1$.
In the former case, only the first round that visits this triple can contribute, with at most $m$ visits. Hence,
\begin{align*}
    \sum_{\substack{t\in[T]:\\N_{t-1}(x,u,a)=0}}
    \big[N_t(x,u,a)-N_{t-1}(x,u,a)\big]
    \Big[
        \hat l_{t-1}(x,u,a)-\log g_h(a|x,u)
    \Big]^2
    \leq4mB^2.
\end{align*}
Summing this bound over all $K:=S\sum_{h=1}^H A^h$ triples and applying Lemma \ref{lem:ar-token-confidence} gives
\begin{align}
\notag
    &I
    \leq
    4mB^2K
    \\\label{eq:nonzero-sum}
    & \qquad+
    \sum_{h=1}^H
    \sum_{(x,u,a)}
    \sum_{\substack{t\in[T]:\\N_{t-1}(x,u,a)\geq1}}
    \big[N_t(x,u,a)-N_{t-1}(x,u,a)\big]\cdot
    \min\Big\{
        16B^2,\,
        4\beta_{t-1}(x,u,a)^2
    \Big\}.
\end{align}
Recall that
\begin{align*}
    \beta_t(x,u,a) = \tilde O\bigg(
        B\sqrt{\frac{mH}{N_t(x,u,a)}}
        +
        \frac{BmH}{N_t(x,u,a)}
    \bigg).
\end{align*}
For $N_{t-1}(x,u,a)\geq1$, squaring the confidence radius gives
\begin{align*}
    4\beta_{t-1}(x,u,a)^2
    \leq
    \widetilde O\left(
        \frac{B^2mH}{N_{t-1}(x,u,a)}
        +
        \frac{B^2m^2H^2}{N_{t-1}(x,u,a)^2}
    \right),
\end{align*}
where we used $(v+w)^2\leq2v^2+2w^2$.
Combining this with the constant bound $16B^2$, we obtain
\begin{align*}
    \min\Big\{
        16B^2,\,
        4\beta_{t-1}(x,u,a)^2
    \Big\}&\leq
    \widetilde O\bigg(
        B^2\min\bigg\{
            1,\,
            \frac{mH}{N_{t-1}(x,u,a)}
            +
            \frac{m^2H^2}{N_{t-1}(x,u,a)^2}
        \bigg\}
    \bigg)\\
    &\qquad\leq
    \widetilde O\bigg(
        B^2 \min \bigg\{1, \frac{mH}{N_{t-1}(x,u,a)}\bigg\}
    \bigg),
\end{align*}
where we use $\min \{1,v+v^2\} \le 2 \min\{1,v\}$ since $v^2 \le v$ when $0\le v \le 1$.

Fix a triple $(x,u,a)$ with $N_T(x,u,a)\geq1$.
Order its visits chronologically, breaking ties within each round
arbitrarily. If its $k$-th visit occurs in round $t$, then
\begin{align*}
    k
    \leq N_{t-1}(x,u,a)+m
    \leq N_{t-1}(x,u,a)+mH,
\end{align*}
because at most $m$ visits occur in that round and $H\geq1$.
For each visit included in the remaining sum,
$N_{t-1}(x,u,a)\geq1$, and hence
\begin{align*}
    \min\left\{
        1,\frac{mH}{N_{t-1}(x,u,a)}
    \right\}
    &\leq
    \frac{2mH}{N_{t-1}(x,u,a)+mH}\\
    &\leq
    \frac{2mH}{k}.
\end{align*}
The first inequality uses
$\min\{1,v\}\leq2v/(1+v)$ for $v\geq0$. Since $N_t(x,u,a)-N_{t-1}(x,u,a)$ counts the visits in round $t$, the following summation counts all the visits to $(x,u,a)$ and thus can be indexed by $k$, leading to
\begin{align}
    &\sum_{\substack{t\in[T]:\\
        N_{t-1}(x,u,a)\geq1}}
    \big[N_t(x,u,a)-N_{t-1}(x,u,a)\big]
    \min\bigg\{
        1,\frac{mH}{N_{t-1}(x,u,a)}
    \bigg\}\notag\\
    &\qquad\leq
    2mH\sum_{k=1}^{N_T(x,u,a)}\frac1k\notag\\
    &\qquad\leq
    2mH\bigl[1+\log N_T(x,u,a)\bigr] \notag\\
    & \qquad 
    \leq
    2mH\log(emT),
    \label{eq:rkl-proof-harmonic-sum}
\end{align}
where we use $\sum_{k=1}^T 1/k \le 1 + \log T$ and $N_T(x,u,a)\leq mT$.

Substituting \eqref{eq:rkl-proof-harmonic-sum} into \eqref{eq:nonzero-sum} and summing over all $(x,u,a)$ triples, we have
\begin{align*}
    I
    &\leq
    4mB^2K
    +
    \widetilde O\bigl(B^2mHK\log(emT)\bigr)=
    \widetilde O\bigl(B^2mHK \log T\bigr) = \tilde O\bigl(B^2mSHA^H \log T\bigr).
\end{align*}
In conclusion, \eqref{eq:rkl-proof-regret-empirical} gives that with probability at least $1-3\delta$, the following inequality holds:
\begin{align*}
    \operatorname{Regret}(T) &\le \frac{H}{m} I + \tilde O(B^2H^2)\\
    &\le \tilde O\big(B^2H^2SA^H \log T\big).
\end{align*}
This completes the proof of Theorem \ref{thm:ar-rkl-regret}.
\end{proof}

\section{Proof of Lemmas Used in Appendices \ref{sec:proof-fkl-thm} and \ref{sec:proof-rkl-thm}}
\label{sec:proof-fkl-lemma}
\subsection{Proof of Lemma \ref{lem:statewise-kt}}
\begin{proof}[Proof of Lemma \ref{lem:statewise-kt}]
When $W_T = 0$, the inequality holds naturally. We therefore assume throughout the rest of the proof that $W_T>0$. If $A=1$, then $\mathrm B(v)=\frac{\Gamma(v)}{\Gamma(v)}=1$ for every $v>0$, and
\begin{align*}
    W_T(1)\log\frac{W_T(1)}{W_T}=0.
\end{align*}
Thus, $\cR(W_T(\cdot))=0$. It remains to consider $A\geq2$.

By the definition of the beta function \eqref{eq:beta-function}, we have
\begin{align}
    -\log
    \mathrm B\Big(
        W_T(\cdot)+\frac12\mathbf1
    \Big)
    &=
    \log\Gamma\Big(W_T+\frac A2\Big)
    -
    \sum_{a\in\cA}
    \log\Gamma\Big(W_T(a)+\frac12\Big).
    \label{eq:expand-beta}
\end{align}
Using Lemma \ref{lemma:stirling} with
$s=W_T+A/2$, we have
\begin{align}
    \log\Gamma\Big(W_T+\frac A2\Big)
    &\leq
    \Big(W_T+\frac{A-1}{2}\Big)
    \log\Big(W_T+\frac A2\Big)
    -
    \Big(W_T+\frac A2\Big)
    \notag\\
    &\qquad
    +
    \frac12\log(2\pi)
    +
    \frac{1}{12(W_T+A/2)}.
    \label{eq:stirling-total-count}
\end{align}
For every $a\in\cA$, applying Lemma \ref{lemma:stirling} with
$s=W_T(a)+1/2$, we have
\begin{align}
    \log\Gamma\Big(W_T(a)+\frac12\Big)
    &\geq
    W_T(a)\log\Big(W_T(a)+\frac12\Big)
    -
    \Big(W_T(a)+\frac12\Big)
    +
    \frac12\log(2\pi).
    \label{eq:stirling-action-count}
\end{align}
Substituting \eqref{eq:stirling-total-count} and
\eqref{eq:stirling-action-count} into \eqref{eq:expand-beta}, we obtain
\begin{align}
    &-\log
    \mathrm B\Big(
        W_T(\cdot)+\frac12\mathbf1
    \Big)
    \leq
    -
    \sum_{a\in\cA}
    W_T(a)\log\Big(W_T(a)+\frac12\Big)
    \notag\\
    &\quad+
    \Big(W_T+\frac{A-1}{2}\Big)
    \log\left(W_T+\frac A2\right)
    \notag-
    \frac{A-1}{2}\log(2\pi)
    +
    \frac{1}{12(W_T+A/2)},\notag
\end{align}
where we use $\sum_{a\in\cA}W_T(a)=W_T$. Recall that
\begin{align*}
    \cR\bigl(W_T(\cdot)\bigr)
    &=
    \log\mathrm B\Big(\frac12\mathbf1\Big)
    -
    \log\mathrm B\Big(
        W_T(\cdot)+\frac12\mathbf1
    \Big)
    +
    \sum_{a\in\cA}
    W_T(a)\log\frac{W_T(a)}{W_T}.
\end{align*}
Therefore, we have
\begin{align}
    &\cR\bigl(W_T(\cdot)\bigr)
    \leq
    \log\mathrm B\Big(\frac12\mathbf1\Big)
    +
    \underbrace{\sum_{a\in\cA}
    W_T(a)
    \log
    \frac{W_T(a)}{W_T(a)+1/2}}_{I_1}
    \notag\\
    &\quad+
    \underbrace{W_T
    \log\frac{W_T+A/2}{W_T}}_{I_2}
    +
    \underbrace{\frac{A-1}{2}
    \log\Big(W_T+\frac A2\Big)}_{I_3}
    -
    \frac{A-1}{2}\log(2\pi)
    +
    \frac{1}{12(W_T+A/2)}.
    \label{eq:redundancy-expanded}
\end{align}
Since $0\leq
    {W_T(a)}/{[W_T(a)+1/2]}
    \leq1$, $I_1 \le 0$.
For $I_2$, using $\log(1+s)\leq s$ for every $s\geq0$, we have
\begin{align}
    I_2
    &=
    W_T\log\left(1+\frac{A}{2W_T}\right)
    \leq
    \frac A2.
\notag
\end{align}
For $I_3$, we have $ W_T+\frac A2
    \leq
    (W_T+1)\big(1+ A/2\big)$. Thus, 
\begin{align*}
    W_T+\frac A2
    \leq
    (W_T+1)\Big(1+\frac A2\Big),
\end{align*}
and hence $I_3 \le (A-1) \log(W_T+1)/2 + (A-1) \log(1+A/2)/2$. As a result, 
\eqref{eq:redundancy-expanded} indicates
\begin{align}
    \cR\bigl(W_T(\cdot)\bigr)
    &\leq
    \frac{A-1}{2}\log(W_T+1)
    +
    \log\mathrm B\Big(\frac12\mathbf1\Big)
    +
    \frac A2
    +
    \frac{A-1}{2}
    \log\Big(1+\frac A2\Big)
    +
    \frac1{12}.
    \label{eq:redundancy-before-constants}
\end{align}
Finally, by definition,
\begin{align}
\notag
    \log\mathrm B\Big(\frac12\mathbf1\Big)
    =
    \frac A2\log\pi
    -
    \log\Gamma\Big(\frac A2\Big).
\end{align}
Applying Lemma \ref{lemma:stirling} with $s = A/2$, we have
\begin{align*}
    \log\Gamma\Big(\frac A2\Big)
    &\geq
    \frac{A-1}{2}\log\frac A2
    -
    \frac A2
    +
    \frac12\log(2\pi).
\end{align*}
Since $A\geq2$, we have
\begin{align}
    \log\mathrm B\Big(\frac12\mathbf1\Big)
    \leq
    \frac A2(\log\pi+1) - \frac{A-1}{2} \log \frac{A}{2}.
    \label{eq:initial-beta-bound}
\end{align}
Substituting \eqref{eq:initial-beta-bound} into \eqref{eq:redundancy-before-constants}, we have
\begin{align*}
    \cR\bigl(W_T(\cdot)\bigr)
    &\leq
    \frac{A-1}{2}\log(W_T+1)
    +
    \frac A2
    +
    \frac{A-1}{2}
    \log\Big(1+\frac A2\Big)
    +
    \frac1{12}\\
    &\qquad + \frac A2(\log\pi+1) - \frac{A-1}{2} \log \frac{A}{2}\\
    & \le A\log(W_T+1) + 3A,
\end{align*}
where we use $\log(1+A/2) - \log (A/2) \le 2/A$. This completes the proof of Lemma \ref{lem:statewise-kt}.
\end{proof}
\subsection{Proof of Lemma \ref{lem:central-to-mgf}}
 \begin{proof}[Proof of Lemma \ref{lem:central-to-mgf}]
By conditional Jensen's inequality,
\begin{align*}
    \exp\bigl(\EE[Y\mid\cF]\bigr)
    \leq
    \EE[e^Y\mid\cF]
    \leq1.
\end{align*}
Therefore, $\EE[Y\mid\cF]\leq0$ and hence $\mu\geq0$. Moreover,
$Y\geq-L$ implies $\EE[Y\mid\cF]\geq-L$, so $\mu\leq L$.

We first show that, for every $y\geq-L$ and $\theta\in[0,1]$,
\begin{align}
    e^{\theta y}-1-\theta y
    \leq
    (1+L)\theta^2(e^y-1-y).
    \label{eq:central-pointwise}
\end{align}
First, suppose that $y\geq0$. Expanding the exponential using Taylor's series gives
\begin{align*}
    e^{\theta y}-1-\theta y
    &=
    \sum_{k=2}^{\infty}
    \frac{\theta^ky^k}{k!}\\
    &\leq
    \theta^2
    \sum_{k=2}^{\infty}
    \frac{y^k}{k!}\\
    &=
    \theta^2(e^y-1-y),
\end{align*}
where we used $\theta^k\leq\theta^2$ for every $k\geq2$.

It remains to show \eqref{eq:central-pointwise} when $-L \le y \le 0$. Let $y=-s$ for some $s\in[0,L]$. Then, on the one hand, the left-hand side of \eqref{eq:central-pointwise} satisfies
\begin{align*}
    e^{-\theta s}-1+\theta s
    \leq
    \frac{\theta^2s^2}{2}.
\end{align*}
On the other hand, the right-hand side of \eqref{eq:central-pointwise} satisfies:
\begin{align*}
    (1+L)\theta^2(e^{-s}-1+s)
    &=
    (1+L)\theta^2\int_0^s(1-e^{-v})\,\mathrm dv\\
    &\geq(1+L)\theta^2
    \int_0^s\frac{v}{1+v}\,\mathrm dv\\
    &\geq(1+L)\theta^2
    \int_0^s\frac{v}{1+L}\,\mathrm dv\\
    &=
    \frac{\theta^2s^2}{2},
\end{align*}
where the first inequality holds due to $e^x \ge 1+x$, and thus $e^{-x} \le 1/(1+x)$ whenever $x \ge 0$. The second inequality holds due to $v \le s \le L$.
Therefore, we have proved \eqref{eq:central-pointwise} for $y\in[-L,0]$.

Using \eqref{eq:central-pointwise}, we obtain
\begin{align}
    \EE[e^{\theta Y}\mid\cF]
    &=
    1+\theta\EE[Y\mid\cF]
    +
    \EE[
        e^{\theta Y}-1-\theta Y
        \mid\cF
    ]
    \notag\\
    &\leq
    1-\theta\mu
    +
    (1+L)\theta^2
    \EE[e^Y-1-Y\mid\cF].
    \label{eq:central-uncentered}
\end{align}
Using $\EE[e^Y\mid\cF]\leq1$, we have
\begin{align*}
    \EE[e^Y-1-Y\mid\cF]
    &=
    \EE[e^Y\mid\cF]-1-\EE[Y\mid\cF]\\
    &\leq
    -\EE[Y\mid\cF]\\
    &=
    \mu.
\end{align*}
Therefore, \eqref{eq:central-uncentered} gives
\begin{align*}
    \EE[e^{\theta Y}\mid\cF]
    \leq
    1-\theta\mu+(1+L)\theta^2\mu.
\end{align*}
Consequently,
\begin{align*}
    \log
    \EE\left[
        e^{\theta(Y+\mu)}
        \,\middle|\,
        \cF
    \right]
    &=
    \theta\mu+
    \log\EE[e^{\theta Y}\mid\cF]\\
    &\leq
    \theta\mu+
    \log\left(
        1-\theta\mu+(1+L)\theta^2\mu
    \right)\\
    &\leq
    \theta\mu-\theta\mu+(1+L)\theta^2\mu\\
    &=
    (1+L)\theta^2\mu,
\end{align*}
where the last inequality uses $\log(1+s)\leq s$. This completes the proof of Lemma \ref{lem:central-to-mgf}.
\end{proof}
\subsection{Proof of Lemma \ref{lem:ar-token-confidence}}
 \begin{proof}
Fix a level $h$ and a triple $z=(x,u,a)$.
Let $\cF_t$ denote the history before round $t$. Define
\begin{align*}
    Y_t(z)
    :=
    \frac1m
    \sum_{j=1}^m
    \ind\bigl(
        x_{t,j}=x,\,
        y_{t,j,<h}=u,\,
        a_{t,j,h}=a
    \bigr)
    \bigl[
        l_{t,j,h}-\log g_h(a|x,u)
    \bigr].
\end{align*}
By the definitions of $S_t$ and $N_t$, we have
\begin{align}
    \sum_{s=1}^t Y_s(z)
    =
    \frac{
        S_t(x,u,a)
        -
        N_t(x,u,a)\log g_h(a|x,u)
    }{m}.
    \label{eq:confidence-centered-sum}
\end{align}
Under the on-policy protocol, conditionally on $\cF_t$, the teacher index is uniform over $\cI$, the context follows $\rho_{i_t}$, and the response is generated by $\hat\pi_t$. Therefore,
\begin{align*}
    \EE[Y_t(z)\mid\cF_t]
    &=
    \frac1m\sum_{j=1}^m
    \frac1I\sum_{i\in\cI}
    \rho_i(x)\hat\pi_t(u|x)\hat\pi_t(a|x,u)
    \bigl[
        \log p_i(a|x,u)-\log g_h(a|x,u)
    \bigr]\\
    &=
    \bar\rho(x)\hat\pi_t(u|x)\hat\pi_t(a|x,u)
    \sum_{i\in\cI}w_i(x)
    \bigl[
        \log p_i(a|x,u)-\log g_h(a|x,u)
    \bigr]\\
    &=0,
\end{align*}
where the second equality uses $\rho_i(x)/I=\bar\rho(x)w_i(x)$. The last equality holds due to $\sum_iw_i(x)=1$ and $\log g_h(a|x,u)=\sum_iw_i(x)\log p_i(a|x,u)$. 
Using Assumption~\ref{assump:token-bound}, for every teacher $i$,
\begin{align*}
    \big|
        \log p_i(a|x,u)-\log g_h(a|x,u)
    \big|
    &=
    \bigg|
        \log\frac{p_i(a|x,u)}{\piref(a|x,u)}
        -
        \sum_{k\in\cI}w_k(x)
        \log\frac{p_k(a|x,u)}{\piref(a|x,u)}
    \bigg|\\
    &\leq
    B+\sum_{k\in\cI}w_k(x)B
    =2B.
\end{align*}
When $\ind\bigl(x_{t,j}=x,\,y_{t,j,<h}=u,\,a_{t,j,h}=a\bigr) \neq 0$, $l_{t,j,h}=\log p_i(a|x,u)$ for some $i\in \cI$. Thus, 
\begin{align*}
    \big|l_{t,j,h}-\log g_h(a|x,u)\big| \le 2B.
\end{align*}
In the definition of $Y_t(z)$, at most $m$ summands can be nonzero. Thus, $|Y_t(z)| \le 2Bm/m = 2B$. Moreover, Jensen's inequality gives
\begin{align*}
    Y_t(z)^2
    &\leq
    \frac1m\sum_{j=1}^m
    \ind\bigl(
        x_{t,j}=x,\,
        y_{t,j,<h}=u,\,
        a_{t,j,h}=a
    \bigr)
    \bigl[
        l_{t,j,h}-\log g_h(a|x,u)
    \bigr]^2\\
    &\leq
    \frac{4B^2}{m}
    \sum_{j=1}^m
    \ind\bigl(
        x_{t,j}=x,\,
        y_{t,j,<h}=u,\,
        a_{t,j,h}=a
    \bigr).
\end{align*}
Taking the conditional expectation, we obtain
\begin{align}
    \Var(Y_t(z)\mid\cF_t)
    &=
    \EE[Y_t(z)^2\mid\cF_t]\notag\\
    &\leq
    \frac{4B^2}{m}
    \EE\bigg[
        \sum_{j=1}^m
    \ind\bigl(
        x_{t,j}=x,\,
        y_{t,j,<h}=u,\,
        a_{t,j,h}=a
    \bigr)\bigg|\,
        \cF_t
    \bigg].
    \label{eq:bernstein-variance}
\end{align}
Define $V_t(z):=\sum_{s=1}^t\Var(Y_s(z)\mid\cF_s)$. Since $\EE[Y_s(z)\mid\cF_s]=0$ and $|Y_s(z)|\leq2B$,
Lemma~\ref{lem:bakhtiari11} implies that both
\begin{align*}
    \bigg(
        \sum_{s=1}^tY_s(z),\, V_t(z)
    \bigg)
    \quad\text{and}\quad
    \bigg(
        -\sum_{s=1}^tY_s(z),\, V_t(z)
    \bigg)
\end{align*}
are sub-gamma processes with parameter $2B/3$. Let $K:=S\sum_{h=1}^H A^h$. 
Then, applying Lemma~\ref{lem:bakhtiari10} with $\rho=2B$, with probability at least $1-\delta/(2K)$, we have
\begin{align}
\label{eq:bernstein}
    \sum_{s=1}^tY_s(z) \le 4 \sqrt{V_t(z) \log (2H_tK/\delta)} + \frac{88}{3}B\log(2H_tK/\delta),
\end{align}
where we can choose $H_t = e+\log(1+T)$ since $V_t(z)\leq4B^2T$. Applying the same inequality to $-Y_s(z)$, and taking the union bound, we obtain that with probability at least $1-\delta$, the following inequality holds for any $h$ and any $z=(x,u,a)$:
\begin{align*}
    \bigg|\sum_{s=1}^tY_s(z)\bigg| \le 4 \sqrt{V_t(z) \log (2H_tK/\delta)} + \frac{88}{3}B\log(2H_tK/\delta).
\end{align*}
Since
\begin{align*}
    \frac1m\sum_{j=1}^m
    \ind\bigl(
        x_{s,j}=x,\,
        y_{s,j,<h}=u,\,
        a_{s,j,h}=a
    \bigr)
    \in[0,1],
\end{align*}
we can apply Lemma~\ref{lem:foster} and obtain that for any $h$ and $z=(x,u,a)$, with probability at least $1-\delta/K$, the following inequality holds simultaneously for all $t \in [T]$:
\begin{align}
\notag
    &\sum_{s=1}^t\EE\bigg[
        \frac{1}{m}\sum_{j=1}^m
        \ind\bigl(
            x_{s,j}=x,\,
            y_{s,j,<h}=u,\,
            a_{s,j,h}=a
        \bigr)
        \,\bigg|\,
        \cF_s
    \bigg]\\\label{eq:Ville}
    &\qquad \le 2 \sum_{s=1}^t\frac{1}{m}\sum_{j=1}^m
    \ind\bigl(
        x_{s,j}=x,\,
        y_{s,j,<h}=u,\,
        a_{s,j,h}=a
    \bigr) + 8 \log (2K/\delta).
\end{align}
Taking a union bound over $h$ and $(x,u,a)$, the inequality holds for all $h$ and $(x,u,a)$ simultaneously. Combining \eqref{eq:bernstein-variance} and \eqref{eq:Ville}, we have
\begin{align}
    V_t(z)
    &\leq
    \frac{8B^2}{m}
    \sum_{s=1}^t\sum_{j=1}^m
    \ind\bigl(
        x_{s,j}=x,\,
        y_{s,j,<h}=u,\,
        a_{s,j,h}=a
    \bigr)
    +
    32B^2\log({2K}/{\delta})\notag\\
    &=
    \frac{8B^2N_t(x,u,a)}m
    +
    32B^2\log({2K}/{\delta}).
    \label{eq:confidence-observed-variance}
\end{align}
Substituting \eqref{eq:confidence-observed-variance} into \eqref{eq:bernstein}, we have, with
probability at least $1-2\delta$, simultaneously for every
$t\in[T]$, $h$, and $z=(x,u,a)$, the following inequality holds
\begin{align*}
    \bigg|\sum_{s=1}^tY_s(z)\bigg|
    &\leq
    4\sqrt{
        \Big(
            \frac{8B^2N_t(x,u,a)}m
            +
            32B^2\log\frac{2K}{\delta}
        \Big)
        \log\frac{2H_tK}{\delta}
    }
    +
    \frac{88}{3}B\log\frac{2H_tK}{\delta}\\
    &\lesssim
    B\sqrt{
        \frac{N_t(x,u,a)}m
        \log\frac{2H_tK}{\delta}
    }
    +
    B\log\frac{2H_tK}{\delta},
\end{align*}
where we use $\sqrt{v+w}\leq\sqrt v+\sqrt w$ and
$\log(2K/\delta)\leq\log(2H_tK/\delta)$.

For every triple with $N_t(x,u,a)\geq1$,
\eqref{eq:confidence-centered-sum} gives
\begin{align*}
    \big|
        \bar l_t(x,u,a)-\log g_h(a|x,u)
    \big|
    &=
    \frac{m}{N_t(x,u,a)}
    \bigg|\sum_{s=1}^tY_s(z)\bigg|\\
    &\lesssim
    B\sqrt{
        \frac{m}{N_t(x,u,a)}
        \log\frac{2H_tK}{\delta}
    }
    +
    \frac{Bm}{N_t(x,u,a)}
    \log\frac{2H_tK}{\delta}.
\end{align*}
Finally, since $K\leq SH A^H$ and $H_t=e+\log(1+T)$,
\begin{align*}
    \log\frac{2H_tK}{\delta}
    \leq
    H\log A
    +
    \log\bigg[
        \frac{2SH\bigl(e+\log(1+T)\bigr)}{\delta}
    \bigg].
\end{align*}
Hiding the logarithmic factors, we can define 
\begin{align*}
    \beta_t(x,u,a) = \tilde O\bigg(
        B\sqrt{\frac{mH}{N_t(x,u,a)}}
        +
        \frac{BmH}{N_t(x,u,a)}
    \bigg),
\end{align*}
and conclude that
\begin{align*}
    \big|
        \bar l_t(x,u,a)-\log g_h(a|x,u)
    \big|
    \leq \beta_t(x,u,a).
\end{align*}
This completes the proof of Lemma \ref{lem:ar-token-confidence}.
\end{proof}
%======================================================================
\section{Missing Proof in Section \ref{sec:ar-target-comparison}}
\label{sec:proof-target-comparison}
\label{sec:proof-uniform-threshold}
\begin{proof}[Proof of Proposition~\ref{prop:uniform-teacher-threshold}]
Fix $N\geq2$ and $\alpha\in(0,1)$, and let $b:=1/N$. For $q\in(b,1)$, write the correct-response probabilities as
\begin{align*}
    F(q)&:=\pif^*(y^\star|x)=\alpha q+(1-\alpha)b,\\
    R(q)&:=\pir^*(y^\star|x)
    =\frac{q^\alpha}{q^\alpha+(N-1)^{1-\alpha}(1-q)^\alpha}.
\end{align*}
Both probabilities lie in $(0,1)$. Since the log-odds function is strictly increasing, $F(q)-R(q)$ has the same sign as
\begin{align*}
    G(q)&:=\log\frac{F(q)}{1-F(q)}-\log\frac{R(q)}{1-R(q)}\\
    &=\log\frac{F(q)}{1-F(q)}-\alpha\log\frac{q}{1-q}
      +(1-\alpha)\log(N-1).
\end{align*}
The function $G$ extends continuously to $q=b$, with $G(b)=0$. Differentiating gives
\begin{align*}
    G'(q)
    &=\frac{\alpha}{F(q)(1-F(q))}-\frac{\alpha}{q(1-q)}\\
    &=\frac{\alpha(q-F(q))(1-q-F(q))}
    {F(q)(1-F(q))q(1-q)}.
\end{align*}
For $q>b$, we have $q-F(q)=(1-\alpha)(q-b)>0$. Thus, the sign of $G'(q)$ is the sign of
\[
    1-q-F(q)=1-(1+\alpha)q-(1-\alpha)b,
\]
which vanishes at $q_{\mathrm m}:=[1-(1-\alpha)b]/(1+\alpha)$.
If $N\geq3$, then $b<q_{\mathrm m}<1$, so $G$ is strictly increasing on $(b,q_{\mathrm m})$ and strictly decreasing on $(q_{\mathrm m},1)$. In particular, $G(q_{\mathrm m})>0$. Moreover, $F(q)\to\alpha+(1-\alpha)b\in(0,1)$ as $q\uparrow1$, and hence $G(q)\to-\infty$. The intermediate value theorem and strict decrease on $(q_{\mathrm m},1)$ imply that there is exactly one zero $q_{\mathrm c}\in(q_{\mathrm m},1)$. We have $G(q)>0$ for $b<q<q_{\mathrm c}$ and $G(q)<0$ for $q_{\mathrm c}<q<1$, giving the claimed ordering.

If $N=2$, then $q_{\mathrm m}=b=1/2$, so $G'(q)<0$ throughout $(1/2,1)$. Since $G(b)=0$, this gives $G(q)<0$ and therefore $R(q)>F(q)$ on that interval.
Finally, both targets are below the expert probability for $q>b$: we have $F(q)<q$, and
\[
    \frac{R(q)}{1-R(q)}
    =\frac{q}{1-q}
      \left(\frac{1-q}{(N-1)q}\right)^{1-\alpha}
    <\frac{q}{1-q},
\]
which implies $R(q)<q$.
\end{proof}

\begin{proof}[Proof of Proposition~\ref{prop:reverse-extreme-value}]
Fix a context $x$ and suppress its dependence in the notation. Since the forward target is the weighted arithmetic mixture of the teachers, summing over $E$ gives
\[
    \pif^*(E)=\sum_i w_i p_i(E)\geq w_kp_k(E),
\]
where the inequality follows from the nonnegativity of every summand.

For the reverse target, fix $\alpha,q,\eta\in(0,1)$ and write $m:=|E|$ and $n:=|\cY\setminus E|$. Both $m$ and $n$ are positive. For $\varepsilon\in(0,1)$, define two teachers by
\[
    p_1(y)=
    \begin{cases}
        q/m, & y\in E,\\
        (1-q)/n, & y\notin E,
    \end{cases}
    \qquad
    p_2(y)=
    \begin{cases}
        \varepsilon/m, & y\in E,\\
        (1-\varepsilon)/n, & y\notin E.
    \end{cases}
\]
These distributions have full support, and $p_1(E)=q$. With weights $\alpha$ and $1-\alpha$, the unnormalized geometric masses on the two sets are
\begin{align*}
    \sum_{y\in E}p_1(y)^\alpha p_2(y)^{1-\alpha}
    &=q^\alpha\varepsilon^{1-\alpha},\\
    \sum_{y\notin E}p_1(y)^\alpha p_2(y)^{1-\alpha}
    &=(1-q)^\alpha(1-\varepsilon)^{1-\alpha}.
\end{align*}
Consequently,
\[
    \pir^*(E)
    =\frac{q^\alpha\varepsilon^{1-\alpha}}
    {q^\alpha\varepsilon^{1-\alpha}+(1-q)^\alpha(1-\varepsilon)^{1-\alpha}}.
\]
As $\varepsilon\downarrow0$, the numerator tends to zero and the denominator tends to $(1-q)^\alpha>0$. Hence $\pir^*(E)\to0$, and a sufficiently small positive $\varepsilon$ ensures $\pir^*(E)<\eta$. This choice retains full support for both teachers and proves the claim.
\end{proof}

\begin{proof}[Proof of Proposition~\ref{prop:reverse-long-horizon}]
Fix $\beta\in(0,1/2)$ and $r\in(1/2,1)$, and set $\alpha:=1-\beta$. We construct the expert using token probabilities that do not depend on the horizon. Fix $\delta\in(0,1/2)$, let $p_1(a|x)=r$, and, at every subsequent position, set
\[
    p_1(a|x,u)=
    \begin{cases}
        1-\delta, & \text{if the first token of $u$ is $a$},\\
        1/2, & \text{if the first token of $u$ is $b$}.
    \end{cases}
\]
The probability of $b$ is the complementary probability at every prefix. Let $p_2$ assign probability $1/2$ to each token at every prefix. Both teachers assign positive probability to every response in $\{a,b\}^H$ for every finite $H$.

Marginalizing the forward mixture over the last $H-1$ tokens yields
\[
    \pif^*(a|x)=\alpha r+\frac{\beta}{2}
    =\frac12+\alpha\left(r-\frac12\right)>\frac12.
\]
Thus $\pif^*(a|x)>\pif^*(b|x)$ for every $H$.

For the reverse target, the uniform teacher contributes the same factor $2^{-\beta H}$ to every complete response, so
\[
    \pir^*(y|x)=\frac{p_1(y|x)^\alpha}
    {\sum_{z\in\{a,b\}^H}p_1(z|x)^\alpha}.
\]
Define $C_\alpha:=(1-\delta)^\alpha+\delta^\alpha$ and $D_\alpha:=2^{1-\alpha}$. Summing the powered expert probabilities over continuations gives
\begin{align*}
    \sum_{v\in\{a,b\}^{H-1}}p_1(av|x)^\alpha
    &=r^\alpha C_\alpha^{H-1},\\
    \sum_{v\in\{a,b\}^{H-1}}p_1(bv|x)^\alpha
    &=(1-r)^\alpha D_\alpha^{H-1}.
\end{align*}
Since $0<\alpha<1$, the function $t\mapsto t^\alpha$ is strictly concave. As $\delta\neq1/2$, strict concavity implies
\[
    C_\alpha=(1-\delta)^\alpha+\delta^\alpha
    <2\left(\frac12\right)^\alpha=D_\alpha.
\]
The reverse target's first-token odds therefore satisfy
\[
    \frac{\pir^*(a|x)}{\pir^*(b|x)}
    =\left(\frac{r}{1-r}\right)^\alpha
      \left(\frac{C_\alpha}{D_\alpha}\right)^{H-1}.
\]
In particular, these odds are strictly less than one whenever
\[
    H>1+\frac{\alpha\log(r/(1-r))}{\log(D_\alpha/C_\alpha)}.
\]
The denominator is strictly positive, so this bound is finite for every fixed $\beta$ and $r$ in the stated ranges. The same construction thus satisfies $\pir^*(a|x)<\pir^*(b|x)$ for every sufficiently large integer $H$, completing the proof.
\end{proof}
\section{Parameter Sensitivity of Aggregation Targets}
\label{sec:aggregation-ablations}
We extend the examples in Section~\ref{sec:ar-target-comparison} by varying the teacher weights, expert confidence, response-space size, and continuation distributions. All curves are deterministic evaluations of the closed-form targets, rather than results of training a student. Each comparison holds the other parameters fixed. The vertical axes show the full probability range, and dashed lines indicate the expert's probability.

\paragraph{Uninformative teachers.}
The expert has weight $\alpha$, assigns probability $q$ to the correct response, and distributes its remaining mass uniformly over the $N-1$ incorrect responses. The other teachers are uniform and have total weight $1-\alpha$. Writing $F(q):=\pif^*(y^\star|x)$ and $R(q):=\pir^*(y^\star|x)$, we evaluate
\[
    F(q)=\alpha q+\frac{1-\alpha}{N},
    \qquad
    R(q)=\frac{q^\alpha}{q^\alpha+(N-1)^{1-\alpha}(1-q)^\alpha}.
\]
Here $N$ counts responses; splitting the uniform teachers' total weight among more identical teachers leaves both targets unchanged.

\begin{figure}[!ht]
    \centering
    \includegraphics[width=\linewidth]{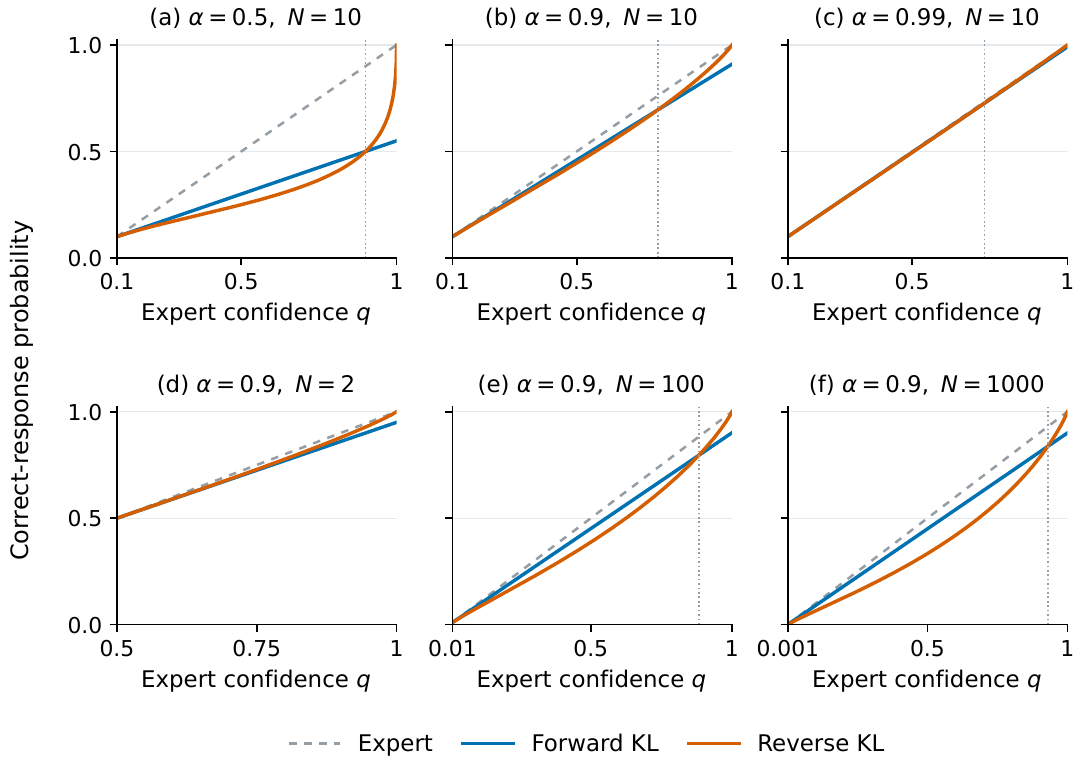}
    \caption{Sensitivity to expert weight and response-space size with uninformative teachers. Panels (a--c) vary $\alpha\in\{0.5,0.9,0.99\}$ at $N=10$. Panels (d--f), together with (b), compare $N\in\{2,10,100,1000\}$ at $\alpha=0.9$. Each curve varies $q$ from $1/N$ toward $1$. Dotted vertical lines mark the interior crossing when it exists.}
    \label{fig:ablation-uniform}
\end{figure}

Figure~\ref{fig:ablation-uniform} shows that the region in which reverse aggregation retains more correct-response probability depends on both parameters. For the weights shown at $N=10$, increasing $\alpha$ lowers the confidence threshold, although the targets become close as both approach the expert. At $\alpha=0.9$, increasing $N$ from $10$ to $1000$ raises the threshold, narrowing the range of $q$ with a reverse advantage. The binary case differs: reverse aggregation is larger throughout $q\in(1/2,1)$, with no interior crossing.

\paragraph{Misleading teachers.}
We fix $N=10$ and replace the uniform teacher by a teacher of weight $1-\alpha$ that assigns probability $\varepsilon$ to the correct response. Both teachers distribute their remaining probability uniformly over the incorrect responses. The targets are
\[
    F(\varepsilon)=\alpha q+(1-\alpha)\varepsilon,
    \qquad
    R(\varepsilon)=\frac{q^\alpha\varepsilon^{1-\alpha}}
    {q^\alpha\varepsilon^{1-\alpha}+(1-q)^\alpha(1-\varepsilon)^{1-\alpha}}.
\]
Figure~\ref{fig:ablation-misleading} varies the expert weight at fixed $q=0.99$, and its confidence at fixed $\alpha=0.9$. The horizontal coordinate is $\log_{10}\varepsilon$; we evaluate the reverse target in the log domain to retain accuracy for extremely small probabilities.

\begin{figure}[!ht]
    \centering
    \includegraphics[width=\linewidth]{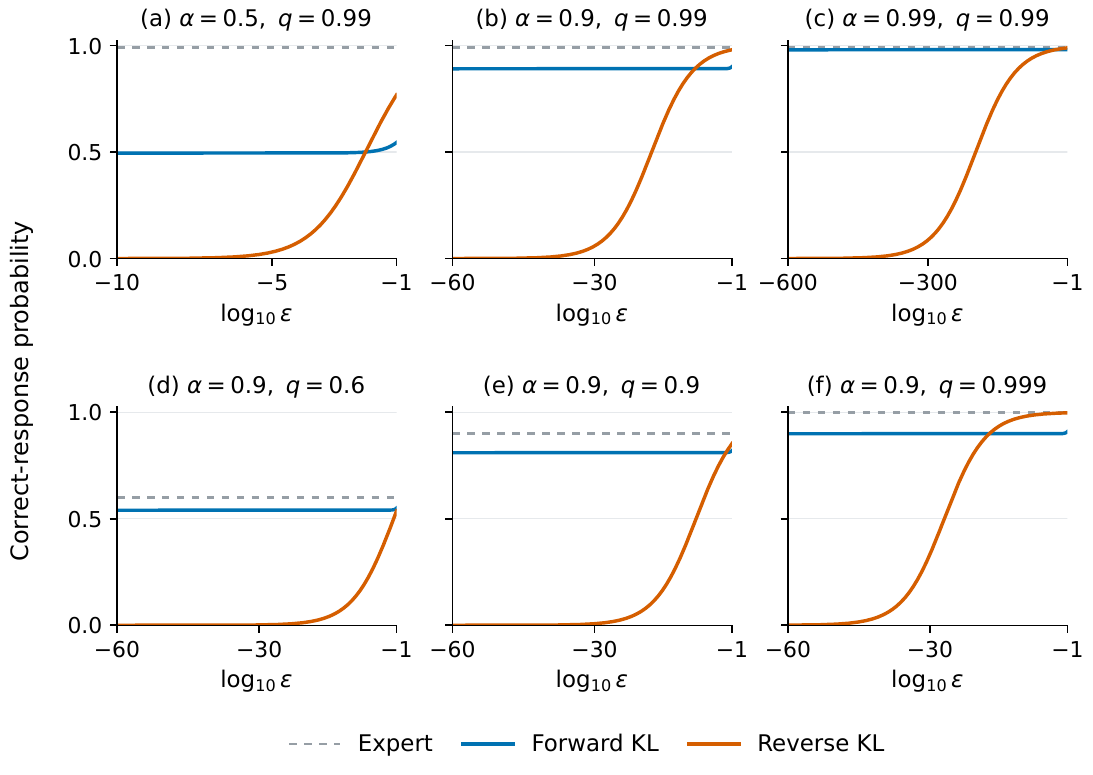}
    \caption{Sensitivity to expert weight and confidence with misleading teachers. Panels (a--c) fix $q=0.99$ and vary $\alpha$. Panels (d--f), together with (b), fix $\alpha=0.9$ and compare $q\in\{0.6,0.9,0.99,0.999\}$. Horizontal ranges differ: $\log_{10}\varepsilon\in[-10,-1]$ in (a), $[-600,-1]$ in (c), and $[-60,-1]$ otherwise. Moving left makes the other teacher more misleading.}
    \label{fig:ablation-misleading}
\end{figure}

Higher expert weight or confidence shifts the decline in reverse probability to smaller values of $\varepsilon$. This protection can be substantial: at $q=0.99$, the reverse probability equals $1/2$ at approximately $\varepsilon=10^{-18}$ for $\alpha=0.9$, compared with $10^{-198}$ for $\alpha=0.99$. Nevertheless, every fixed $\alpha<1$ permits suppression as $\varepsilon\downarrow0$, whereas forward aggregation remains above $\alpha q$. The extreme ranges illustrate the absence of a uniform positive lower bound for reverse aggregation; they do not estimate how often such predictions arise in practice.

\paragraph{Long horizons.}
The expert selects the correct first token $a$ with probability $r$. At every subsequent position, its probabilities for $(a,b)$ are $(1-\delta,\delta)$ after an initial $a$ and $(1/2,1/2)$ after an initial $b$. The other teacher is uniform at every prefix. With expert weight $\alpha$, let $C_\alpha=(1-\delta)^\alpha+\delta^\alpha$ and $D_\alpha=2^{1-\alpha}$. The first-token probabilities are
\[
    F_H=\alpha r+\frac{1-\alpha}{2},
    \qquad
    R_H=\left[1+\left(\frac{1-r}{r}\right)^\alpha
    \left(\frac{D_\alpha}{C_\alpha}\right)^{H-1}\right]^{-1}.
\]
Figure~\ref{fig:ablation-horizon} varies $\alpha$, $r$, and $\delta$ separately around $(\alpha,r,\delta)=(0.9,0.99,0.01)$, evaluating every integer $H$ from $1$ to $1000$.

\begin{figure}[!ht]
    \centering
    \includegraphics[width=\linewidth]{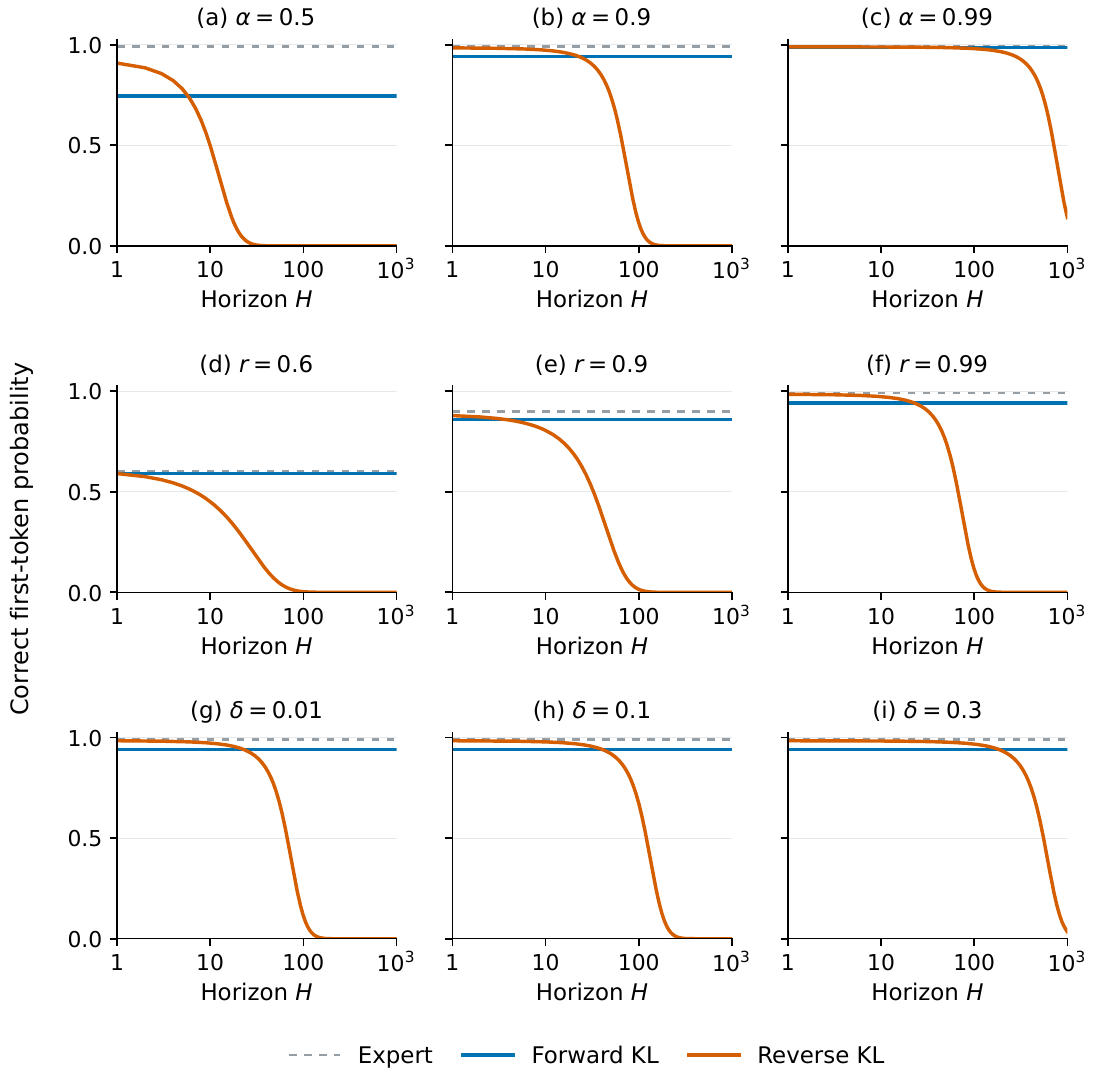}
    \caption{Sensitivity of first-token probabilities to the horizon. Top row: vary $\alpha$ with $r=0.99$ and $\delta=0.01$. Middle row: vary $r$ with $\alpha=0.9$ and $\delta=0.01$. Bottom row: vary $\delta\in\{0.01,0.1,0.3\}$ with $\alpha=0.9$ and $r=0.99$. The baseline is repeated in (b), (f), and (g). The horizontal axis is logarithmic; curves connect values computed at integer horizons.}
    \label{fig:ablation-horizon}
\end{figure}

Increasing the expert's weight or first-token confidence delays the reversal in the displayed settings. At $r=0.99$ and $\delta=0.01$, the first horizon with $R_H<1/2$ is $10$, $68$, or $717$ for $\alpha=0.5$, $0.9$, or $0.99$, respectively. Making the continuations after $a$ more diffuse also delays reversal. For $\delta=0.3$, the reverse probability first falls below $1/2$ at $H=555$, compared with $H=68$ at $\delta=0.01$. Forward aggregation is independent of $H$ in every panel. Thus, in this construction, the deterioration arises from the difference between the continuation distributions, while expert confidence and weight determine how long the initial preference persists.

\section{Forward KL with Function Approximation}
\label{sec:fa-forward}

The tabular estimator learns a separate conditional distribution at each
prefix. Function approximation allows these conditionals to share a common
representation. Let $\Pi=\{\pi_\theta:\theta\in\Theta\}$ be a finite class of
autoregressive policies, where
\begin{align*}
    \pi_\theta(y|x)
    =
    \prod_{h=1}^H
    \pi_\theta(a_h|x,y_{<h}).
\end{align*}
We make the standard realizability assumption for the function class.

\begin{assumption}
\label{assump:fa-forward-realizability}
There exists $\theta^*\in\Theta$ such that, for every context $x$ and every response $y\in\cY$,
\begin{align*}
    \pi_{\theta^*}(y|x)=\pif^*(y|x).
\end{align*}
\end{assumption}
\subsection{Algorithm Design}
We aggregate the candidate policies at the token level, using sampled
tokens when logits are unavailable and full teacher probabilities
otherwise. Recall the token feedback $q_{t,j,h}(a)$ defined in Section~\ref{sec:ar-fkl-tabular}: 
\begin{align}
\notag
    q_{t,j,h}(a):=
    \begin{cases}
        \ind(a_{t,j,h}=a)
        & \mathrm{w/o \ logit},\\
        \frac{\exp(Z_{t,j,h}(a))}
             {\sum_{b\in\cA}\exp(Z_{t,j,h}(b))},
        & \mathrm{w/ \ logit}.
    \end{cases}
\end{align}
For each $\theta \in \Theta$, define its empirical loss in round $t$ by
\begin{align}
\label{eq:empirical-loss}
    \ell_t(\pi_\theta)
    :=
    -\frac1m
    \sum_{j=1}^m\sum_{h=1}^H\sum_{a\in\cA}
    q_{t,j,h}(a)
    \log\pi_\theta(a|x_{t,j},y_{t,j,<h}).
\end{align}
The main idea of the algorithm is to maintain a distribution $p_t(\theta)$ over the parameter space $\Theta$. 
More specifically, choose a prior distribution $p_0$ on $\Theta$ with
$p_0(\theta)>0$ for every $\theta\in\Theta$, and initialize
$L_0(\theta)=0$. At the beginning of round $t$, for any prefix $(x,u)$, the learner uses
\begin{align}
    \hat\pi_t(a|x,u)
    :=
    \sum_{\theta\in\Theta}
    p_{t-1}(\theta)\pi_\theta(a|x,u).
    \label{eq:fa-forward-mixture}
\end{align}
After observing the feedback, the learner updates
\begin{align}
    % L_t(\theta)
    % &:=
    % L_{t-1}(\theta)+\ell_t(\pi_\theta),\notag\\
    p_t(\theta)
    &:=
    \frac{
        p_{t-1}(\theta)
        \exp\bigl(-\ell_t(\pi_\theta)/H\bigr)
    }{
        \sum_{\theta'\in\Theta}
        p_{t-1}(\theta')
        \exp\bigl(-\ell_t(\pi_{\theta'})/H\bigr)
    }.\label{eq:fa-forward-weight-update}
    % \notag
    % \\
    % &=
    % \frac{
    %     q_0(\theta)\exp\bigl(-L_t(\theta)/H\bigr)
    % }{
    %     \sum_{\vartheta\in\Theta}
    %     q_0(\vartheta)\exp\bigl(-L_t(\vartheta)/H\bigr)
    % }.
\end{align}
The factor $1/H$ normalizes the sum of the $H$ token losses in each
rollout. Intuitively, the update favors candidates with smaller cumulative empirical losses, relative to their prior weights.
\begin{algorithm}[H]
\caption{Forward KL with Function Approximation}
\label{algo:fa-forward}
\begin{algorithmic}[1]
    \STATE \textbf{Input:} Finite policy class
    $\Pi=\{\pi_\theta:\theta\in\Theta\}$,
    prior $p_0$ with $p_0(\theta)>0$ for every $\theta$,
    number of rounds $T$.
    \FOR{$t=1,\ldots,T$}
        \STATE For any prefix $(x,u)$, construct the policy $\hat\pi_t(a|x,u)$ as in \eqref{eq:fa-forward-mixture}.
        \STATE Observe the batch of teacher-generated rollouts
        $\{(x_{t,j},y_{t,j})\}_{j=1}^m$ and, when available, the teacher logits $\{Z_{t,j,h}\}_{j,h}$.
        \STATE For every $\theta\in\Theta$, compute $\ell_t(\pi_\theta)$ as in \eqref{eq:empirical-loss}
        \STATE For every $\theta\in \Theta$, update the mixture weights $p_t(\theta)$ as in \eqref{eq:fa-forward-weight-update}
    \ENDFOR
    \STATE \textbf{Output:} $\{\hat\pi_t\}_{t=1}^T$.
\end{algorithmic}
\end{algorithm}
This exponential reweighting resembles the posterior updates used in Thompson sampling. Unlike Thompson sampling, however, our algorithm predicts with a mixture of the candidate policies rather than a single sampled candidate. Evaluating this mixture at each decoding step introduces large computational cost, which is a key limitation of our forward-KL algorithm with function approximation.
\subsection{Theoretical Guarantee}
In this section, we first work on a finite function class $\Theta$. We will apply the results to the setting with finite covering number later. The following theorem bounds the regret by the prior weight of a good
candidate and its approximation error. Under realizability, the
approximation term vanishes.

\begin{theorem}
\label{thm:ar-fa-forward}
Let $T \ge 2$. Suppose the regret in \eqref{eq:regret} is defined with
$D(p\|q)=\KL(q\|p)$. Under off-policy distillation protocol and the function approximation setting, both with and without access to teacher logits, Algorithm \ref{algo:fa-forward} satisfies
\begin{align}
    \EE\big[\operatorname{Regret}(T)\big]
    &\leq
    \inf_{\theta\in\Theta}
    \bigg\{
        H\log\frac1{p_0(\theta)}
        +
        T\EE_{x\sim\bar\rho}
        \KL\bigl(
            \pif^*(\cdot|x)
            \,\big\|\,
            \pi_\theta(\cdot|x)
        \bigr)
    \bigg\},
    \notag
\end{align}
provided the right-hand side is finite. In particular, under
Assumption~\ref{assump:fa-forward-realizability}, for the optimal parameter $\theta^*$, Algorithm \ref{algo:fa-forward} satisfies
\begin{align}
    \EE\big[\operatorname{Regret}(T)\big]
    \leq
    H\log\frac1{p_0(\theta^*)}.
    \notag
\end{align}
Moreover, for any $\delta>0$, under Assumption \ref{assump:fa-forward-realizability}, the following inequality holds with probability at least $1-\delta$:
\begin{align*}
    \operatorname{Regret}(T)
    \leq
    2H\log\frac1{p_0(\theta^*)}
    +
    4H\bigg(
        1+\log\frac{2}{\delta p_0(\theta^*)}
    \bigg)\log\frac2\delta.
\end{align*}
\end{theorem}

\begin{proof}
Let $\cF_t$ be the $\sigma$-algebra generated by all rollouts and
feedback observed before round $t$:
\begin{align*}
    \cF_t
    :=
    \sigma\Big(
        (x_{s,j},y_{s,j}),q_{s,j,h}(a):
        s<t,\ j\in[m],\ h\in[H],\ a\in\cA
    \Big).
\end{align*}
Thus, both $p_{t-1}$ and $\hat\pi_t$ are $\cF_t$-measurable. Recall the feedback counts
$c_t(x,u,a)$ defined in \eqref{eq:fkl-count}. For any policy $\pi$, \eqref{eq:empirical-loss} can be rewritten as
\begin{align}
    \ell_t(\pi)
    =
    -\sum_{h=1}^H
    \sum_{(x,u,a)}
    c_t(x,u,a)\log\pi(a|x,u).
    \label{eq:fa-proof-token-loss}
\end{align}
Easy to see, the following equation holds
\begin{align}
\notag
    \sum_{(x,u,a)}c_t(x,u,a)=1
\end{align}
Therefore, we have
\begin{align}  \sum_{h=1}^H\sum_{(x,u,a)}\frac{c_t(x,u,a)}{H}=1.
    \label{eq:fa-proof-feedback-mass}
\end{align}
For any $t\ge 1$, define the normalizing constant
\begin{align*}
     C_t
    :=
    \sum_{\theta\in\Theta}
    p_0(\theta)
    \exp\bigg(
        -\frac1H\sum_{s=1}^t\ell_s(\pi_\theta)
    \bigg),
\end{align*}
and $C_0 =1$. Direct iteration of \eqref{eq:fa-forward-weight-update} gives
\begin{align}
\notag
    p_t(\theta)
    =
    \frac{
        p_0(\theta)
        \exp\big(
            -H^{-1}\sum_{s=1}^t\ell_s(\pi_\theta)
        \big)
    }{
         C_t
    }.
\end{align}
Moreover, we have
\begin{align}
\notag
    \frac{C_t}{C_{t-1}}
    &=
    \frac1{C_{t-1}}
    \sum_{\theta'\in\Theta}
    p_0(\theta')
    \exp\bigg(
        -\frac1H\sum_{s=1}^{t-1}\ell_s(\pi_{\theta'})
        -
        \frac{\ell_t(\pi_{\theta'})}{H}
    \bigg)\\\notag
    &=
    \sum_{\theta'\in\Theta}
    \underbrace{
        \frac{
            p_0(\theta')
            \exp\big(
                -H^{-1}\sum_{s=1}^{t-1}\ell_s(\pi_{\theta'})
            \big)
        }{C_{t-1}}
    }_{=\,p_{t-1}(\theta')}
    \exp\bigl(-\ell_t(\pi_{\theta'})/H\bigr)\\\label{eq:normalize-ratio}
    &=
    \sum_{\theta'\in\Theta}
    p_{t-1}(\theta')
    \exp\bigl(-\ell_t(\pi_{\theta'})/H\bigr).
\end{align}
The definition of $\ell_t$ in \eqref{eq:fa-proof-token-loss} gives
\begin{align}
\notag
    \exp\bigl(-\ell_t(\pi_{\theta'})/H\bigr)
    &=
    \exp\bigg(
        \sum_{h=1}^H\sum_{(x,u,a)}
        \frac{c_t(x,u,a)}H
        \log\pi_{\theta'}(a|x,u)
    \bigg)\\\label{eq:Holder-prep}
    &=
    \prod_{h=1}^H\prod_{(x,u,a)}
    \pi_{\theta'}(a|x,u)^{c_t(x,u,a)/H}.
\end{align}
Substituting \eqref{eq:Holder-prep} into \eqref{eq:normalize-ratio}, we have
\begin{align}
    \frac{C_t}{C_{t-1}}
    =
    \sum_{\theta'\in\Theta}
    p_{t-1}(\theta')
    \prod_{h=1}^H\prod_{(x,u,a)}
    \pi_{\theta'}(a|x,u)^{c_t(x,u,a)/H}.
    \label{eq:fa-proof-potential-ratio}
\end{align}
Using the generalized
H\"older's inequality, we have
\begin{align}
    \frac{C_t}{C_{t-1}}
    &\leq
    \prod_{h=1}^H\prod_{(x,u,a)}
    \bigg(
        \sum_{\theta'\in\Theta}
        p_{t-1}(\theta')\pi_{\theta'}(a|x,u)
    \bigg)^{c_t(x,u,a)/H}
    \notag\\
    &=
    \prod_{h=1}^H\prod_{(x,u,a)}
    \hat\pi_t(a|x,u)^{c_t(x,u,a)/H}
    \notag\\
    &=
    \exp\bigl(-\ell_t(\hat\pi_t)/H\bigr),
    \notag
\end{align}
where the last equation holds using \eqref{eq:fa-proof-token-loss}.

Taking the negative logarithm, and summing over $t \in [T]$ yields
\begin{align}
    \sum_{t=1}^T\ell_t(\hat\pi_t)
    &\leq
    -H\sum_{t=1}^T\log\frac{C_t}{C_{t-1}}
    =
    -H\log C_T,
    \label{eq:fa-proof-telescope}
\end{align}
where the equality uses $C_0=1$.

Now fix any comparator policy $\pi_\theta$ such that $\theta \in \Theta$. Easy to see
\begin{align*}
    C_T &= \sum_{\theta'}p_0(\theta')
    \exp\bigg(
        -\frac1H\sum_{t=1}^T\ell_t(\pi_{\theta'})
    \bigg)\\
    &\geq
    p_0(\theta)
    \exp\bigg(
        -\frac1H\sum_{t=1}^T\ell_t(\pi_\theta)
    \bigg).
\end{align*}
Therefore, we have
\begin{align}
    -\sum_{t=1}^T \ell_t(\pi_\theta) \le H \log C_T + H \log \frac{1}{p_0(\theta)}.
    \label{eq:fa-proof-theta}
\end{align}
Summing up \eqref{eq:fa-proof-telescope} and \eqref{eq:fa-proof-theta}, we have
\begin{align}
    \sum_{t=1}^T
    \big[
        \ell_t(\hat\pi_t)-\ell_t(\pi_\theta)
    \big]
    \leq
    H\log\frac1{p_0(\theta)}.
    \label{eq:fa-proof-comparator}
\end{align}
As a result, 
\begin{align}
    \sum_{t=1}^T
    \big[
        \ell_t(\hat\pi_t)-\ell_t(\pif^*)
    \big] &= \sum_{t=1}^T
    \big[
        \ell_t(\hat\pi_t)-\ell_t(\pi_{\theta})
    \big] + \sum_{t=1}^T
    \big[
        \ell_t(\pi_\theta)-\ell_t(\pif^*)
    \big]
    \notag\\
    &\qquad\leq
    H\log\frac1{p_0(\theta)}
    +
    \sum_{t=1}^T
    \big[
        \ell_t(\pi_\theta)-\ell_t(\pif^*)
    \big].
    \label{eq:fa-proof-empirical-oracle}
\end{align}
Using the same argument as in Appendix \ref{sec:proof-fkl-thm}, we have
\begin{align}
    \EE[c_t(x,u,a)\mid\cF_t]
    &=
    \frac1I\sum_{i\in\cI}
    \rho_i(x)p_i(u|x)p_i(a|x,u)
    \notag\\
    &=
    \bar\rho(x)\pif^*(u|x)\pif^*(a|x,u).
    \notag
\end{align}
Consequently, using \eqref{eq:fa-proof-token-loss}, for any $\cF_t$-measurable policy $\pi$ with finite
loss,
\begin{align}
    &\EE\big[
        \ell_t(\pi)-\ell_t(\pif^*)
        \mid\cF_t
    \big]
    \notag\\
    &\quad=
    \sum_{h=1}^H\sum_{(x,u,a)}
    \EE[c_t(x,u,a)\mid\cF_t]
    \log\frac{\pif^*(a|x,u)}{\pi(a|x,u)}
    \notag\\
    &\quad=
    \sum_{h=1}^H\sum_{(x,u)}
    \bar\rho(x)\pif^*(u|x)
    \KL\big(
        \pif^*(\cdot|x,u)
        \,\big\|\,
        \pi(\cdot|x,u)
    \big)
    \notag\\
    &\quad=
    \EE_{x\sim\bar\rho}
    \KL\big(
        \pif^*(\cdot|x)
        \,\big\|\,
        \pi(\cdot|x)
    \big).
    \label{eq:fa-proof-loss-to-kl}
\end{align}
The last equality follows from the chain rule for KL divergence.

Taking expectations in \eqref{eq:fa-proof-empirical-oracle} and
applying \eqref{eq:fa-proof-loss-to-kl} to both $\hat\pi_t$ and
$\pi_\theta$, we conclude that
\begin{align*}
    \EE[\operatorname{Regret}(T)]
    &\leq
    H\log\frac1{p_0(\theta)}
    +
    \sum_{t=1}^T
    \EE_{x\sim\bar\rho}
    \KL\big(
        \pif^*(\cdot|x)
        \,\big\|\,
        \pi_\theta(\cdot|x)
    \big)\\
    &=
    H\log\frac1{p_0(\theta)}
    +
    T\EE_{x\sim\bar\rho}
    \KL\big(
        \pif^*(\cdot|x)
        \,\big\|\,
        \pi_\theta(\cdot|x)
    \big).
\end{align*}
Since the inequality holds for every fixed $\theta$, taking the infimum gives
\begin{align}
    \EE[\operatorname{Regret}(T)]
    \leq
    \inf_{\theta\in\Theta}
    \bigg\{
        H\log\frac1{p_0(\theta)}
        +
        T\EE_{x\sim\bar\rho}
        \KL\big(
            \pif^*(\cdot|x)
            \,\big\|\,
            \pi_\theta(\cdot|x)
        \big)
    \bigg\}.
    \label{eq:fa-proof-oracle}
\end{align}

\paragraph{High-probability bound under realizability.}
Suppose $\pi_{\theta^*}=\pif^*$ for some $\theta^*\in\Theta$.
As in Appendix~\ref{sec:proof-fkl-thm}, define
\begin{align*}
    X_t
    &:=
    \ell_t(\pif^*)-\ell_t(\hat\pi_t),\\
    r_t
    &:=
    -\EE[X_t\mid\cF_t].
\end{align*}
By \eqref{eq:fa-proof-loss-to-kl},
$\operatorname{Regret}(T)=\sum_{t=1}^T r_t$. The argument establishing \eqref{eq:proof-Xt-central-condition}
applies to any $\cF_t$-measurable policy $\pi$, since it only
uses policy normalization and the conditional mean of the
feedback counts. Therefore,
\begin{align}
    \EE\Big[\exp\Big(\frac{X_t}{H}\Big)\Big| \cF_t\Big] = \EE\bigg[
        \exp\bigg(
            \frac{\ell_t(\pif^*)-\ell_t(\hat\pi_t)}H
        \bigg)
        \,\bigg|\,\cF_t
    \bigg]
    \leq1.
    \label{eq:fa-hp-central}
\end{align}

To replace the lower bound provided by KT smoothing in the
tabular proof, we control the mixture weight of $\theta^*$.
Define $W_0=1$ and
\begin{align*}
    W_t
    &:=
    C_t\exp\bigg(
        \frac1H\sum_{s=1}^t\ell_s(\pif^*)
    \bigg)
    =
    \frac{p_0(\theta^*)}{p_t(\theta^*)}.
\end{align*}
The equality follows from the formula for $p_t(\theta^*)$
and realizability. By \eqref{eq:normalize-ratio},
\begin{align*}
    \frac{W_t}{W_{t-1}}
    &=
    \frac{C_t}{C_{t-1}}
    \exp\bigl(\ell_t(\pif^*)/H\bigr)\\
    &=
    \sum_{\theta\in\Theta}
    p_{t-1}(\theta)
    \exp\bigg(
        \frac{\ell_t(\pif^*)-\ell_t(\pi_\theta)}H
    \bigg),
\end{align*}
where we apply \eqref{eq:normalize-ratio}. 
Since $W_{t-1}$ and $p_{t-1}$ are $\cF_t$-measurable,
\eqref{eq:fa-hp-central} implies
\begin{align*}
    \EE[W_t\mid\cF_t]
    &=
    W_{t-1}\sum_{\theta\in\Theta}
    p_{t-1}(\theta)
    \EE\bigg[
        \exp\bigg(
            \frac{\ell_t(\pif^*)-\ell_t(\pi_\theta)}H
        \bigg)
        \,\bigg|\,\cF_t
    \bigg]\\
    &\leq W_{t-1},
\end{align*}
where the last inequality holds due to \eqref{eq:fa-hp-central}. Thus, $W_t$ is a nonnegative supermartingale with respect
to $\{\cF_{t+1}\}_{t=0}^T$. Using Ville's inequality (Lemma \ref{lem:Ville}), we have
\begin{align}
    \PP\bigg(
        \max_{0\leq t\leq T}W_t\leq\frac2\delta
    \bigg)
    \geq1-\frac\delta2.
    \label{eq:fa-hp-weight-control}
\end{align}
Define
\begin{align*}
    J_t
    &:=
    \ind\bigg(
        \max_{0\leq s<t}W_s\leq\frac2\delta
    \bigg),
    \qquad
    L:=\log\frac{2}{\delta p_0(\theta^*)}.
\end{align*}
Whenever $J_t=1$, the token-level mixture satisfies
\begin{align*}
    \frac{\hat\pi_t(a|x,u)}{\pif^*(a|x,u)}
    \geq
    p_{t-1}(\theta^*)
    =
    \frac{p_0(\theta^*)}{W_{t-1}}
    \geq e^{-L},
\end{align*}
where the first inequality holds due to $\hat \pi_t(a|x,u):= \sum_{\theta} p_{t-1}(\theta) \pi_{\theta}(a|x,u) \ge p_{t-1}(\theta^*) \pif^*(a|x,u)$. The second inequality holds due to \eqref{eq:fa-hp-weight-control}. Thus,
\begin{align*}
    \log\frac{\hat\pi_t(a|x,u)}{\pif^*(a|x,u)}
    \geq-L
\end{align*}
for every $(x,u,a)$.
Using the definition of $X_t$ and the loss representation, we obtain
\begin{align*}
    \frac{X_t}{H}
    &=
    \frac{\ell_t(\pif^*)-\ell_t(\hat\pi_t)}H\\
    &=
    \sum_{h=1}^H\sum_{(x,u,a)}
    \frac{c_t(x,u,a)}H
    \log\frac{\hat\pi_t(a|x,u)}{\pif^*(a|x,u)}\\
    &\geq
    -L\sum_{h=1}^H\sum_{(x,u,a)}
    \frac{c_t(x,u,a)}H\\
    &=-L.
\end{align*}
The inequality uses $c_t(x,u,a)\geq0$, and the last equality
follows from \eqref{eq:fa-proof-feedback-mass}.

When $J_t=0$, we instead have $J_tX_t/H=0\geq-L$, since $L\geq0$.
Combining the two cases gives
\begin{align*}
    \frac{J_tX_t}{H}\geq-L.
\end{align*}
Moreover, since $J_t$ is $\cF_t$-measurable, 
\begin{align*}
    \EE[e^{J_tX_t/H}\mid\cF_t]
    &=1 * \ind(J_t=0)+\ind(J_t=1) \EE[e^{X_t/H}\mid\cF_t]\\ 
    &=1-J_t+J_t\EE[e^{X_t/H}\mid\cF_t]\\
    &\leq1,
\end{align*}
where the last inequality holds due to \eqref{eq:fa-hp-central}.

Hence, Lemma~\ref{lem:central-to-mgf} and the argument leading
to \eqref{eq:fkl-martingale-bound} apply to $J_tX_t$, whose conditional means are $-J_tr_t$.
Replacing $L_T$ by $L$ and the failure probability by $\delta/2$,
we obtain, for any fixed $\lambda\in(0,1/H]$, with probability
at least $1-\delta/2$,
\begin{align*}
    \sum_{t=1}^T J_t(X_t+r_t)
    \leq
    H(1+L)\lambda\sum_{t=1}^T J_tr_t
    +
    \frac{\log(2/\delta)}{\lambda}.
\end{align*}

Choose $\lambda=[2H(1+L)]^{-1}$.
Conditioned on the event in \eqref{eq:fa-hp-weight-control},
and taking a union bound, we conclude
that, with probability at least $1-\delta$,
\begin{align*}
    \operatorname{Regret}(T)
    \leq
    -\sum_{t=1}^T X_t
    +
    \frac12\operatorname{Regret}(T)
    +
    2H(1+L)\log\frac2\delta.
\end{align*}
Finally, \eqref{eq:fa-proof-comparator} with
$\theta=\theta^*$ gives
\begin{align*}
    -\sum_{t=1}^T X_t
    \leq H\log\frac1{p_0(\theta^*)}.
\end{align*}
Substituting this bound and rearranging yields
\begin{align*}
    \operatorname{Regret}(T)
    \leq
    2H\log\frac1{p_0(\theta^*)}
    +
    4H\left(
        1+\log\frac{2}{\delta p_0(\theta^*)}
    \right)\log\frac2\delta.
\end{align*}
\end{proof}

\subsection{Function Class with Finite Covering Number}
Let $\Pi=\{\pi_\theta:\theta\in\Theta\}$ be a class of
autoregressive policies containing $\pif^*$.
For $\varepsilon>0$, we define the $\varepsilon$-covering number
of the token log-probability class $\log\Pi$ as follows.

\begin{definition}[Covering number]
\label{def:log-policy-covering}
The covering number
$\mathcal N(\varepsilon,\log\Pi,\|\cdot\|_\infty)$
is the smallest cardinality of a finite subset
$\Pi_\varepsilon\subseteq\Pi$ such that, for every $\pi\in\Pi$,
there exists $\tilde\pi\in\Pi_\varepsilon$ satisfying
\begin{align*}
    \|\log\pi-\log\tilde\pi\|_\infty
    &:=
    \sup_{\substack{
        h\in[H],\,x\in\cX\\
        u\in\cA^{h-1},\,a\in\cA
    }}
    \big|
        \log\pi(a|x,u)
        -
        \log\tilde\pi(a|x,u)
    \big|
    \leq\varepsilon.
\end{align*}
\end{definition}
When the function class $\Pi$ has finite covering number $N_\varepsilon:=\mathcal N(\varepsilon,\log\Pi,\|\cdot\|_\infty)<\infty$, we can choose an $\varepsilon$-cover $\Pi_\varepsilon\subseteq\Pi$
of cardinality $N_\varepsilon$, and run Algorithm \ref{algo:fa-forward} on $\Pi_\varepsilon$ with a prior distribution $p_0$. Then, we have the following regret guarantee:
\begin{theorem}
\label{thm:fa-forward-covering}
Let $\Pi$ be a class of autoregressive policies containing
$\pif^*$. Fix $\varepsilon>0$ such that
\begin{align*}
    N_\varepsilon
    :=
    \mathcal N(\varepsilon,\log\Pi,\|\cdot\|_\infty)
    <\infty.
\end{align*}
Let the algorithm be described as above and $p_0$ be the uniform prior. Under off-policy feedback, both with and without teacher logits, the expected regret can be bounded by
\begin{align}
    \EE[\operatorname{Regret}(T)]
    \leq
    H\log N_\varepsilon+TH\varepsilon.
    \notag
\end{align}
Moreover, for every $\delta\in(0,1)$, with probability at least
$1-\delta$,
\begin{align}
    \operatorname{Regret}(T)
    &\leq
    2H\bigl(\log N_\varepsilon+T\varepsilon\bigr)
    \notag+
    4H\Big(
        1+\log N_\varepsilon+T\varepsilon
        +\log\frac2\delta
    \Big)\log\frac2\delta.
\end{align}
\end{theorem}

\begin{proof}
Since $\pif^*\in\Pi$, the definition of the $\varepsilon$-cover
ensures that there exists $\pi_{\bar\theta}\in\Pi_\varepsilon$
such that
\begin{align}
    \|\log\pif^*-\log\pi_{\bar\theta}\|_\infty
    \leq\varepsilon.
    \label{eq:covering-target-approximation}
\end{align}
In particular, for every prefix--token triple $(x,u,a)$,
\begin{align*}
    \log\frac{\pif^*(a|x,u)}
             {\pi_{\bar\theta}(a|x,u)}
    \leq\varepsilon.
\end{align*}
Using the autoregressive factorization, for every complete
response $y=(a_1,\ldots,a_H)$, we obtain
\begin{align*}
    \log\frac{\pif^*(y|x)}{\pi_{\bar\theta}(y|x)}
    &=
    \sum_{h=1}^H
    \log\frac{\pif^*(a_h|x,y_{<h})}
             {\pi_{\bar\theta}(a_h|x,y_{<h})}\\
    &\leq H\varepsilon.
\end{align*}
Taking expectations over $x\sim\bar\rho$ and
$y\sim\pif^*(\cdot|x)$ therefore gives
\begin{align}
    \EE_{x\sim\bar\rho}
    \KL\big(
        \pif^*(\cdot|x)
        \,\big\|\,
        \pi_{\bar\theta}(\cdot|x)
    \big)
    \leq H\varepsilon.
    \label{eq:covering-kl-approximation}
\end{align}

The uniform prior on the cover assigns
$p_0(\bar\theta)=1/N_\varepsilon$.
Applying Theorem \ref{thm:ar-fa-forward}, we have
\begin{align*}
    \EE[\operatorname{Regret}(T)]
    &\leq
    H\log\frac1{p_0(\bar\theta)}
    +
    T\EE_{x\sim\bar\rho}
    \KL\big(
        \pif^*(\cdot|x)
        \,\big\|\,
        \pi_{\bar\theta}(\cdot|x)
    \big)\\
    &\leq
    H\log N_\varepsilon+TH\varepsilon,
\end{align*}
where we use that $p_0$ is uniform and \eqref{eq:covering-kl-approximation}. This has proved the expected regret bound. 

For the high-probability bound, let $\Theta_\varepsilon$ index the $N_\varepsilon$ policies
in $\Pi_\varepsilon$, with
$p_0(\theta)=1/N_\varepsilon$ for every
$\theta\in\Theta_\varepsilon$.
We use the same filtration $\cF_t$ as in the finite-class
analysis and define
\begin{align*}
    X_t
    &:=
    \ell_t(\pif^*)-\ell_t(\hat\pi_t),\\
    r_t
    &:=
    -\EE[X_t\mid\cF_t].
\end{align*}
By \eqref{eq:fa-proof-loss-to-kl},
$\operatorname{Regret}(T)=\sum_{t=1}^T r_t$. For the policy $\pi_{\bar\theta}$ in
\eqref{eq:covering-target-approximation},
\begin{align}
    \ell_t(\pi_{\bar\theta})-\ell_t(\pif^*)
    &=
    \sum_{h=1}^H\sum_{(x,u,a)}
    c_t(x,u,a)
    \log\frac{\pif^*(a|x,u)}
             {\pi_{\bar\theta}(a|x,u)}
    \notag\\
    &\leq
    \varepsilon
    \sum_{h=1}^H\sum_{(x,u,a)}c_t(x,u,a)
    =
    H\varepsilon.
    \label{eq:covering-empirical-approximation}
\end{align}
Combining \eqref{eq:covering-empirical-approximation} with
\eqref{eq:fa-proof-comparator} gives
\begin{align}
    -\sum_{t=1}^T X_t
    &=
    \sum_{t=1}^T
    [\ell_t(\hat\pi_t)-\ell_t(\pif^*)]
    \notag\\\notag
    &=\sum_{t=1}^T
    [\ell_t(\hat\pi_t)-\ell_t(\pi_{\bar \theta})]+
    \sum_{t=1}^T
    [\ell_t(\pi_{\bar\theta})-\ell_t(\pif^*)]\\
    &\leq
    H\log N_\varepsilon
    +
    \sum_{t=1}^T
    [\ell_t(\pi_{\bar\theta})-\ell_t(\pif^*)]
    \notag\\
    &\leq
    H\log N_\varepsilon+TH\varepsilon.
    \label{eq:covering-empirical-bound}
\end{align}
Define $C_t$ and $W_t$ the same as in the last section, that is,
\begin{align*}
    C_t
    &:=
    \sum_{\theta\in\Theta_\varepsilon}p_0(\theta)
    \exp\bigg(
        -\frac1H\sum_{s=1}^t\ell_s(\pi_\theta)
    \bigg)=\frac{1}{N_\varepsilon}\sum_{\theta\in\Theta_\varepsilon}
    \exp\bigg(
        -\frac1H\sum_{s=1}^t\ell_s(\pi_\theta)
    \bigg),\\
    W_t
    &:=
    C_t\exp\bigg(
        \frac1H\sum_{s=1}^t\ell_s(\pif^*)
    \bigg),
\end{align*}
with $C_0=W_0=1$. With a similar argument to
\eqref{eq:proof-Xt-central-condition}, we have 
\begin{align}
    \EE\bigg[
        \exp\bigg(
            \frac{\ell_t(\pif^*)-\ell_t(\hat\pi_t)}H
        \bigg)
        \,\bigg|\,\cF_t
    \bigg]
    \leq1.
    \label{eq:covering-central}
\end{align}
Using the normalizer-ratio identity
\eqref{eq:normalize-ratio}, we have
\begin{align*}
    \frac{W_t}{W_{t-1}}
    &=
    \frac{C_t}{C_{t-1}}
    \exp\bigl(\ell_t(\pif^*)/H\bigr)\\
    &=
    \sum_{\theta\in\Theta_\varepsilon}
    p_{t-1}(\theta)
    \exp\bigg(
        \frac{\ell_t(\pif^*)-\ell_t(\pi_\theta)}H
    \bigg),
\end{align*}
where the last equality holds due to the same argument as in \eqref{eq:normalize-ratio}.
We also have
\begin{align*}
    p_t(\bar\theta)
    &=
    \frac{
        N_\varepsilon^{-1}
        \exp\big(
            -H^{-1}\sum_{s=1}^t\ell_s(\pi_{\bar\theta})
        \big)
    }{C_t}\\
    &=
    \frac1{N_\varepsilon W_t}
    \exp\bigg(
        -\frac1H\sum_{s=1}^t
        [\ell_s(\pi_{\bar\theta})-\ell_s(\pif^*)]
    \bigg)\\
    &\geq
    \frac{e^{-t\varepsilon}}{N_\varepsilon W_t},
\end{align*}
where the inequality uses
\eqref{eq:covering-empirical-approximation}.
Furthermore, \eqref{eq:covering-target-approximation} implies
$\pi_{\bar\theta}(a|x,u)\geq
e^{-\varepsilon}\pif^*(a|x,u)$.
Therefore, the token-level mixture satisfies
\begin{align*}
    \frac{\hat\pi_t(a|x,u)}{\pif^*(a|x,u)}
    &\geq
    p_{t-1}(\bar\theta)
    \frac{\pi_{\bar\theta}(a|x,u)}{\pif^*(a|x,u)}\\
    &\geq
    \frac{e^{-t\varepsilon}}
         {N_\varepsilon W_{t-1}}.
\end{align*}
Using the same concentration argument as in the last section, we can consider the high-probability event
\begin{align*}
    \bigg\{\max_{0\leq s<t}W_s\leq\frac2\delta\bigg\}.
\end{align*}
On this event, we have
\begin{align*}
    \frac{\hat\pi_t(a|x,u)}{\pif^*(a|x,u)}
    \geq
    \frac{\delta e^{-T\varepsilon}}{2N_\varepsilon}
    =
    e^{-L},
\end{align*}
where we define $L:=\log({2N_\varepsilon}/{\delta})+T\varepsilon$.
Repeating the stopped-process argument from the proof of Theorem~\ref{thm:ar-fa-forward}, we obtain, with probability at least $1-\delta$,
\begin{align*}
    \operatorname{Regret}(T)
    &\leq
    2H\bigl(\log N_\varepsilon+T\varepsilon\bigr)+
    4H\bigg(
        1+\log N_\varepsilon+T\varepsilon
        +\log\frac2\delta
    \bigg)\log\frac2\delta.
\end{align*}
This proves Theorem~\ref{thm:fa-forward-covering}.
\end{proof}
\begin{example}
\label{cor:fa-forward-parametric}
Suppose $\Theta\subseteq\RR^d$ is a nonempty compact set contained
in a Euclidean ball of radius $R$, where $d\geq1$.
Assume that $\pi_{\theta^*}=\pif^*$ for some $\theta^*\in\Theta$.
Moreover, suppose that, for some $L_0>0$,
\begin{align*}
    \big|
        \log\pi_\theta(a|x,u)
        -
        \log\pi_{\theta'}(a|x,u)
    \big|
    \leq
    L_0\|\theta-\theta'\|_2
\end{align*}
for every $\theta,\theta'\in\Theta$, $h\in[H]$, and
$(x,u,a)\in\cX\times\cA^{h-1}\times\cA$.
Then, a standard Euclidean covering argument gives
\begin{align*}
    \mathcal N(\varepsilon,\log\Pi,\|\cdot\|_\infty)
    \leq
    \left(1+\frac{2RL_0}{\varepsilon}\right)^d.
\end{align*}
For every $T\geq2$, choose such a cover with
$\varepsilon=1/T$.
By Theorem~\ref{thm:fa-forward-covering}, the expected regret satisfies
\begin{align*}
    \EE\big[\operatorname{Regret}(T)\big]
    \leq
    H\left[d\log\bigl(1+2RL_0T\bigr)+1\right].
\end{align*}
Moreover, with probability at least $1-\delta$, we have
\begin{align*}
    \operatorname{Regret}(T) \le \tilde O(dH \log T).
\end{align*}
\end{example}

\section{Reverse KL with Function Approximation}
\label{sec:fa-reverse}
We now study on-policy distillation with function approximation. Fix any $h\in [H]$. For any prefix $x\in \cX$, $u \in \cA^{h-1}$, $a \in \cA$, following the notation in Theorem \ref{thm:obj-reverse-kl}, we define the target as 
\begin{align}
    f^*(h,x,u,a)
    :=
    \log g_h(a|x,u)
    =
    \sum_{i\in\cI}w_i(x)\log p_i(a|x,u).
    \notag
\end{align}
Specifically, we define
\begin{align*}
    \cZ
    :=
    \bigcup_{h=1}^{H}
    \bigl(
        \{h\}\times\cX\times\cA^{h-1}\times\cA
    \bigr),
\end{align*}
and write $z_{t,j,h}:=(h,x_{t,j},y_{t,j,<h},a_{t,j,h})$ for the query associated with the $h$-th token of rollout $j$
in round $t$.

Throughout this section, we retain Assumption~\ref{assump:token-bound}, with reference policy
$\piref$ and bound $B >0$. Moreover, let $\cF$ be a finite class of functions $f:\cZ\to\RR$.
We further impose the following realizability assumption.
\begin{assumption}
\label{assump:fa-rkl-realizability}
$f^* \in \cF$.
\end{assumption}
To measure how past observations constrain predictions at a new query, we introduce a new version of generalized Eluder dimension adapted to autoregressive trajectory batches.
The underlying uncertainty compares the disagreement between
two candidate score functions at the current query with their
cumulative squared disagreement on previously observed data.
Since the estimator is updated only after each round, this
comparison uses observations from completed rounds, without
incorporating feedback from the current batch. 
\begin{definition}
\label{def:batched-eluder}
Fix $\lambda>0$.
For a sequence of trajectory batches, write
\begin{align*}
    Z_t
    :=
    (z_{t,j,h})_{j\in[m],\,h\in[H]},
    \qquad
    Z_{<t}:=(Z_1,\ldots,Z_{t-1}).
\end{align*}
For any query $z\in\cZ$, define
\begin{align}
    D_{\cF}^{2}(z;Z_{<t})
    :=
    \sup_{f_1,f_2\in\cF}
    \frac{
        \bigl(f_1(z)-f_2(z)\bigr)^2
    }{
        \lambda+
        \displaystyle\sum_{s=1}^{t-1}
        \frac1{mH}
        \sum_{j=1}^{m}\sum_{h=1}^{H}
        \bigl(
            f_1(z_{s,j,h})-f_2(z_{s,j,h})
        \bigr)^2
    }.
    \label{eq:ar-generalized-uncertainty}
\end{align}
The batched generalized Eluder dimension is
\begin{align}
    \dim_T(\cF;\lambda)
    :=
    \sup_{Z_{1:T}}
    \sum_{t=1}^{T}
    \frac1{mH}
    \sum_{j=1}^{m}\sum_{h=1}^{H}
    \min\left\{
        1,\,
        D_{\cF}^{2}(z_{t,j,h};Z_{<t})
    \right\},
    \label{eq:ar-generalized-eluder}
\end{align}
where the supremum is over all sequences of $T$ trajectory
batches, each containing $m$ rollouts of horizon $H$.
\end{definition}
\subsection{Algorithm Design}
Recall that, for every visited query
$z_{t,j,h}=(h,x,u,a)$,
\begin{align}
    \EE\big[
        l_{t,j,h}
        \,\big|\,
        \cF_t,\,
        z_{t,j,h}=(h,x,u,a)
    \big]
    &=
    \sum_{i\in\cI}
    w_i(x)\log p_i(a|x,u)
    \notag\\
    &=
    \log g_h(a|x,u)
    =
    f^*(h,x,u,a),
    \label{eq:fa-rkl-conditional-mean}
\end{align}
where $\cF_t$ denotes the history before round $t$.
This identity motivates estimating $f^*$ by least squares. For each $f\in\cF$, define
\begin{align}
    \ell_t(f)
    &:=
    \frac{1}{mH}
    \sum_{j=1}^{m}\sum_{h=1}^{H}
    \bigl(f(z_{t,j,h})-l_{t,j,h}\bigr)^2.
    \label{eq:fa-rkl-regression-loss}
\end{align}
Moreover, we define
\begin{align}
    L_t(f)
    &:=
    \sum_{s=1}^{t}\ell_s(f),
    \qquad L_0(f):=0.
    \notag
\end{align}
At the beginning of round $t$, the learner computes
\begin{align}
    \bar f_{t-1}
    \in
    \argmin_{f\in\cF}L_{t-1}(f).
    \label{eq:fa-rkl-erm}
\end{align}
Using the principle of optimism for online reinforcement learning, we therefore define a bonus using the batched generalized Eluder dimension, i.e., for any $z\in\cZ$, the bonus function is defined as
\begin{align}
    b_{t-1}(z)
    :=
    \sqrt{
        (4\beta+\lambda)
        D_{\cF}^{2}(z;Z_{<t})
    },
    \label{eq:fa-rkl-bonus}
\end{align}
where $\lambda>0$ and $\beta>0$ are parameters to be specified later. Given the reference policy $\piref$ and the bound $B$
from Assumption~\ref{assump:token-bound}, as it implies
\begin{align*}
    \big|
        f^*(h,x,u,a)-\log\piref(a|x,u)
    \big|
    \leq B,
\end{align*}
we define the optimistic score by
\begin{align}
    \hat f_{t-1}(h,x,u,a)
    :=
    \min\Big\{
        \bar f_{t-1}(h,x,u,a)
        +b_{t-1}(h,x,u,a),\,
        \log\piref(a|x,u)+B
    \Big\}.
    \label{eq:fa-rkl-optimistic-score}
\end{align}
The following lemma establishes pointwise optimism and
bounds the error of the optimistic score.
\begin{lemma}
\label{lem:fa-rkl-optimism}
Under Assumptions~\ref{assump:token-bound}
and~\ref{assump:fa-rkl-realizability}, fix $\delta\in(0,1)$
and choose
\begin{align*}
    \beta
    =
    32B^2\log\frac{|\cF|}{\delta},
    \qquad \lambda>0.
\end{align*}
Then, with probability at least $1-\delta$, simultaneously for
every $t\in[T]$ and $z\in\cZ$,
\begin{align}
    \big|\bar f_{t-1}(z)-f^*(z)\big|
    &\leq b_{t-1}(z).
    \notag
\end{align}
Furthermore, the following inequality holds
\begin{align}
    0
    \leq \hat f_{t-1}(z)-f^*(z)
    &\leq
    \min\bigl\{2b_{t-1}(z),2B\bigr\}.
    \notag
\end{align}
\end{lemma}
To construct the output policy, we apply the backward
propagation as in the tabular setting to the estimated scores $\hat f_t$. More specifically, set $\hat V_{t-1,H+1}(x,y)=1$ and recursively define
\begin{align}
    \hat V_{t-1,h}(x,u)
    &:=
    \sum_{a\in\cA}
    \exp\bigl(\hat f_{t-1}(h,x,u,a)\bigr)
    \hat V_{t-1,h+1}(x,(u,a)),
    \label{eq:fa-rkl-backward}
\end{align}
for $h=H,\ldots,1$. We then output
\begin{align}
    \hat\pi_t(a|x,u)
    :=
    \frac{
        \exp\bigl(\hat f_{t-1}(h,x,u,a)\bigr)
        \hat V_{t-1,h+1}(x,(u,a))
    }{
        \hat V_{t-1,h}(x,u)
    }.
    \label{eq:fa-rkl-policy}
\end{align}
This procedure is described in Algorithm~\ref{algo:fa-reverse}.
\begin{algorithm}[H]
\caption{Optimistic Reverse KL with Function Approximation}
\label{algo:fa-reverse}
\begin{algorithmic}[1]
    \STATE \textbf{Input:}
    Function class $\cF$, reference policy $\piref$,
    bound $B$, horizon $H$, batch size $m$,
    number of rounds $T$, and parameters $\lambda,\beta>0$.
    \STATE Initialize $L_0(f)=0$ for every $f\in\cF$.
    \FOR{$t=1,\ldots,T$}
        \STATE Compute the regression estimate $\bar f_{t-1}$
        using \eqref{eq:fa-rkl-erm}.
        \STATE Compute $b_{t-1}$ and $\hat f_{t-1}$
        using \eqref{eq:fa-rkl-bonus}
        and \eqref{eq:fa-rkl-optimistic-score}.
        \STATE Set $\hat V_{t-1,H+1}(x,y)=1$
        and compute the backward recursion
        \eqref{eq:fa-rkl-backward}.
        \STATE Construct $\hat\pi_t$ using
        \eqref{eq:fa-rkl-policy}.
        \STATE Receive contexts $\{x_{t,j}\}_{j=1}^{m}$
        according to the on-policy protocol and generate
        $y_{t,j}\sim\hat\pi_t(\cdot|x_{t,j})$.
        \STATE Observe teacher feedback
        $\{l_{t,j,h}\}_{j\in[m],\,h\in[H]}$.
        \STATE Compute $\ell_t(f)$ using
        \eqref{eq:fa-rkl-regression-loss} and update
        $L_t(f)=L_{t-1}(f)+\ell_t(f)$
        for every $f\in\cF$.
    \ENDFOR
    \STATE \textbf{Output:} $\{\hat\pi_t\}_{t=1}^{T}$.
\end{algorithmic}
\end{algorithm}
\subsection{Theoretical Guarantee}
Given Lemma \ref{lem:fa-rkl-optimism}, we have the following guarantee on the regret of Algorithm \ref{algo:fa-reverse}.
\begin{theorem}
\label{thm:fa-rkl-regret}
Under Assumptions~\ref{assump:token-bound}
and~\ref{assump:fa-rkl-realizability}, fix $\delta\in(0,1)$
and choose $\lambda > 0$ and
\begin{align*}
    \beta
    \ge
    32B^2\log\frac{2|\cF|}{\delta}
\end{align*}
Then, with probability at least $1-\delta$,
Algorithm~\ref{algo:fa-reverse} satisfies
\begin{align}
    \operatorname{Regret}(T)
    &:=
    \sum_{t=1}^T
    \EE_{x\sim\bar\rho}
    \KL\bigl(
        \hat\pi_t(\cdot|x)
        \,\big\|\,
        \pir^*(\cdot|x)
    \bigr)
    \notag\\
    &\leq
    4H^2(4\beta+\lambda)
    \dim_T(\cF;\lambda)
    +
    16H^2B^2\log\frac4\delta.
    \notag
\end{align}
In particular, if we choose $\lambda=B^2$ and $\beta=32B^2\log(2|\cF|/\delta)$, we have
\begin{align*}
    \operatorname{Regret}(T)
    \leq
    O\bigg(
        B^2H^2
            \dim_T(\cF;B^2)
            \log\frac{2|\cF|}{\delta}
    \bigg).
\end{align*}
\end{theorem}
\begin{proof}
Let $\cF_t$ denote the history before round $t$.
In particular, $\bar f_{t-1}$, $b_{t-1}$, and
$\hat\pi_t$ are $\cF_t$-measurable.
Applying Lemma~\ref{lem:fa-rkl-optimism}, we have with probability at least $1-\delta/2$, the following inequality holds  simultaneously for every $t\in[T]$ and $z\in\cZ$
\begin{align}
    0
    \leq
    \hat f_{t-1}(z)-f^*(z)
    \leq
    \min\{2b_{t-1}(z),2B\},
    \label{eq:fa-rkl-proof-optimism}
\end{align}
when $\beta \ge
    32B^2\log({2|\cF|}/{\delta})$. Let $\cE$ be this high-probability event. For every $t$ and $z$, define
\begin{align}
    \psi_t(z)
    &:=
    \min\{4b_{t-1}(z)^2,4B^2\}.
    \label{eq:fa-rkl-proof-error-proxy}
\end{align}
Then we have $0\leq\psi_t(z)\leq4B^2$. On $\cE$, \eqref{eq:fa-rkl-proof-optimism} further gives
\begin{align}
    \big[
        \hat f_{t-1}(z)-f^*(z)
    \big]^2
    \leq
    \psi_t(z).
    \label{eq:fa-rkl-proof-error-domination}
\end{align}
Fix a context $x$ and a response $y$. Write
\begin{align*}
    v_t(x,y)
    :=
    \sum_{h=1}^H
    \Big[
        \hat f_{t-1}(h,x,y_{<h},a_h)
        -
        f^*(h,x,y_{<h},a_h)
    \Big].
\end{align*}
On $\cE$, every summand is nonnegative, so
$v_t(x,y)\geq0$. Recall that for $y=(a_1,\ldots,a_H)$, the backward propagation gives
\begin{align*}
    \hat\pi_t(y|x)
    &=
    \prod_{h=1}^H
    \frac{
        e^{\hat f_{t-1}(h,x,y_{<h},a_h)}
        \hat V_{t-1,h+1}(x,y_{\leq h})
    }{
        \hat V_{t-1,h}(x,y_{<h})
    }\\
    &=
    \frac{
        \exp\big(
            \sum_{h=1}^H
            \hat f_{t-1}(h,x,y_{<h},a_h)
        \big)
    }{
        \hat V_{t-1,1}(x)
    },
\end{align*}
where we used $\hat V_{t-1,H+1}(x,y)=1$.
Similarly, the reverse target satisfies
\begin{align*}
    \pir^*(y|x)
    =
    \frac{
        \exp\big(
            \sum_{h=1}^H
            f^*(h,x,y_{<h},a_h)
        \big)
    }{
        V_1(x)
    },
\end{align*}
where $V_1(x) = \sum_y \exp\big(\sum_{h=1}^H f^*(h,x,y_{<h},a_h)\big)$. Taking the logarithm of their ratio therefore yields
\begin{align}
\notag
    \log\frac{\hat\pi_t(y|x)}{\pir^*(y|x)}
    &=
    \sum_{h=1}^H
    \big[
        \hat f_{t-1}(h,x,y_{<h},a_h)
        -
        f^*(h,x,y_{<h},a_h)
    \big]\\\notag
    &\qquad+
    \log V_1(x)-\log\hat V_{t-1,1}(x)\\\label{eq:fa-rkl-log}
    &=
    v_t(x,y)
    +
    \log\frac{V_1(x)}{\hat V_{t-1,1}(x)}.
\end{align}
To express the ratio of normalizing constants as an
expectation, expand the expectation over complete responses:
\begin{align}
\notag&\EE_{y\sim\hat\pi_t(\cdot|x)}
    \big[e^{-v_t(x,y)}\big]\\\notag
    &=
    \sum_{y\in\cY}
    \hat\pi_t(y|x)e^{-v_t(x,y)}\\\notag
    &=
    \sum_{y\in\cY}
    \frac{
        \exp\big(
            \sum_{h=1}^H
            \hat f_{t-1}(h,x,y_{<h},a_h)
        \big)
    }{
        \hat V_{t-1,1}(x)
    }
    \exp\bigg(
        \sum_{h=1}^H
        \big[
            f^*(h,x,y_{<h},a_h)
            -
            \hat f_{t-1}(h,x,y_{<h},a_h)
        \big]
    \bigg)\\\notag
    &=
    \frac1{\hat V_{t-1,1}(x)}
    \sum_{y\in\cY}
    \exp\bigg(
        \sum_{h=1}^H f^*(h,x,y_{<h},a_h)
    \bigg)\\\label{eq:fa-rkl-normal-const}
    &=
    \frac{V_1(x)}{\hat V_{t-1,1}(x)}.
\end{align}
The third equality follows by cancellation of the
$\hat f_{t-1}$ terms in the exponent. The last equality uses the definition of $V_1(x)$.

Taking expectation over $\hat\pi_t(\cdot|x)$ in \eqref{eq:fa-rkl-log} and considering \eqref{eq:fa-rkl-normal-const}, we have
\begin{align}
    \KL\bigl(
        \hat\pi_t(\cdot|x)
        \,\big\|\,
        \pir^*(\cdot|x)
    \bigr)
    &=
    \EE_{\hat\pi_t}[v_t(x,y)]
    +
    \log\EE_{\hat\pi_t}[e^{-v_t(x,y)}]
    \notag\\
    &\leq
    \EE_{\hat\pi_t}[v_t(x,y)]
    +
    \EE_{\hat\pi_t}[e^{-v_t(x,y)}]-1
    \notag\\
    &\leq
    \frac12
    \EE_{\hat\pi_t}[v_t(x,y)^2].
    \label{eq:fa-rkl-proof-kl-square}
\end{align}
Here the first inequality uses $\log u\leq u-1$ for
$u>0$, and the second uses
$e^{-v}\leq1-v+v^2/2$ for $v\geq0$.
All expectations in this display are conditional on $x$.

By the Cauchy--Schwarz inequality and
\eqref{eq:fa-rkl-proof-error-domination},
\begin{align*}
    v_t(x,y)^2
    &\leq
    H\sum_{h=1}^H
    \big[
        \hat f_{t-1}(h,x,y_{<h},a_h)
        -
        f^*(h,x,y_{<h},a_h)
    \big]^2\\
    &\leq
    H\sum_{h=1}^H
    \psi_t(h,x,y_{<h},a_h).
\end{align*}
Consequently, on $\cE$,
\begin{align}
    \operatorname{Regret}(T)
    &\leq
    \frac H2
    \sum_{t=1}^T
    \EE_{x\sim\bar\rho,\,
         y\sim\hat\pi_t(\cdot|x)}
    \bigg[
        \sum_{h=1}^H
        \psi_t(h,x,y_{<h},a_h)
    \bigg].
    \label{eq:fa-rkl-proof-predictable}
\end{align}
Define 
\begin{align*}
    X_t
    :=
    \frac1{mH}
    \sum_{j=1}^m\sum_{h=1}^H
    \psi_t(z_{t,j,h}).
\end{align*}
Then $X_t$ is $\cF_{t+1}$-measurable and satisfies
$0\leq X_t\leq4B^2$ almost surely.
Under the on-policy protocol, each rollout has conditional
marginal distribution
\begin{align*}
    \PP\bigl(
        x_{t,j}=x,\,
        y_{t,j}=y
        \mid\cF_t
    \bigr)
    &=
    \frac1I\sum_{i\in\cI}
    \rho_i(x)\hat\pi_t(y|x)\\
    &=
    \bar\rho(x)\hat\pi_t(y|x).
\end{align*}
Since $\psi_t$ is determined before round $t$, this implies
\begin{align}
    \EE[X_t\mid\cF_t]
    =
    \frac1H
    \EE_{x\sim\bar\rho,\,
         y\sim\hat\pi_t(\cdot|x)}
    \bigg[
        \sum_{h=1}^H
        \psi_t(h,x,y_{<h},a_h)
    \bigg].
    \label{eq:fa-rkl-proof-proxy-mean}
\end{align}
This identity uses only the marginal law of each rollout,
not independence within the batch.
Combining \eqref{eq:fa-rkl-proof-predictable} and
\eqref{eq:fa-rkl-proof-proxy-mean}, we obtain, on $\cE$,
\begin{align}
    \operatorname{Regret}(T)
    \leq
    \frac{H^2}{2}
    \sum_{t=1}^T\EE[X_t\mid\cF_t].
    \label{eq:fa-rkl-proof-predictable-proxy}
\end{align}
Apply Lemma~\ref{lem:foster} to $\{X_t\}_{t=1}^T$,
with range bound $4B^2$ and failure probability $\delta/2$.
With probability at least $1-\delta/2$,
\begin{align}
    \sum_{t=1}^T\EE[X_t\mid\cF_t]
    \leq
    2\sum_{t=1}^T X_t
    +
    32B^2\log\frac4\delta.
    \label{eq:fa-rkl-proof-concentration}
\end{align}
Taking a union bound, we conclude that, with
probability at least $1-\delta$,
\begin{align}
    \operatorname{Regret}(T)
    \leq
    H^2\sum_{t=1}^T X_t
    +
    16H^2B^2\log\frac4\delta.
    \label{eq:fa-rkl-proof-empirical}
\end{align}
Finally, by the definition of the bonus function, we have
\begin{align*}
    \psi_t(z)
    =
    4\min\big\{
        (4\beta+\lambda)
        D_{\cF}^2(z;Z_{<t}),
        B^2
    \big\}.
\end{align*}
Our choice of $\beta$ ensures $4\beta+\lambda\geq B^2$.
Therefore,
\begin{align}
    \psi_t(z)
    \leq
    4(4\beta+\lambda)
    \min\big\{1,D_{\cF}^2(z;Z_{<t})\big\}.
    \label{eq:fa-rkl-proof-proxy-dimension}
\end{align}
Indeed, when $D_{\cF}^2(z;Z_{<t})\leq1$, we use the
first term in the minimum defining $\psi_t$.
Otherwise, we use $\psi_t(z)\leq4B^2\leq4(4\beta+\lambda)$.

Summing \eqref{eq:fa-rkl-proof-proxy-dimension}
over the observed queries gives
\begin{align*}
    \sum_{t=1}^T X_t
    &=
    \sum_{t=1}^T
    \frac1{mH}\sum_{j=1}^m\sum_{h=1}^H
    \psi_t(z_{t,j,h})\\
    &\leq
    4(4\beta+\lambda)
    \sum_{t=1}^T
    \frac1{mH}\sum_{j=1}^m\sum_{h=1}^H
    \min\big\{
        1,
        D_{\cF}^2(z_{t,j,h};Z_{<t})
    \big\}\\
    &\leq
    4(4\beta+\lambda)
    \dim_T(\cF;\lambda).
\end{align*}
The last inequality follows from the definition of the generalized Eluder dimension, since the observed batches form a valid sequence of autoregressive trajectories. Substituting this inequality into~\eqref{eq:fa-rkl-proof-empirical} completes the proof of Theorem \ref{thm:fa-rkl-regret}.
\end{proof}
\subsection{Proof of Lemma \ref{lem:fa-rkl-optimism}}
\begin{proof}
Let $\cF_t$ denote the history before round $t$.
Fix $f\in\cF$, and define
\begin{align*}
    \Delta_{t,j,h}(f)
    &:=
    f(z_{t,j,h})-f^*(z_{t,j,h}),\\
    \xi_{t,j,h}
    &:=
    l_{t,j,h}-f^*(z_{t,j,h}).
\end{align*}
The conditional-mean identity gives
\begin{align}
    \EE[\xi_{t,j,h}\mid\cF_t,z_{t,j,h}]
    =0.
    \label{eq:fa-rkl-lemma-centered-noise}
\end{align}
Moreover, conditioned on
$z_{t,j,h}=z=(h,x,u,a)$,
Assumption~\ref{assump:token-bound} implies
\begin{align*}
    \xi_{t,j,h}
    \in
    \big[
        \log\piref(a|x,u)-B-f^*(z),\,
        \log\piref(a|x,u)+B-f^*(z)
    \big].
\end{align*}
We use the abbreviation $\Delta:=\Delta_{t,j,h}(f),    \xi:=\xi_{t,j,h}$. Then, applying Assumption \ref{assump:token-bound} to $f^*$ and observing that it is an average over all the teachers, we have
\begin{align*}
    \EE[\xi\mid\cF_t,z_{t,j,h}]=0,
    \qquad
    |\xi|\leq2B.
\end{align*}
Set $\eta=1/(32B^2)$. 
We claim that
\begin{align}
    \EE\big[
        e^{\eta(2\Delta\xi-\Delta^2/2)}
        \mid\cF_t,z_{t,j,h}
    \big]
    \leq1.
    \label{eq:fa-rkl-lemma-single-query}
\end{align}
In fact, if $|\Delta|\geq8B$, then
\begin{align*}
    2\Delta\xi-\frac{\Delta^2}{2}
    \leq
    4B|\Delta|-\frac{\Delta^2}{2}
    \leq0.
\end{align*}
Thus, \eqref{eq:fa-rkl-lemma-single-query} holds in this case. Otherwise, if $|\Delta|<8B$, then
\begin{align*}
    |2\eta\Delta\xi|
    \leq
    2\eta\cdot8B\cdot2B
    =1.
\end{align*}
Using $e^u\leq1+u+u^2$ for $|u|\leq1$,
we obtain
\begin{align*}
    \EE[e^{2\eta\Delta\xi}\mid\cF_t,z_{t,j,h}]
    &\leq
    1+
    2\eta\Delta
    \EE[\xi\mid\cF_t,z_{t,j,h}]
    +
    4\eta^2\Delta^2
    \EE[\xi^2\mid\cF_t,z_{t,j,h}]\\
    &\leq
    1+16\eta^2B^2\Delta^2\\
    &\leq
    \exp(16\eta^2B^2\Delta^2)\\
    &=
    \exp(\eta\Delta^2/2).
\end{align*}
The second inequality uses the zero conditional mean
and $\xi^2\leq4B^2$.
The third uses $1+v\leq e^v$, and the last equality
follows from our choice of $\eta$.
Multiplying both sides by $e^{-\eta\Delta^2/2}$
proves \eqref{eq:fa-rkl-lemma-single-query}.

To handle the dependence within a round, apply Jensen's
inequality to the exponential function:
\begin{align*}
    &\exp\bigg(
        \frac{\eta}{mH}
        \sum_{j=1}^m\sum_{h=1}^H
        \bigg[
            2\Delta_{t,j,h}(f)\xi_{t,j,h}
            -
            \frac12\Delta_{t,j,h}(f)^2
        \bigg]
    \bigg)\\
    &\qquad\leq
    \frac1{mH}
    \sum_{j=1}^m\sum_{h=1}^H
    \exp\bigg(
        \eta\bigg[
            2\Delta_{t,j,h}(f)\xi_{t,j,h}
            -
            \frac12\Delta_{t,j,h}(f)^2
        \bigg]
    \bigg).
\end{align*}
Taking conditional expectations given $\cF_t$ and
applying \eqref{eq:fa-rkl-lemma-single-query}
to each summand gives
\begin{align}
    &\EE\bigg[
        \exp\bigg(
            \frac{\eta}{mH}
            \sum_{j=1}^m\sum_{h=1}^H
            \bigg[
                2\Delta_{t,j,h}(f)\xi_{t,j,h}
                -
                \frac12\Delta_{t,j,h}(f)^2
            \bigg]
        \bigg)
        \,\bigg|\,
        \cF_t
    \bigg]
    \leq1.
    \label{eq:fa-rkl-lemma-round-exponential}
\end{align}
We do not condition on the entire batch or assume
independence among its observations.

Define $M_0(f)=1$ and
\begin{align*}
    M_t(f)
    :=
    \exp\bigg(
        \eta
        \sum_{s=1}^t
        \frac1{mH}
        \sum_{j=1}^m\sum_{h=1}^H
        \bigg[
            2\Delta_{s,j,h}(f)\xi_{s,j,h}
            -
            \frac12\Delta_{s,j,h}(f)^2
        \bigg]
    \bigg).
\end{align*}
Equation~\eqref{eq:fa-rkl-lemma-round-exponential}
implies
\begin{align*}
    \EE[M_t(f)\mid\cF_t]
    \leq M_{t-1}(f).
\end{align*}
Thus, $\{M_t(f)\}_{t=0}^T$ is a nonnegative supermartingale. By Ville's inequality (Lemma \ref{lem:Ville}),
\begin{align*}
    \PP\bigg(
        \max_{0\leq t\leq T}M_t(f)
        >
        \frac{|\cF|}{\delta}
    \bigg)
    \leq
    \frac{\delta}{|\cF|}.
\end{align*}
Taking a union bound over $f\in\cF$, we obtain an event
$\cE$ with $\PP(\cE)\geq1-\delta$ on which,
simultaneously for every $f\in\cF$ and $t\in\{0,\ldots,T\}$,
\begin{align}
    \sum_{s=1}^t
    \frac1{mH}
    \sum_{j=1}^m\sum_{h=1}^H
    \bigg[
        2\Delta_{s,j,h}(f)\xi_{s,j,h}
        -
        \frac12\Delta_{s,j,h}(f)^2
    \bigg]
    &\leq
    \frac1\eta\log\frac{|\cF|}{\delta}
    \notag\\\notag
    &= 32B^2 \log\frac{|\cF|}{\delta}\\
    &\leq\beta.
    \label{eq:fa-rkl-lemma-uniform-noise}
\end{align}
Expanding the difference of squared losses gives
\begin{align*}
    &(f(z_{s,j,h})-l_{s,j,h})^2
    -
    (f^*(z_{s,j,h})-l_{s,j,h})^2=
    \Delta_{s,j,h}(f)^2
    -
    2\Delta_{s,j,h}(f)\xi_{s,j,h}.
\end{align*}
Therefore, on $\cE$,
\begin{align}
    L_t(f)-L_t(f^*)
    &=
    \frac12
    \sum_{s=1}^t\frac1{mH}
    \sum_{j=1}^m\sum_{h=1}^H
    \Delta_{s,j,h}(f)^2
    \notag\\
    &\quad-
    \sum_{s=1}^t\frac1{mH}
    \sum_{j=1}^m\sum_{h=1}^H
    \bigg[
        2\Delta_{s,j,h}(f)\xi_{s,j,h}
        -
        \frac12\Delta_{s,j,h}(f)^2
    \bigg]
    \notag\\
    &\geq
    \frac12
    \sum_{s=1}^t\frac1{mH}
    \sum_{j=1}^m\sum_{h=1}^H
    \Delta_{s,j,h}(f)^2
    -
    \beta,
    \label{eq:fa-rkl-lemma-regression-bound}
\end{align}
where the last inequality holds due to \eqref{eq:fa-rkl-lemma-uniform-noise}. It holds for any $f \in \cF$, so it also applies to $\bar f_{t-1}$.
Moreover, by \eqref{eq:fa-rkl-erm} and the realizability assumption, we have
\begin{align*}
    L_{t-1}(\bar f_{t-1})
    \leq
    L_{t-1}(f^*).
\end{align*}
Applying \eqref{eq:fa-rkl-lemma-regression-bound}
at time $t-1$ consequently gives
\begin{align}
    \sum_{s<t}\frac1{mH}
    \sum_{j=1}^m\sum_{h=1}^H
    \big[
        \bar f_{t-1}(z_{s,j,h})
        -
        f^*(z_{s,j,h})
    \big]^2
    \leq2\beta.
    \label{eq:fa-rkl-lemma-erm-error}
\end{align}
Since both $\bar f_{t-1}$ and $f^*$ belong to $\cF$,
the definition of $D_{\cF}^2$ implies, for every $z\in\cZ$,
\begin{align*}
    &\big|\bar f_{t-1}(z)-f^*(z)\big|^2\\
    &\quad\leq
    D_{\cF}^2(z;Z_{<t})
    \bigg(
        \lambda
        +
        \sum_{s<t}\frac1{mH}
        \sum_{j=1}^m\sum_{h=1}^H
        \big[
            \bar f_{t-1}(z_{s,j,h})
            -
            f^*(z_{s,j,h})
        \big]^2
    \bigg)\\
    &\quad\leq
    (\lambda+2\beta)D_{\cF}^2(z;Z_{<t})\\
    &\quad\leq
    b_{t-1}(z)^2.
\end{align*}
Taking square roots proves the pointwise confidence bound.

Finally, for $z=(h,x,u,a)$, Assumption~\ref{assump:token-bound}
and the definition of $f^*$ give
\begin{align*}
    \log\piref(a|x,u)-B
    \leq f^*(z)
    \leq \log\piref(a|x,u)+B.
\end{align*}
On $\cE$, the confidence bound also gives
\begin{align*}
    \bar f_{t-1}(z)+b_{t-1}(z)
    \geq f^*(z).
\end{align*}
Therefore, $\hat f_{t-1}(z) \ge f^*(z)$, which proves optimism. Moreover,
\begin{align*}
    \hat f_{t-1}(z)-f^*(z)
    &=
    \min\big\{
        \bar f_{t-1}(z)-f^*(z)+b_{t-1}(z),\,
        \log\piref(a|x,u)+B-f^*(z)
    \big\}\\
    &\leq
    \min\{2b_{t-1}(z),2B\}.
\end{align*}
This completes the proof of Lemma \ref{lem:fa-rkl-optimism}.
\end{proof}
\section{Auxiliary Lemmas}
\label{sec:auxiliary-lemmas}
\subsection{Uniform Bernstein Concentration in \citet{bakhtiari2026eluder}}
First, we list several key results from \citet{bakhtiari2026eluder} that are used in our analysis. To start with, we present several definitions.
\begin{definition}[CGF-like, Definition 4 \citealt{bakhtiari2026eluder}]
A twice differentiable function $\psi:[0,\lambda_{\max})\to\RR_+$ is called CGF-like if it is convex, non-negative, and satisfies $\psi(0)=\psi'(0)=0$.
\end{definition}
\begin{definition}[sub-$\psi$ process, Definition 5 \citealt{bakhtiari2026eluder}]
    Let $\{\cF_t\}_{t\in [T]}$ be a filtration, $\psi: [0,\lambda_{\max}) \to \RR_+$ be a CGF-like function. Let $\{S_t\}_{t\in[T]}$ be a real-valued stochastic process adapted to $\{\cF_t\}$, $\{V_t\}_{t\in[T]}$ be another process adapted to $\{\cF_t\}$ taking values in $\RR_+$. We say that $\{(S_t,V_t)\}_{t\in[T]}$ is a sub-$\psi$ process if for all $t \in [T]$ and $\lambda \in [0,\lambda_{\max})$, there exists an $\cF$-adapted supermartingale $\{L_t(\lambda)\}_{t\in[T]}$ such that 
    \begin{align*}
        M_t(\lambda) := \exp\Big(\lambda S_t - \sum_{s=1}^t \psi(\lambda)V_s\Big) \le L_t(\lambda), \quad \text{a.s.}
    \end{align*}
\end{definition}
\begin{definition}[Definition 6 \citealt{bakhtiari2026eluder}]
    A random process $\{(S_t,V_t)\}_{t \in [T]}$ is sub-gamma with parameter $v > 0$ if it is sub-$\psi$ with the CGF-like function $\psi(\lambda) = \frac{\lambda^2 }{2(1-v\lambda)}$ for some $v > 0$.
    
\end{definition}
\begin{lemma}[Proposition~11 \citealt{bakhtiari2026eluder}]\label{lem:bakhtiari11}
    Let $\{\cF_t\}_{t\in [T]}$ be a filtration and let $(X_t)_{t \in [T]}$ be a real-valued stochastic process adapted to $\cF$. Assume that there exists a constant $b > 0$ such that for all $t \in [T]$,
    \begin{align*}
        X_t \le \EE[X_t | \cF_{t-1}] + b \ a.s.
    \end{align*}
    Then, for 
    \begin{align*}
        S_t = \sum_{s=1}^t (X_s - \EE[X_s | \cF_{s-1}]), \quad V_t = \sum_{s=1}^t \Var(X_s | \cF_{s-1}),
    \end{align*}
    the process $\{(S_t,V_t)\}_{t \in [T]}$ is sub-gamma with parameter $b/3$.
\end{lemma}
\begin{lemma}
[Theorem~10 \citealt{bakhtiari2026eluder}]\label{lem:bakhtiari10}
    For any sub-gamma process $\{(S_t,V_t)\}_{t \in [T]}$ with parameter $v > 0$, and any $\rho > 0$, $\delta \in (0,1)$, with probability at least $1-\delta$, we have for all $t \in [T]$,
    \begin{align*}
        S_t \le 4\sqrt{V_t\log(H_t/\delta)} + 11 (v+\rho) \log(H_t/\delta),
    \end{align*}
    where $H_t = \log(1+V_t/\rho^2) + e$.
\end{lemma}
\subsection{Other Auxiliary Lemmas}
\begin{lemma}[Lemma B.3, \citealt{foster2024behavior}]\label{lem:foster}
    Let $\{\bX_t\}_{t \in [T]}$ be a sequence of non-negative random variables adapted to a filtration $\{\cF_t\}_{t \in [T]}$. Assume that $0 \le \bX_t \le R$ almost surely for all $t \in [T]$ and some constant $R > 0$. Then, for any $\delta \in (0,1)$ and all $T' \le T$, with probability at least $1-\delta$, we have
    \begin{align*}
        \sum_{t=1}^{T'} \EE[\bX_t | \cF_{t-1}] &\le 2\sum_{t=1}^{T'} \bX_t + 8R\log(2/\delta),\\
        \sum_{t=1}^{T'} \bX_t &\le \frac{3}{2}\sum_{t=1}^{T'} \EE[\bX_t | \cF_{t-1}] + 4R\log(2/\delta).
    \end{align*}
    
\end{lemma}
\begin{lemma}
\label{lem:KL}
    Let $\mu_f(\cdot) \propto \exp(f(\cdot))$, $\mu_g (\cdot)\propto \exp(g(\cdot))$ be two probability distributions. Then,
    \begin{align*}
        \KL(\mu_f\|\mu_g) = \EE_{\mu_f} [f-g] + \log \big[\EE_{\mu_f}[\exp(g-f)]\big].
    \end{align*}
\end{lemma}
\begin{proof}
    Let $Z_f = \sum_x \exp(f(x))$, $Z_g = \sum_x \exp(g(x))$ be the normalization constants. Then,
    \begin{align*}
        \mu_f(x) = \frac{\exp(f(x))}{Z_f}, \ \mu_g(x) = \frac{\exp(g(x))}{Z_g}.
    \end{align*}
    The KL-divergence is equal to 
    \begin{align*}
        \KL(\mu_f\|\mu_g) &= \sum_x \mu_f(x) \log \frac{\mu_f(x)}{\mu_g(x)}\\
        &= \sum_x \mu_f(x) \big[f(x)-g(x) + \log(Z_g/Z_f)\big]\\
        &= \EE_{\mu_f}[f-g] + \log(Z_g/Z_f).
    \end{align*}
    Moreover,
    \begin{align*}
        Z_g/Z_f &= \frac{\sum_x \exp(g(x))}{\sum_x \exp(f(x))}\\
        &= \sum_x \frac{\exp(f(x))}{\sum_y \exp(f(y))} \cdot \exp(g(x)-f(x))\\
        &= \EE_{\mu_f}[\exp(g-f)].
    \end{align*}
    Thus, we complete the proof.
\end{proof}
\begin{lemma}[Stirling's inequality]
\label{lemma:stirling}
Let $\Gamma(\cdot)$ be the gamma function. For any $s>0$, the following inequality holds
\begin{align}
    \Big(s-\frac12\Big)\log s-s+\frac12\log(2\pi)
    &\leq
    \log\Gamma(s)
    \notag\\
    &\leq
    \Big(s-\frac12\Big)\log s-s+\frac12\log(2\pi)
    +\frac{1}{12s}.
\notag
\end{align}
\end{lemma}
\begin{lemma}[Ville's inequality]
\label{lem:Ville}
Let $\{X_t\}_{t=0}^T$ be a non-negative supermartingale. Then, for any $v > 0$, 
\begin{align*}
    \PP\Big[\max_{0\le t \le T} X_t \ge v\Big] \le \frac{\EE[X_0]}{v}.
\end{align*}
\end{lemma}
\begin{proof}
    Please refer to \citet{durrett2019probability}.
\end{proof}
% \begin{lemma}\label{lemma:stirling}
% Let $\Gamma(x)$ denote the Gamma function. Then $\log\Gamma(x)$ satisfies
% \begin{align*}
% S(x)<\log\Gamma(x)<S(x)+\frac{1}{12x},
% \end{align*}
% where
% \begin{align*}
% S(x)=(x-1/2)\log x-x+\frac12\log(2\pi).
% \end{align*}
% \end{lemma}
\bibliography{reference}

@article{bakhtiari2026eluder,
  title={Eluder dimension: localise it!},
  author={Bakhtiari, Alireza and Ayoub, Alex and Robertson, Samuel and Janz, David and Szepesv{\'a}ri, Csaba},
  journal={arXiv preprint arXiv:2601.09825},
  year={2026}
}

@article{foster2024behavior,
  title={Is behavior cloning all you need? understanding horizon in imitation learning},
  author={Foster, Dylan J and Block, Adam and Misra, Dipendra},
  journal={Advances in Neural Information Processing Systems},
  volume={37},
  pages={120602--120666},
  year={2024}
}

@article{shenfeld2026self,
  title={Self-distillation enables continual learning},
  author={Shenfeld, Idan and Damani, Mehul and H{\"u}botter, Jonas and Agrawal, Pulkit},
  journal={arXiv preprint arXiv:2601.19897},
  year={2026}
}

@article{ma2026one,
  title={One Student, Many Teachers: Multi-Task On-Policy Distillation via Soft-Prompt Privileged Context},
  author={Ma, Yingzi and Zhu, Zichen and Jiang, Ming and Xiao, Chaowei},
  journal={arXiv preprint arXiv:2607.18293},
  year={2026}
}

@article{zhong2026vla,
  title={Vla-opd: Bridging offline sft and online rl for vision-language-action models via on-policy distillation},
  author={Zhong, Zhide and Yan, Haodong and Li, Junfeng and He, Junjie and Zhang, Tianran and Li, Haoang},
  journal={arXiv preprint arXiv:2603.26666},
  year={2026}
}

@article{hubotter2026reinforcement,
  title={Reinforcement learning via self-distillation},
  author={H{\"u}botter, Jonas and L{\"u}beck, Frederike and Behric, Lejs and Baumann, Anton and Bagatella, Marco and Marta, Daniel and Hakimi, Ido and Shenfeld, Idan and Buening, Thomas Kleine and Guestrin, Carlos and others},
  journal={arXiv preprint arXiv:2601.20802},
  year={2026}
}

@article{liu2026self,
  title={Self-distilled policy gradient},
  author={Liu, Yifeng and Zhang, Shiyuan and Zhang, Yifan and Gu, Quanquan},
  journal={arXiv preprint arXiv:2606.04036},
  year={2026}
}

@article{ma2026mopd,
  title={Mopd: Multi-teacher on-policy distillation for capability integration in llm post-training},
  author={Ma, Wenhan and Wei, Jianyu and Zhao, Liang and Zhang, Hailin and Xiao, Bangjun and Li, Lei and Yang, Qibin and Gao, Bofei and Wang, Yudong and Li, Rang and others},
  journal={arXiv preprint arXiv:2606.30406},
  year={2026}
}

@article{wang2026demystifying,
  title={Demystifying On-Policy Distillation: Roles, Pathologies, and Regulations},
  author={Wang, Rui and Wang, Hongru and Chen, Yi and Xue, Boyang and Fang, Tianqing and Yu, Wenhao and Wong, Kam-Fai},
  journal={arXiv preprint arXiv:2607.13399},
  year={2026}
}

@article{li2026rethinking,
  title={Rethinking on-policy distillation of large language models: Phenomenology, mechanism, and recipe},
  author={Li, Yaxuan and Zuo, Yuxin and He, Bingxiang and Zhang, Jinqian and Xiao, Chaojun and Qian, Cheng and Yu, Tianyu and Gao, Huan-ang and Yang, Wenkai and Liu, Zhiyuan and others},
  journal={arXiv preprint arXiv:2604.13016},
  year={2026}
}

@article{ziheng2026less,
  title={Less is more: Early stopping rollout for on-policy distillation},
  author={Ziheng, Zhou and Li, Jiaqi and Tang, Huacong and Wu, Ying Nian and Terzopoulos, Demetri},
  journal={arXiv preprint arXiv:2605.27028},
  year={2026}
}

@article{krichevsky1981performance,
  title={The performance of universal encoding},
  author={Krichevsky, Raphail and Trofimov, Victor},
  journal={IEEE Transactions on Information Theory},
  volume={27},
  number={2},
  pages={199--207},
  year={1981},
  publisher={IEEE}
}

@book{durrett2019probability,
  title={Probability: theory and examples},
  author={Durrett, Rick},
  volume={49},
  year={2019},
  publisher={Cambridge university press}
}

@inproceedings{agarwal2024policy,
  title={On-policy distillation of language models: Learning from self-generated mistakes},
  author={Agarwal, Rishabh and Vieillard, Nino and Zhou, Yongchao and Stanczyk, Piotr and Ramos Garea, Sabela and Geist, Matthieu and Bachem, Olivier},
  booktitle={International Conference on Learning Representations},
  volume={2024},
  pages={21246--21263},
  year={2024}
}

@article{lu2025onpolicydistillation,
  author = {Kevin Lu and Thinking Machines Lab},
  title = {On-Policy Distillation},
  journal = {Thinking Machines Lab: Connectionism},
  year = {2025},
  note = {https://thinkingmachines.ai/blog/on-policy-distillation},
  doi = {10.64434/tml.20251026},
}

@inproceedings{shenfeld2026rl,
  title={Rl's razor: Why online reinforcement learning forgets less},
  author={Shenfeld, Idan and Pari, Jyothish and Agrawal, Pulkit},
  booktitle={International Conference on Learning Representations},
  volume={2026},
  pages={59839--59864},
  year={2026}
}

@article{chen2025retaining,
  title={Retaining by doing: The role of on-policy data in mitigating forgetting},
  author={Chen, Howard and Razin, Noam and Narasimhan, Karthik and Chen, Danqi},
  journal={arXiv preprint arXiv:2510.18874},
  year={2025}
}

@article{jin2025rl,
  title={Rl fine-tuning heals ood forgetting in sft},
  author={Jin, Hangzhan and Luan, Sitao and Ni, Tianwei and Lyu, Sicheng and Rabusseau, Guillaume and Rabbany, Reihaneh and Precup, Doina and Hamdaqa, Mohammad},
  journal={arXiv preprint arXiv:2509.12235},
  year={2025}
}

@article{yang2025qwen3,
  title={Qwen3 technical report},
  author={Yang, An and Li, Anfeng and Yang, Baosong and Zhang, Beichen and Hui, Binyuan and Zheng, Bo and Yu, Bowen and Gao, Chang and Huang, Chengen and Lv, Chenxu and others},
  journal={arXiv preprint arXiv:2505.09388},
  year={2025}
}

@article{xiao2026mimo,
  title={Mimo-v2-flash technical report},
  author={Xiao, Bangjun and Xia, Bingquan and Yang, Bo and Gao, Bofei and Shen, Bowen and Zhang, Chen and He, Chenhong and Lou, Chiheng and Luo, Fuli and Wang, Gang and others},
  journal={arXiv preprint arXiv:2601.02780},
  year={2026}
}

@article{zeng2026glm,
  title={Glm-5: from vibe coding to agentic engineering},
  author={Zeng, Aohan and Lv, Xin and Hou, Zhenyu and Du, Zhengxiao and Zheng, Qinkai and Chen, Bin and Yin, Da and Ge, Chendi and Huang, Chenghua and Xie, Chengxing and others},
  journal={arXiv preprint arXiv:2602.15763},
  year={2026}
}

@article{sriraman2026behavior,
  title={Behavior Cloning is Not All You Need: The Optimality of On-Policy Distillation for Noisy Expert Feedback},
  author={Sriraman, Ved and Liu, Peihan and Hsu, Daniel and Block, Adam},
  journal={arXiv preprint arXiv:2606.30923},
  year={2026}
}

@article{viano2026interaction,
  title={When Does On-Policy Interaction Help? Representational Tradeoffs in Value-Based Imitation Learning},
  author={Viano, Luca and Moulin, Antoine and Huang, Audrey and Cevher, Volkan and Amortila, Philip and Foster, Dylan J.},
  journal={arXiv preprint arXiv:2607.29617},
  year={2026}
}

@inproceedings{gu2024minillm,
  title={Minillm: Knowledge distillation of large language models},
  author={Gu, Yuxian and Dong, Li and Wei, Furu and Huang, Minlie},
  booktitle={International Conference on Learning Representations},
  volume={2024},
  pages={32694--32717},
  year={2024}
}

@article{ko2024distillm,
  title={Distillm: Towards streamlined distillation for large language models},
  author={Ko, Jongwoo and Kim, Sungnyun and Chen, Tianyi and Yun, Se-Young},
  journal={arXiv preprint arXiv:2402.03898},
  year={2024}
}

@article{russo2013eluder,
  title={Eluder dimension and the sample complexity of optimistic exploration},
  author={Russo, Daniel and Van Roy, Benjamin},
  journal={Advances in Neural Information Processing Systems},
  volume={26},
  year={2013}
}

@article{zhu2026hybrid,
  title={Hybrid policy distillation for llms},
  author={Zhu, Wenhong and Xie, Ruobing and Wang, Rui and Liu, Pengfei},
  journal={arXiv preprint arXiv:2604.20244},
  year={2026}
}

@article{jin2026entropy,
  title={Entropy-aware on-policy distillation of language models},
  author={Jin, Woogyeol and Min, Taywon and Yang, Yongjin and Wei, Dennis and Zhou, Yi and Kadhe, Swanand Ravindra and Baracaldo, Nathalie and Lee, Kimin},
  journal={arXiv preprint arXiv:2603.07079},
  year={2026}
}

@inproceedings{jang2026stable,
  title={Stable on-policy distillation through adaptive target reformulation},
  author={Jang, Ijun and Yeom, Jewon and Yeo, Juan and Lim, Hyunggyu and Kim, Taesup},
  booktitle={Findings of the Association for Computational Linguistics: ACL 2026},
  pages={42217--42227},
  year={2026}
}

@article{xie2026trust,
  title={Trust Region Policy Distillation},
  author={Xie, Zhengpeng and Zhang, Li Lyna and Xie, Zeke and Yang, Mao},
  journal={arXiv preprint arXiv:2607.04751},
  year={2026}
}

@article{xing2026trust,
  title={Trust Region On-Policy Distillation},
  author={Xing, Xingrun and Wang, Haoqing and Gao, Boyan and Li, Ziheng and Tang, Yehui},
  journal={arXiv preprint arXiv:2606.01249},
  year={2026}
}

@article{luo2026demystifying,
  title={Demystifying opd: Length inflation and stabilization strategies for large language models},
  author={Luo, Feng and Chuang, Yu-Neng and Wang, Guanchu and Xu, Zicheng and Han, Xiaotian and Zhang, Tianyi and Braverman, Vladimir},
  journal={arXiv preprint arXiv:2604.08527},
  year={2026}
}

@article{oh2026kl,
  title={KL for a KL: On-Policy Distillation with Control Variate Baseline},
  author={Oh, Minjae and Song, Sangjun and Choi, Gyubin and Choi, Yunho and Jo, Yohan},
  journal={arXiv preprint arXiv:2605.07865},
  year={2026}
}

@article{zhao2026poweropd,
  title={PowerOPD: Stabilizing On-Policy Distillation with Bounded Power Transformation},
  author={Zhao, Anhao and Tong, Junlong and Fan, Yingqi and Nie, Ping and Li, Wenjie and Shen, Xiaoyu},
  journal={arXiv preprint arXiv:2606.17199},
  year={2026}
}

@article{zheng2026scope,
  title={Scope: Signal-calibrated on-policy distillation enhancement with dual-path adaptive weighting},
  author={Zheng, Binbin and Ma, Xing and Liang, Yiheng and Ruan, Jingqing and Fu, Xiaoliang and Lin, Kepeng and Zhu, Benchang and Zeng, Ke and Cai, Xunliang},
  journal={arXiv preprint arXiv:2604.10688},
  year={2026}
}

@article{hou2026uni,
  title={Uni-opd: Unifying on-policy distillation with a dual-perspective recipe},
  author={Hou, Wenjin and Peng, Shangpin and Wang, Weinong and Ruan, Zheng and Zhang, Yue and Zhou, Zhenglin and Gao, Mingqi and Chen, Yifei and Wang, Kaiqi and Yang, Hongming and others},
  journal={arXiv preprint arXiv:2605.03677},
  year={2026}
}

@article{xin2026escaping,
  title={Escaping the KL Agreement Trap in On-Policy Distillation},
  author={Xin, Haoran and Zhao, Anhao and Sun, Ying and Li, Jin and Shen, Xiaoyu and Xiong, Hui},
  journal={arXiv preprint arXiv:2606.09471},
  year={2026}
}

@article{zhao2026self,
  title={Self-distilled reasoner: On-policy self-distillation for large language models},
  author={Zhao, Siyan and Xie, Zhihui and Liu, Mengchen and Huang, Jing and Pang, Guan and Chen, Feiyu and Grover, Aditya},
  journal={arXiv preprint arXiv:2601.18734},
  year={2026}
}

@article{ye2026policy,
  title={On-policy context distillation for language models},
  author={Ye, Tianzhu and Dong, Li and Wu, Xun and Huang, Shaohan and Wei, Furu},
  journal={arXiv preprint arXiv:2602.12275},
  year={2026}
}

@inproceedings{xu2025speculative,
  title={Speculative knowledge distillation: Bridging the teacher-student gap through interleaved sampling},
  author={Xu, Wenda and Han, Rujun and Wang, Zifeng and Le, Long and Madeka, Dhruv and Li, Lei and Wang, William and Agarwal, Rishabh and Lee, Chen-Yu and Pfister, Tomas},
  booktitle={International Conference on Learning Representations},
  volume={2025},
  pages={64616--64646},
  year={2025}
}

@inproceedings{zhang2026fast,
  title={Fast and effective on-policy distillation from reasoning prefixes},
  author={Zhang, Dongxu and Yang, Zhichao and Janghorbani, Sepehr and Han, Jun and Ressler II, Andrew and Qian, Qian and Lyng, Gregory D and Batra, Sanjit Singh and Tillman, Robert E},
  booktitle={Findings of the Association for Computational Linguistics: ACL 2026},
  pages={25553--25569},
  year={2026}
}

@article{fu2026revisiting,
  title={Revisiting on-policy distillation: Empirical failure modes and simple fixes},
  author={Fu, Yuqian and Huang, Haohuan and Jiang, Kaiwen and Liu, Jiacai and Jiang, Zhuo and Zhu, Yuanheng and Zhao, Dongbin},
  journal={arXiv preprint arXiv:2603.25562},
  year={2026}
}

@article{ko2025distillm,
  title={Distillm-2: A contrastive approach boosts the distillation of llms},
  author={Ko, Jongwoo and Chen, Tianyi and Kim, Sungnyun and Ding, Tianyu and Liang, Luming and Zharkov, Ilya and Yun, Se-Young},
  journal={arXiv preprint arXiv:2503.07067},
  year={2025}
}

@article{song2026survey,
  title={A survey of on-policy distillation for large language models},
  author={Song, Mingyang and Zheng, Mao},
  journal={arXiv preprint arXiv:2604.00626},
  year={2026}
}

@article{syed2010reduction,
  title={A reduction from apprenticeship learning to classification},
  author={Syed, Umar and Schapire, Robert},
  journal={Advances in neural information processing systems},
  volume={23},
  year={2010}
}

@inproceedings{ross2010efficient,
  title={Efficient reductions for imitation learning},
  author={Ross, St{\'e}phane and Bagnell, Drew},
  booktitle={Proceedings of the thirteenth international conference on artificial intelligence and statistics},
  pages={661--668},
  year={2010},
  organization={JMLR Workshop and Conference Proceedings}
}

@inproceedings{ross2011reduction,
  title={A reduction of imitation learning and structured prediction to no-regret online learning},
  author={Ross, St{\'e}phane and Gordon, Geoffrey and Bagnell, Drew},
  booktitle={Proceedings of the fourteenth international conference on artificial intelligence and statistics},
  pages={627--635},
  year={2011},
  organization={JMLR Workshop and Conference Proceedings}
}

@article{rajaraman2020toward,
  title={Toward the fundamental limits of imitation learning},
  author={Rajaraman, Nived and Yang, Lin and Jiao, Jiantao and Ramchandran, Kannan},
  journal={Advances in Neural Information Processing Systems},
  volume={33},
  pages={2914--2924},
  year={2020}
}

@article{rajaraman2021value,
  title={On the value of interaction and function approximation in imitation learning},
  author={Rajaraman, Nived and Han, Yanjun and Yang, Lin and Liu, Jingbo and Jiao, Jiantao and Ramchandran, Kannan},
  journal={Advances in Neural Information Processing Systems},
  volume={34},
  pages={1325--1336},
  year={2021}
}

@article{zhao2026sharp,
  title={Sharp analysis for kl-regularized contextual bandits and rlhf},
  author={Zhao, Heyang and Ye, Chenlu and Gu, Quanquan and Zhang, Tong},
  journal={Advances in Neural Information Processing Systems},
  volume={38},
  pages={107964--108002},
  year={2026}
}

@article{zhao2025logarithmic,
  title={Logarithmic regret for online KL-regularized reinforcement learning},
  author={Zhao, Heyang and Ye, Chenlu and Xiong, Wei and Gu, Quanquan and Zhang, Tong},
  journal={arXiv preprint arXiv:2502.07460},
  year={2025}
}

@article{ji2026near,
  title={Near-optimal regret for kl-regularized multi-armed bandits},
  author={Ji, Kaixuan and Zhao, Qingyue and Zhao, Heyang and Di, Qiwei and Gu, Quanquan},
  journal={arXiv preprint arXiv:2603.02155},
  year={2026}
}

@inproceedings{zhao2026towards,
  title={Towards a sharp analysis of offline policy learning for $ f $-divergence-regularized contextual bandits},
  author={Zhao, Qingyue and Ji, Kaixuan and Zhao, Heyang and Zhang, Tong and Gu, Quanquan},
  booktitle={International Conference on Learning Representations},
  volume={2026},
  pages={117176--117210},
  year={2026}
}

@article{hong2026online,
  title={Online KL-Regularized Reinforcement Learning with Function Approximation under Misspecification},
  author={Hong, Haoyang and Wang, Zichen and Gu, Quanquan and Wang, Huazheng},
  journal={arXiv preprint arXiv:2606.06053},
  year={2026}
}

@inproceedings{ayoub2020model,
  title={Model-based reinforcement learning with value-targeted regression},
  author={Ayoub, Alex and Jia, Zeyu and Szepesvari, Csaba and Wang, Mengdi and Yang, Lin},
  booktitle={International Conference on Machine Learning},
  pages={463--474},
  year={2020},
  organization={PMLR}
}

@article{wang2020reinforcement,
  title={Reinforcement learning with general value function approximation: Provably efficient approach via bounded eluder dimension},
  author={Wang, Ruosong and Salakhutdinov, Russ R and Yang, Lin},
  journal={Advances in Neural Information Processing Systems},
  volume={33},
  pages={6123--6135},
  year={2020}
}

@article{ji2026optimal,
  title={On the Optimal Sample Complexity of Offline Multi-Armed Bandits with KL Regularization},
  author={Ji, Kaixuan and Di, Qiwei and Zhao, Heyang and Zhao, Qingyue and Gu, Quanquan},
  journal={arXiv preprint arXiv:2605.02141},
  year={2026}
}

@article{zhao2026fast,
  title={Fast Rates for Offline Contextual Bandits with Forward-KL Regularization under Single-Policy Concentrability},
  author={Zhao, Qingyue and Ji, Kaixuan and Zhao, Heyang and Gu, Quanquan},
  journal={arXiv preprint arXiv:2605.09214},
  year={2026}
}

@article{aminian2026kl,
  title={Kl-regularized rlhf with multiple reference models: Exact solutions and sample complexity},
  author={Aminian, Gholamali and Asadi, Amir R and Shenfeld, Idan and Mroueh, Youssef},
  journal={Advances in Neural Information Processing Systems},
  volume={38},
  pages={101117--101151},
  year={2026}
}

@article{wu2026greedy,
  title={Greedy sampling is provably efficient for RLHF},
  author={Wu, Di and Shi, Chengshuai and Yang, Jing and Shen, Cong},
  journal={Advances in Neural Information Processing Systems},
  volume={38},
  pages={108198--108232},
  year={2026}
}

@article{nayak2025achieving,
  title={Achieving logarithmic regret in kl-regularized zero-sum markov games},
  author={Nayak, Anupam and Yang, Tong and Yagan, Osman and Joshi, Gauri and Chi, Yuejie},
  journal={arXiv preprint arXiv:2510.13060},
  year={2025}
}

@article{wu2025offline,
  title={Offline and online kl-regularized rlhf under differential privacy},
  author={Wu, Yulian and Thareja, Rushil and Vepakomma, Praneeth and Orabona, Francesco},
  journal={arXiv preprint arXiv:2510.13512},
  year={2025}
}

@inproceedings{agarwal2023vo,
  title={VO $ Q $ L: Towards Optimal Regret in Model-free RL with Nonlinear Function Approximation},
  author={Agarwal, Alekh and Jin, Yujia and Zhang, Tong},
  booktitle={The Thirty Sixth Annual Conference on Learning Theory},
  pages={987--1063},
  year={2023},
  organization={PMLR}
}

@article{zhao2024nearly,
  title={A nearly optimal and low-switching algorithm for reinforcement learning with general function approximation},
  author={Zhao, Heyang and He, Jiafan and Gu, Quanquan},
  journal={Advances in Neural Information Processing Systems},
  volume={37},
  pages={94684--94735},
  year={2024}
}

@inproceedings{di2024pessimistic,
  title={Pessimistic nonlinear least-squares value iteration for offline reinforcement learning},
  author={Di, Qiwei and Zhao, Heyang and He, Jiafan and Gu, Quanquan},
  booktitle={International Conference on Learning Representations},
  volume={2024},
  pages={4377--4410},
  year={2024}
}
\bibliographystyle{ims}
\end{document}